\documentclass[opre]{informs4}
\makeatletter
\def\theARTICLETOP{}%
\def\setoddRH{\hbox to \textwidth{\fs.7.8.\tabcolsep0pt
  \begin{tabular*}{\textwidth}[b]{l@{\extracolsep\fill}r}
  {\theRRHFirstLine}& \raisebox{0pt}[0pt][0pt]{\fs.10.10.\thepage}\\[-4pt]
  \rlap{\VRHDW{0.5pt}{0pt}{\textwidth}}&\\
  \end{tabular*}}}
\def\setevenRH{\hbox to \textwidth{\fs.7.8.\tabcolsep0pt
  \begin{tabular*}{\textwidth}[b]{l@{\extracolsep\fill}r}
  \raisebox{0pt}[0pt][0pt]{\fs.10.10.\thepage}& {\theLRHFirstLine}\\[-4pt]
  \rlap{\VRHDW{0.5pt}{0pt}{\textwidth}}&\\
  \end{tabular*}}}

\def\theARTICLETITLE{%
  \HOOKtop
  \begin{center}
  \vspace*{-57pt}%   <-- was 0pt; negative value moves the title up
  \HD{24}{0}%
  \resizebox{\textwidth}{!}{\sffamily\bfseries\theTITLE}%
  \HD{0}{15}%
  \end{center}}
\def\AUTHOR#1{\begin{center}\AUTHORfont\HD{15}{0}#1\HD{0}{6}\end{center}}
\def\AFF#1{\begin{center}\AFFfont{#1}\vskip1.6pt\end{center}}
\def\theARTICLEABSTRACT{%
  \HOOKb
  \vspace*{18pt}
  \begin{minipage}[t]{\textwidth}\parindent1em
    \ABSfont
    \noindent\theABSTRACT\endgraf
    \vskip5pt
    \theFUNDING\theKEYWORDS\theSUBJECTCLASS
    \theAREAOFREVIEW\theMSCCLASS\theORMSCLASS
    \if@BLINDREV\else\theHISTORY\fi
  \noindent  \hrulefill
  \end{minipage}%
  \vspace*{0pt}}
\makeatother

\newenvironment{lemma*}[1][]{%
  \par\addvspace{6pt}\noindent
  \textbf{Lemma\if\relax\detokenize{#1}\relax\else\ (#1)\fi.}\itshape\ \ignorespaces
}{%
  \par\addvspace{6pt}\upshape
}

\newenvironment{proposition*}[1][]{%
  \par\addvspace{6pt}\noindent
  \textbf{Proposition\if\relax\detokenize{#1}\relax\else\ (#1)\fi.}\itshape\ \ignorespaces
}{%
  \par\addvspace{6pt}\upshape
}

\RequirePackage{tgtermes}
\RequirePackage{newtxtext}
\RequirePackage{newtxmath}
\RequirePackage{bm}
\RequirePackage{endnotes}

\OneAndAHalfSpacedXI
\usepackage{algpseudocode}
\usepackage{tikz}
\usepackage{multirow}
\usepackage[table]{xcolor}

\usepackage{natbib}
 \bibpunct[, ]{(}{)}{,}{a}{}{,}%
 \def\bibfont{\small}%

\usepackage{pgfplots}
\usepackage{subfigure}
\usepackage{bbm}
\usepackage{natbib}
\usepackage{amssymb}
\usepackage{amsmath}
\usepackage{amsfonts}
\bibpunct[, ]{(}{)}{,}{a}{}{,}%
 \def\bibfont{\small}%
\usepackage[ruled]{algorithm2e} % For algorithms

\usepackage{mathtools}

\usepackage{hyperref}
\hypersetup{
  colorlinks = true,
  allcolors  = blue         
}
\usepackage{booktabs,subcaption,makecell}
\usepackage{graphicx,float,bm}
\usetikzlibrary{positioning,calc}
\usepackage{tikz}
\usepackage{pgfplots}
\pgfplotsset{compat=1.17}
\newcommand{\R}{\mathbb{R}}
\newcommand{\E}{\mathbb{E}}
\newcommand{\Prob}{\mathbb{P}}
\newcommand{\Ft}{\mathcal{F}}
\newcommand{\cB}{\mathcal{B}}

\newcommand{\one}{\mathbb{I}}

\SetKwBlock{Init}{Initialization}{end}

\newcommand{\DOW}{\textsc{DOW}\ }
\newcommand{\proofend}{\hfill$\blacksquare$}
\newcommand{\indep}{\perp\!\!\!\!\perp}

\EquationsNumberedThrough    % Default: (1), (2), ...
\TheoremsNumberedThrough     % Preferred (Theorem 1, Lemma 1, Theorem 2)
\ECRepeatTheorems  %  

\MANUSCRIPTNO{ }

\begin{document}
%%%%%%%%%%%%%%%%

%\setcounter{page}{1}
%\DOI{10.1287/xxxx.0000.0000}%

\RUNAUTHOR{Belloni, Chen and Wei}

% \RUNTITLE{Predictive Maintenance in Manufacturing}
\RUNTITLE{Online Pandora's Box for Contextual LLM Cascading}

\TITLE{Online Pandora's Box for Contextual LLM Cascading}

\ARTICLEAUTHORS{%
\AUTHOR{Alexandre Belloni}
\AFF{The Booth School of Business, University of Chicago, \EMAIL{Alex.Belloni@chicagobooth.edu}}

\AUTHOR{Yan Chen}
\AFF{The Fuqua School of Business, Duke University, \EMAIL{yc555@duke.edu}}

\AUTHOR{Yehua Wei}
\AFF{The Fuqua School of Business, Duke University, \EMAIL{yehua.wei@duke.edu}}
}

\ABSTRACT{Motivated by Large Language Model (LLM) cascading, we propose an online contextual Pandora's Box model for adaptively querying and selecting LLM APIs. In each period, a decision-maker observes a request context and faces a two-phase decision problem. In the query phase, the decision-maker sequentially queries APIs, where each query reveals a generated output and the decision-maker incurs an (output-dependent) cost. In the selection phase, the decision-maker selects one of the generated outputs to deploy and observes only the downstream reward of the deployed output. This output-mediated feedback structure differs from classical online contextual Pandora's Box models, in which opening a box directly reveals its reward.

Rather than estimating the full conditional output and cost distributions of each API, we directly model the reservation index and develop a learning approach for the query phase. Specifically, we impose a parametric structure on the contextual reservation index functions induced by the classical Weitzman's policy. 
Our policy combines generalized method of moments (GMM) type estimation of these reservation indices with UCB-style confidence bounds for both these indices and the shared output-level reward evaluator. 
Under regularity conditions, we prove that the resulting policy achieves dimension-dependent \(\widetilde O(\sqrt T)\) cumulative regret over a horizon of \(T\) periods.
}

\maketitle
%%%%%%%%%%%%%%%%%%%%%%%%%%%%%%%%%%%%%%%%%%%%%%%%%%%%%%%%%%%%%%%%%%%%%%
% Text of your paper here
\section{Introduction}\label{sec:intro}
The proliferation of Large Language Models (LLMs) has transformed the economics of content generation and decision-making. Firms increasingly use generative AI systems to generate advertisements, produce code, serve customer requests, and conduct complex analytical tasks. In these applications, the operational challenge is often not whether content can be generated, but how to generate high-quality content reliably and cost-effectively at scale. A firm that produces a high volume of advertisements per day, for example, may have access to a portfolio of LLM APIs that differ in cost and output quality across task types. More capable proprietary models may, in general, deliver higher-quality outputs but incur
substantially higher costs, while smaller or specialized models may be cheaper but less reliable. Because model performance varies across request contexts and it is difficult to predict the quality of the outputs a priori, committing to a single API is often suboptimal. Similarly, querying all of the available APIs is typically prohibitively costly.

We study this problem from the perspective of an organization that uses external or internal LLM APIs to serve a stream of business requests. In such applications, requests arrive as distinct decision units without a strict latency requirement, allowing the decision-maker to query more than one API before selecting a final output. It is worth noting that this request-level decision problem fundamentally differs from the model hosting problem faced by large LLM platforms, where a key operational challenge is to reduce latency through batching, scheduling, routing, and load management \citep[e.g.][]{yu2022orca,kwon2023efficient,agrawal2024sarathi,ao2025optimizing, jaillet2025online, guo2026congestion}. Our focus is instead on the API consumer's decision problem of dynamically querying a portfolio of APIs to secure high-quality outputs in a cost-effective manner.

A prominent approach for managing this cost-quality trade-off
is \emph{LLM cascading} \citep[see e.g.,][]{chenfrugalgpt,yue2024large,gupta2024language}. 
Rather than sending every request directly to a single expensive model, a cascade queries LLM APIs sequentially, typically beginning with lower-cost models and escalating only when earlier outputs appear insufficiently reliable. 
The pioneering work of \citet{chenfrugalgpt}, for example, illustrates this design by routing a query through a sequence of LLM APIs and stopping once a generated response receives a sufficiently high reliability score from an evaluator. Cascading allows the easy requests to be handled by cheaper APIs, with only the difficult requests escalated to the expensive ones. Empirical evidence from~\citet{chenfrugalgpt} suggests that such cascades can substantially reduce inference costs while preserving, and in some cases improving, task performance.

The LLM cascading system raises a fundamental sequential decision problem. For each incoming request, the decision-maker must manage both a query phase and a selection phase. In the query phase, the decision-maker must decide which API to query next and whether the generated outputs justify continuing the search. In the selection phase, the decision-maker must decide which of the generated outputs to deploy. These decisions are tightly coupled. Querying an additional API incurs a cost but may yield a better output, whereas stopping early saves cost at the risk of missing a better outcome. Moreover, selecting among the queried outputs requires evaluating their downstream value. Thus, LLM cascading is not simply a model-routing problem, but a sequential search problem with costly information acquisition.

A natural starting point for principled analysis of this search problem is the classical Pandora’s Box problem. Consider an idealized benchmark in which, for each request, the decision-maker knows the context-dependent distribution of the output-cost pair generated by each API and can evaluate the downstream  value of any generated output once it is observed. In this benchmark, querying an API corresponds to opening a box, the inference cost acts as the inspection cost, and the downstream value of the generated output serves as the realized reward. For this formulation, the seminal work of \citet{weitzman1979optimal} characterizes the structure of the optimal policy. Specifically, given the context of a request, the policy assigns each API a reservation index, then queries APIs in decreasing order of these indices, and terminates the search as soon as the reward of an output exceeds the highest remaining index.

Real applications of LLM cascading, however, depart from this oracle benchmark in ways that make learning central. The value of querying an API is inherently contextual, and the primitives required by the oracle Weitzman's policy are rarely known in advance. In some settings, firms may have substantial historical data on prior human- or AI-generated outputs and their downstream rewards, allowing them to construct an accurate offline model of the reward function. In many others, however, such output-reward pairs are limited or unavailable. Moreover, even when rewards can be modeled from offline data, a firm deploying LLM APIs typically does not know the conditional distribution of outputs and costs generated by each API under a given request context. Since requests arrive sequentially over time, the decision-maker must learn context-dependent reservation indices while continuing to make cost-sensitive query and selection decisions. This naturally leads to an online contextual Pandora’s Box formulation of LLM cascading.

Recent work on LLM cascading has developed increasingly sophisticated methods for balancing cost and performance. These include budget-constrained cascade policies, uncertainty or threshold-based deferral rules, and hybrid methods that combine cascading with model routing \citep{chenfrugalgpt,yue2024large,gupta2024language,zhang2024efficient}. While these approaches provide important insights into efficient multi-LLM deployment, they rely primarily on heuristic designs. To the best of our knowledge, existing work has not yet formulated LLM cascading as an online contextual Pandora’s Box problem with joint query and selection, nor developed formal regret guarantees.
Conversely, the existing Pandora’s Box literature is not directly tailored to this LLM setting. 
A prominent line of work takes distribution-learning approaches \citep{liu2025improved}, in which opening a box reveals a scalar reward and the learner estimates how the reward distribution varies with context, under parametric structure on the full distribution. 
Closer to our work, \citet{atsidakou2024contextual} parameterizes contextual reservation indices directly, but their model remains a classical Pandora's box setting in which opened boxes reveal rewards directly. In LLM cascading, by contrast, an API call reveals a stochastic output-cost pair. In our model, the downstream value depends on the request and the generated output, while API-specific heterogeneity enters through the distribution of outputs and costs. 

\subsection{Our Contributions}
We make three main contributions. First, we introduce an online contextual Pandora's Box model motivated by LLM cascading. Unlike existing online Pandora's Box models, our model features a natural two-phase decision structure: querying an API reveals a stochastic output and incurs a cost, while the downstream reward is observed only after one generated output is selected and deployed. This formulation captures a key practical feature of LLM cascading: API-specific heterogeneity enters through the distributions of generated outputs and costs, while the downstream value of a generated 
output is evaluated through a shared reward model on the request-output pair. That is, an API can influence the downstream value only through the generated output.

Second, we propose a direct reservation index modeling and learning approach for the query phase.
In the full-information benchmark, Weitzman's policy is governed by contextual reservation indices, which are induced by each API's conditional output and cost distributions together with the reward evaluator. Rather than estimating these distributions directly, we impose a parametric structure on the reservation index functions themselves as a generalized linear function of observable context
features. This modeling assumption leads to a practical optimism-based learning approach in the LLM cascading setting, as the decision-maker can focus on constructing optimistic box-specific reservation indices from queried output-cost samples through the moment equations instead of the full conditional output distributions. 
The policy also learns the shared output-level reward evaluator from deployed-output rewards, and the two learned objects jointly determine the cascade's querying and selection decisions.

Third, we provide regret guarantees for the resulting online policy. The policy combines \emph{generalized method of moments} (GMM) estimation for reservation indices with \emph{upper confidence bound} (UCB)-style approaches for both the indices and rewards. Our policy is motivated by a simple but crucial regret decomposition under optimism: cumulative regret separates into errors from estimating reservation indices for only the \emph{queried APIs} and errors from evaluating generated outputs for only the \emph{selected APIs}. 
This decomposition allows us to apply the UCB-style analysis to our moment estimation problem, yielding a dimension-dependent $\tilde O(\sqrt{T})$ cumulative regret bound. The analysis covers the practically important known-evaluator regime, where offline data provide an accurate output-level reward model, as well as the full regime where the evaluator and reservation indices must be learned jointly online.

\subsection{Related Literature}
Our work is related to the literature on cost-aware LLM systems, sequential test-time inference, and the theoretical frameworks of sequential search and online learning.

First, our paper is closely related to the emerging literature on LLM cascading and cost-aware multi-LLM deployment. This literature studies how to allocate inference requests across multiple LLM APIs when models differ in cost, latency, and output quality. Representative approaches include FrugalGPT \citep{chenfrugalgpt}, budget-constrained cascade policies, uncertainty- or reliability-based deferral rules, and hybrid methods that combine cascading with model routing \citep{nie2024online,chen2024cascade,zhang2024efficient,yue2024large,gupta2024language}. Our work contributes to this literature by formulating LLM cascading as an online contextual Pandora's Box problem, where the decision-maker sequentially decides which APIs to query and when to stop under unknown, context-dependent output and cost distributions. Our work is also related to LLM routing and ensembling methods although they address different decision problems than our LLM cascading motivation. Routing methods usually choose a single LLM for each request before
observing any generated output, either through non-predictive rules or through
predicted quality, reward, or utility scores
\citep{ hu2024routerbench, shnitzerlarge,hari2023tryage,lu2024routing,vsakota2024fly, mei2025omnirouter}. Such methods capture one-shot model choice but are not adapted to within-request search. Ensemble methods aggregate or synthesize outputs from multiple LLMs, such as by ranking and fusing candidate responses \citep{ jiang2023llm,fang2024llm,hu2025efficient}; they focus on output aggregation rather than sequential stopping under query costs. 

Second, our paper is related to recent work on sequential testing and sequential stopping for LLM inference. \citet{huang2026optimal} study adaptive self-consistency for a single LLM, where the system sequentially samples reasoning paths and stops once the most likely answer can be identified with high posterior confidence. \citet{li2026asymptotically} study Bayesian sequential testing with heterogeneous LLMs, where the decision-maker adaptively chooses which LLM to query and stops once the posterior confidence for a hypothesis crosses a target threshold. These papers formulate LLM inference as posterior-driven sequential testing, whereas our LLM cascading model focuses on contextual sequential search over realized output-cost pairs, with stopping governed by reservation indices and learning focused on both the index functions and the output-level reward model.

Third, the query phase component of our model extends the contextual Pandora’s Box problem. The classic Pandora’s Box problem without contexts originates from the seminal work of \citet{weitzman1979optimal} and has since been studied under a variety of extensions \citep[e.g.][]{doval2018whether,chawla2020pandora,fu2020learning,boodaghians2020pandora,agarwal2024semi,ezra2026contract}, including online variants \citep{gergatsouli2022online,atsidakou2024contextual}. In particular, \citet{gergatsouli2022online} studies an adversarial online Pandora’s Box setting in which box rewards are chosen by an adversary. \citet{gatmiry2024bandit} subsequently shows that, in this setting, no algorithm can achieve sublinear regret against the optimal policy, even with full-information feedback. For the online contextual Pandora's Box problem, \citet{liu2025improved} obtains a regret bound of $\widetilde{O}(\sqrt{T})$ under the assumption that
the reward distribution of each box is a location shift of a context-invariant noise distribution, but their results do not apply to heterogeneous output distributions across contexts within each box, a key feature for LLM cascading applications. \citet{atsidakou2024contextual} studies a contextual Pandora's Box model under adversarially arriving contexts, where the optimal reservation index is parametrized as a linear function of the observed context. In this general setting, they obtain a $\widetilde{O}(T^{5/6})$ regret bound by reducing the learning problem to online linear regression.
Our approach builds on a similar model of \citet{atsidakou2024contextual} by imposing a parametric reservation index structure, but adopts a different learning approach through a combination of GMM and UCB analysis. This allows us to propose an optimism-based learning algorithm with a $\widetilde{O}(\sqrt{T})$ regret bound under suitable regularity assumptions.

Further, the selection phase of our model draws on the literature on generalized linear contextual bandits. This line of work originates from the seminal paper of \citet{filippi2010parametric}, which introduces an optimism-based algorithm for generalized linear bandits, and has since been extended in various directions \citep[e.g.][]{jun2017scalable,kveton2020randomized,ding2021efficient,kim2023double}. 

Finally, our learning algorithm combines GMM with a UCB–type algorithm. GMM originates from the seminal work of \citet{hansen1982large} and has been widely developed and applied in econometrics \citep[e.g.][]{newey1994large,arellano1991some,chamberlain1987asymptotic,lin2010gmm,cheng2024gmm,andrews2022optimal,hansen2021inference}. UCB-type algorithms trace back to the seminal paper of \citet{lai1985asymptotically} and have since been extensively studied in the bandit literature \citep[e.g.][]{auer2002finite,rusmevichientong2010linearly,filippi2010parametric,fan2025fragility,garivier2011kl,audibert2009exploration}, as well as in broader online learning settings and applications in operations management \citep[e.g.][]{rakhlin2013optimization,gao2022joint,cheung2022inventory,mao2025model}.

In addition, our paper belongs to the growing literature on LLM-assisted decision making across a range of fields, including operations management \citep[e.g.][]{chen2024large,chen2025manager,simchi2026large}, advertising \citep[e.g.][]{yang2023against,reisenbichler2025applying}, healthcare \citep[e.g.][]{thirunavukarasu2023large,hager2024evaluation,hao2025large}.

\subsection{Notations}
Given any integer $k\geq1$ and vector $v\in\R^k$, we use $\|v\|_2$ to denote the Euclidean norm of $v$, and $\|v\|_{\infty}$ to denote $\max_{i=1,\ldots,k}|v_i|$. For any positive semidefinite matrix $\Phi\in\mathbb{R}^{k\times k}$, we use the notation $\|v\|_{\Phi}=\sqrt{v^\top \Phi v}$ to denote the $\Phi$-weighted Euclidean norm of $v$. 
For matrices $A,B\in\mathbb{R}^{d\times d}$, $A\succeq B$ (resp., $A\preceq B$) means that $A-B$ (resp., $B-A$) is positive semidefinite. For $x,y\in\mathbb{R}$, we write $x\vee y=\max\{x,y\}$. For any positive semidefinite matrix $M\in\mathbb{R}^{d\times d}$, $\lambda_{\min}(M)$ denotes its minimum eigenvalue. For any integer $k\ge1$, let $[k]=\{1,2,\ldots,k\}$. Given any set $S$, we use $S^c$ to denote the complement of $S$. We write $\omega\sim p(\cdot)$ to indicate that the random variable $\omega$ follows distribution $p(\cdot)$. We use $N(\theta,\sigma^2)$ to denote the Gaussian distribution with mean $\theta$ and variance $\sigma^2$. Given any $\mu\in\mathbb{R}^d$ and positive definite matrix $\Sigma\in\mathbb{R}^{d\times d}$, we use $\mathcal{N}(\mu,\Sigma)$ to denote the $d$-dimensional Gaussian distribution with mean $\mu$ and covariance $\Sigma$. For any two random variables $X,Y$, $X\indep Y$ means $X$ and $Y$ are independent. For any random event $\mathcal{E}$, we use $\mathbb{I}\{\mathcal{E}\}$ to denote the indicator of $\mathcal{E}$. The notation ``a.s.” stands for ``almost surely.” For any vector $x$ or matrix $A$, $x^\top$ and $A^\top$ denote their transposes. Finally, for any twice-differentiable function $g$, $g'$ and $g''$ denote its first- and second-order derivatives. We use $\widetilde{O}(\cdot)$ to hide logarithmic factors. Specifically, for functions $f,g:\mathbb{N}\to\mathbb{R}_+$, we write $f(T)=\widetilde{O}(g(T))$ if there exists a constant $C>0$ and a poly-logarithmic function $\mathrm{polylog}(T)$ such that $f(T)\leq Cg(T)\textrm{polylog}(T)$. For random variables $Y_1$ and $Y_2$, $Y_1\lesssim Y_2$ denotes that $Y_1\le cY_2$ with high probability for some constant $c>0$. For sequences $f(T), g(T)>0$, we write $f(T)\asymp g(T)$ if both $f(T)\lesssim g(T)$ and $g(T)\lesssim f(T)$ hold up to universal constant factors.

\section{Model}\label{sec:model}
In this section, we formally introduce our online contextual Pandora’s box model motivated by LLM cascading. In our model, there are $A$ different boxes, each of which can be thought of as an API in LLM cascading. There are $T$ decision periods. 
Let $\mathcal F_{t-1}$ denote the history available before period $t$, including all past observations and actions. At the beginning of period \(t\), the decision-maker (DM) observes a request context vector $x_t \in \mathcal{X} \subset \mathbb{R}^{d_x}$, where the contexts $\{x_t\}_{t\in[T]}$ are independent across $t\in[T]$. 
For each box $a \in [A]$, $\omega_{at}$, defined as the potential output vector of $a$, is random and its distribution is assumed to be independent across boxes and independent of the past history. Formally, 
for each $a\in [A]$, $\omega_{at}\in\Omega\subset\mathbb{R}^{d_w}$ is drawn from a context-dependent distribution $p_a(\cdot|x_t)$, and 
for any measurable sets \(B_1,\ldots,B_A\subseteq\Omega\), 
we have
\begin{equation}\label{eq:law:outputs}
\mathbb P\!\left(
\omega_{1t}\in B_1,\ldots,\omega_{At}\in B_A
\mid \mathcal F_{t-1},x_t
\right) =
\prod_{a=1}^A p_a(B_a\mid x_t)
\qquad \text{a.s.}
\end{equation}
The output $\omega_{at}$ is revealed only if box $a$ is queried. In addition, if $a$ is queried, 
a cost $c_{at}:=c_a(x_t,\omega_{at})$ is incurred and observed,
where \(c_a:\mathcal X\times\Omega\to(0,1)\) is the cost function associated with box $a$. 
Finally, at the end of each period, 
the DM selects one queried box $a_{t}$ whose generated output $\omega_{a_{t}t}$ will be deployed. The DM then receives and observes the realized reward associated with the deployed output, $r_{t}=\mu^*(x_t,\omega_{a_{t}t})+\zeta_{t}$, 
where $\mu^*(\cdot,\cdot):\mathcal{X}\times\Omega\rightarrow[0,1]$ is the (potentially unknown) reward function
and $\zeta_{t}$ is the post-deployment noise satisfying $\mathbb{E}[\zeta_t \mid \mathcal{F}_{t-1}, x_t,\omega_{a_t,t}]=0$. 

A key feature of our model is that the reward function is shared across boxes. 
Conditional on the same request-output pair $(x,\omega)$, the expected downstream reward is the same regardless of which box generated the output. 
Thus, box-specific heterogeneity only enters through the distributions of outputs and costs, while realized rewards depend solely on the common reward function $\mu^*(\cdot,\cdot)$ and noise $\zeta_t$. Consequently, outputs $\omega_{at}$ act as intermediaries linking boxes to rewards, which permits cross-learning of $\mu^*$ using reward observations from all selected boxes. 
This reflects our motivating LLM example, where the reward received by the DM depends on the generated output and the request it serves, not directly on the identity of the API that generated it. 

We preview the parametric structures used for learning, with the formal assumptions
and estimators introduced in Sections \ref{sec:known:theta-*} and \ref{sec:estimations}.
Suppose, as an oracle benchmark, the DM knows the conditional output distributions 
\(\{p_a(\cdot\mid x)\}_{a\in[A],x\in\mathcal X}\), and the reward function \(\mu^*\). 
Then, after the context $x_t$ is observed, the decision problem in period $t$ reduces to a contextual Pandora's box problem of the type studied by \citet{weitzman1979optimal}. 
As we formally illustrate in Section~\ref{sec:policy}, the optimal full-information policy is determined by two objects: the reward function \(\mu^*\), and the reservation indices, denoted as \(\{\sigma_a^*(x_t)\}_{a\in[A]}\). These indices are, in turn, determined by the oracle objects \(p_a(\cdot|x_t)\), \(\mu^*\), and \(c_a\) through the reservation index equation \eqref{eq:oracle-weitzman-index} in Proposition \ref{prop:oracle-weitzman}.
In the online learning problem, these oracle objects are not known. Rather than estimating the full conditional output distributions \(p_a(\cdot\mid x)\) and the cost functions $c_a(\cdot,\cdot)$, we impose a generalized linear parametric structure directly on the reward function \(\mu^*\) and on the contextual reservation index functions \(\{\sigma_a^*(\cdot)\}_{a\in[A]}\), namely  
$\mu^*(x,\omega)=G(\theta_*^\top\phi(x,\omega))$ and $\sigma_a^*(x) = \Lambda(\rho_a^\top\psi(x))$ 
where the coefficients $\theta_*$ and $\rho_a$ are unknown; $G$, $\Lambda$ are known (monotonic) link functions and 
$\phi$, $\psi$ are known feature maps, respectively.
This allows the DM to learn the shared reward and box-specific contextual reservation indices through parametric estimation. 

\begin{remark}
Our contextual Pandora's box model maps to LLM cascading as follows. A period corresponds to a service request, and a box corresponds to an LLM API or an API-prompt configuration. Querying box $a$ once produces one candidate output \(\omega_{at}\), the DM then either stops or queries another box in the cascade. This process continues until the DM stops or all boxes have been queried. We restrict each box to be queried at most once within a period. This is consistent with the LLM cascading literature \citep{chenfrugalgpt}, where a request is routed through a sequence of APIs and each API is invoked only if previous outputs are deemed to be not sufficiently reliable by the evaluator.

The output-dependent cost \(c_a(x_t,\omega_{at})\) also arises naturally in LLM applications. For example, under token-based pricing, $c_{at}=c_a(x_t,\omega_{at})=\alpha_a I_{at}+\beta_a O_{at}$, where $I_{at}$ is the input token count of the query and $O_{at}$ is the (random) number of output tokens. Here $I_{at}=\kappa_{a,i}(x_{t})$ and $O_{at}=\kappa_{a,o}(\omega_{at})$, where for each $a\in[A]$, $\kappa_{a,i}:\mathbb{R}^{d_x}\rightarrow\mathbb{R}_{+}$, $\kappa_{a,o}:\mathbb{R}^{d_w}\rightarrow\mathbb{R}_{+}$ are fixed functions. 
\end{remark}

\subsection{DM's Problem and Regret}
Next, we describe the decision-making pipeline in each period $t\in[T]$ in more detail. Each period $t\in[T]$ consists of two phases, which we call \emph{query phase} and \emph{selection phase}.

During the \emph{query phase}, the DM observes $x_t$ and sequentially queries boxes. A policy is non-anticipating: before each query, the DM may use $\mathcal F_{t-1}$, the current context $x_t$, and the output-cost pairs of the boxes already observed in period $t$, but not the output or cost of any unqueried box. Let \(\mathcal A_t\subseteq[A]\) denote the random set of boxes queried in period \(t\). We assume each box can be queried at most once within a period. After each query, the DM observes the generated output and cost, and then decides whether to continue querying or stop.
Combined with (\ref{eq:law:outputs}), this non-anticipativity condition implies a useful sampling property. For any \(a\in[A]\), conditional on \((\mathcal F_{t-1},x_t)\), the event \(\{a\in\mathcal A_t\}\) is determined before observing box \(a\)'s own potential output-cost pair. Hence, whenever \(\mathbb P(a\in\mathcal A_t\mid \mathcal F_{t-1},x_t)>0\), it follows that 
$$(\omega_{at},c_{at})
    \,|\, \mathcal F_{t-1},x_t,\{a\in\mathcal A_t\}
    \;\overset{d}{=}\;
    (\omega_{at},c_{at})
    \,|\, \mathcal F_{t-1},x_t.$$
This property is formally stated and proved as Lemma~\ref{lemma:indep:output:query} in Appendix~\ref{appendix:proof:prop:1:2} and will be used later to justify learning reservation indices from adaptively queried samples.
After the query phase concludes, the DM 
enters the \emph{selection phase}, selecting one queried box $a_{t} \in \mathcal{A}_{t}$ whose generated output $\omega_{a_{t}t}$ will be deployed. 

Let \(\Pi_t\) denote the class of admissible non-anticipating policies in period \(t\). For a policy \(\pi_t\in\Pi_t\), let \(\mathcal A_t(\pi_t)\) be the set of queried boxes and \(a_t(\pi_t)\in\mathcal A_t(\pi_t)\) be the selected box. The conditional expected utility of policy \(\pi_t\) in period \(t\) is 
\[
    U_t(\pi_t;\mathcal F_{t-1},x_t)
    :=
    \mathbb E\bigg[
        \mu^*(x_t,\omega_{a_t(\pi_t)t})
        -
        \sum_{a\in\mathcal A_t(\pi_t)} c_{at}
        \,\bigg|\,
        \mathcal F_{t-1},x_t
    \bigg].
\]
We stress that the post-deployment noise $\zeta_t$ does not enter the utility because we select a queried box prior to obtaining  information about $\zeta_t$ and conditional on the context and the output $\zeta_t$ is mean-zero. However, this noise  matters to the feedback used to learn the reward function $\mu^*(\cdot, \cdot)$. Let $\pi_t^*$ be the (optimal) oracle policy given the knowledge of the distribution $p_a(\cdot|x_t)$, and functions $\mu^*(x_t,\cdot), c_a(x_t,\cdot)$ for every box $a$ and context $x_t$.
The conditional expected utility of $\pi_t^*$ is thus the full-information benchmark 
$$U_t(\pi_t^*;\mathcal F_{t-1},x_t)
= 
\mathbb E\bigg[
    \max_{a\in\mathcal A_t^*}
    \mu^*(x_t,\omega_{at})
    -
    \sum_{a\in\mathcal A_t^*} c_{at}
    \,\bigg|\,
    \mathcal F_{t-1},x_t
\bigg],$$
where $\mathcal{A}_t^*=\mathcal{A}_t(\pi_t^*)$. Then, for any admissible policy \(\boldsymbol \pi=\{\pi_t\}_{t\in[T]}\), we define the cumulative regret as $$R_T(\boldsymbol \pi)
:=
\sum_{t=1}^T
\mathbb E\!\left[
    U_t(\pi_t^*;\mathcal F_{t-1},x_t)
    -
    U_t(\pi_t;\mathcal F_{t-1},x_t)
\right].$$
Readers familiar with the bandit literature might ask whether the problem can be modeled as a contextual bandit by treating each API as an arm. Such a formulation 
would lead to a one-shot API-as-arm policy, in which the decision-maker selects a single API for each request before observing any generated output. It does not, however, capture the adaptive information-acquisition structure of LLM cascading. The following example illustrates that such a restriction can create a constant per-period gap relative to the sequential-search benchmark, even when all primitives are known.

Consider an example with two boxes in which the context $x_t$ is fixed throughout. Querying box 1 costs \(c_1=0.5\) and always generates output \(\omega_{1t}=1\). Querying box 2 costs \(c_2=0.01\) and generates output \(\omega_{2t}=1\) with probability \(0.1\), and \(\omega_{2t}=0\) otherwise. The downstream value is \(\mu^*(x_t,\omega)=\omega\). A one-shot API-as-arm policy obtains expected utility $0.5$ from box 1 and \(0.09\) from box 2, so the best one-shot policy selects box 1.
Now consider a cascade that first queries box 2. If \(\omega_{2t}=1\), it stops and selects box 2; if \(\omega_{2t}=0\), it queries box 1 and selects box 1. Its expected utility is
$-0.01 + 0.1 + 0.9(0.5) = 0.54 > 0.5$.
Thus, relative to the cascade benchmark, any one-shot API-as-arm formulation suffers a constant per-period optimality gap in this instance, and hence a loss that grows linearly with \(T\). 

\section{Policy and Algorithm}\label{sec:policy}
This section develops the online policy and the algorithm. We begin with the full-information benchmark. Fix a period $t$ and condition on
the arriving context $x_t$. Suppose the DM knows the reward function
$\mu^*(\cdot,\cdot)$, the cost function $c_a(\cdot,\cdot)$ and the output distributions $\{p_a(\cdot|x)\}_{a\in[A],x\in\mathcal{X}}$. In this full-information setting, the optimal policy is independent of the history $\mathcal{F}_{t-1}$ and depends only on the current context $x_t$. The problem therefore decouples across periods.
The following proposition characterizes
the optimal oracle policy during period $t$.
\begin{proposition}[Oracle reservation index]
\label{prop:oracle-weitzman}
For each box $a\in[A]$ and any fixed $t \in [T]$, let $\sigma^*_a(\cdot)$
be the index function such that 
\begin{equation}\label{eq:oracle-weitzman-index}
    \mathbb E\left[
        \left\{\mu^*(x_t,\omega_{at})-\sigma^*_a(x_t)\right\}^+
        \,\middle|\, x_t
    \right]
    =
    \mathbb E[c_{at}\mid x_t], \forall x_t.
\end{equation}
Then there exists an optimal policy 
$\pi_t^*$ with the following structure: 
At any period $t$ with context $x_t$, it orders the boxes $\{(1),(2),\ldots,(A)\}$ such that
$\sigma^*_{(1)}(x_t) \ge \cdots \ge \sigma^*_{(A)}(x_t)$, sets
$\sigma^*_{(A+1)}(x_t):=-\infty$, 
queries boxes in this order, stops at the $k$-th query if $\max_{1\le j\le k}
\mu^*(x_t,\omega_{(j)t}) \ge \sigma^*_{(k+1)}(x_t)$,
and selects $a_t^*\in\argmax_{a\in \mathcal A_t^*}\mu^*(x_t,\omega_{at})$, where $\mathcal A_t^*$ is the set of boxes queried by the oracle. 
Moreover, the oracle value satisfies
\begin{equation}\label{eq:max:utility}
    U_t(\pi_t^*;\mathcal{F}_{t-1},x_t)
    =
    \mathbb E\left[
        \max_{a\in[A]}
        \min\left\{
            \mu^*(x_t,\omega_{at}),
            \sigma_a^*(x_t)
        \right\}
        \,\middle|\, x_t,\mathcal{F}_{t-1}
    \right].
\end{equation}
\end{proposition}
Define $\sigma_{at}^*:=\sigma_a^*(x_t)$. When the costs ${c_{at}}$ are known constants (i.e., $c_{at} \equiv c_a$ 
for each $a \in [A]$), the indices $\{\sigma_{at}^*\}_{a\in[A]}$ are optimal by the seminal work of \citet{weitzman1979optimal}. When costs are stochastic (as in our motivating setting), the result above follows from a minor adaptation of the proof in \citet{kleinberg2016descending} (see Appendix~\ref{appendix:proof:prop:1:2}).

Proposition \ref{prop:oracle-weitzman} formally identifies the reward evaluator $\mu^*(x,\omega)$ 
and the reservation index function $\sigma_a^*(x)$ as the two objects required to make the optimal sequential decisions. We next describe the online policy as Algorithm \ref{alg:DOW}. At a high level, the policy follows the oracle structure in
Proposition~\ref{prop:oracle-weitzman}, but replaces the unknown reward evaluator
and reservation indices with optimistic estimates.
We call the resulting policy
Double Optimistic Weitzman's rule (DOW). 
\begin{remark}
Our policy estimates the reservation index function $\sigma_{a}^*(x_t)$ through the conditional moment equation \eqref{eq:oracle-weitzman-index} in Proposition~\ref{prop:oracle-weitzman}. Note that equation \eqref{eq:oracle-weitzman-index} characterizes the index at the population level. In our setting, however, the output-cost pair $(\omega_{at},c_{at})$ is observed only when box $a$ is queried, that is, when $a\in\mathcal{A}_t$. 
Hence, a direct empirical analogue of  \eqref{eq:oracle-weitzman-index} cannot use unqueried boxes, whose outputs and costs are not observed. Nevertheless, conditioning on  box $a$ being queried, the same moment equation continues to identify the reservation index. That is, we have
\begin{equation}\label{eq:known-theta-queried-moment}
\mathbb{E}\left[\{ \mu^*(x_t,\omega_{at})-\sigma_{a}^*(x_t)\}^{+}-c_{at}|x_t,\mathcal{F}_{t-1}, a\in\mathcal{A}_t\right]=0,\quad\forall a 
\mbox{ such that } P(a \in \mathcal{A}_t |x_t,\mathcal{F}_{t-1})>0.
\end{equation}
This is because conditional on $x_t,\mathcal{F}_{t-1}$, the query event $a\in\mathcal{A}_t$ is determined before the realized output-cost pair $(\omega_{at},c_{at})$ is observed, and hence is independent of this pair. 
Thus although outputs and costs are only observed for queried boxes, these observations continue to identify $\sigma_a^*(x_t)$.
\end{remark}

\subsection{Optimism and Regret Decomposition}
\label{subsec:optimism-regret}
Our \DOW policy constructs optimistic estimates $\widetilde \mu_t(x,\omega)$ and $\widetilde \sigma_{at},\ \forall a\in[A], t\in[T]$. The estimate $\widetilde \mu_t(x,\omega)$ is an optimistic estimate of $\mu^*(x,\omega)$ for a realized output. The estimate
$\widetilde \sigma_{at}$ is an optimistic estimate of the oracle reservation
index $\sigma_a^*(x_t)$ for box $a$ at the current context. The formal
construction of these estimates is postponed to Sections~\ref{sec:known:theta-*}
and~\ref{sec:estimations}, and here we focus only on their roles in the policy. To understand why optimism is crucial in our DOW policy, we next present a theorem showing that, under optimism, the one-period regret decomposes into the reward estimation error for the selected output and the index estimation errors for the queried boxes.
\begin{theorem}[Regret decomposition under optimism]
\label{thm:optimistic-regret-decomposition}
Fix any period $t\in[T]$ and define  $\Delta_t(\widetilde\pi)
    :=
    U_t(\pi_t^*;\mathcal{F}_{t-1}, x_t)-U_t(\widetilde\pi_t;\mathcal{F}_{t-1}, x_t)$
as the period-$t$ conditional regret of the policy
$\widetilde\pi$ induced by Algorithm~\ref{alg:DOW}.
Suppose that, for all $a\in[A]$, $\widetilde\mu_t(x_t,\omega_{at})
    \ge
    \mu^*(x_t,\omega_{at})$, $\widetilde\sigma_{at}
    \ge
    \sigma^*_{at}$. Let $a_t
    \in
    \operatorname*{arg\,max}_{a\in\mathcal A_t}
    \widetilde\mu_t(x_t,\omega_{at})$ be the box selected by Algorithm~\ref{alg:DOW}. Then,
\[
    \Delta_t(\widetilde\pi)
    \le
    \mathbb E\left[
        \widetilde\mu_t(x_t,\omega_{a_t t})
        -
        \mu^*(x_t,\omega_{a_t t})\mid \mathcal F_{t-1},x_t
    \right]
    +
    \mathbb E\left[
        \sum_{a\in\mathcal A_t}
        \left(
            \widetilde\sigma_{at}-\sigma^*_{at}
        \right)\mid \mathcal F_{t-1},x_t
    \right].
\]
\end{theorem}
We provide intuition on why Theorem \ref{thm:optimistic-regret-decomposition} holds.
Consider an auxiliary ``optimistic'' problem in which both its reward estimate $\widetilde\mu_t$ and its index $\widetilde\sigma_{at}$ upper-bound the truth. 
Then \DOW is the optimal policy for the ``optimistic'' problem, and its corresponding optimistic expected value is thus at least as large as the optimal expected value of the true problem. Hence, the loss can be bounded by the discrepancy between the optimistic quantities used by \DOW and the true quantities realized along the same path. This yields exactly two errors: a reward estimation error for the selected output, and an index-estimation error for every queried box.

The decomposition in Theorem \ref{thm:optimistic-regret-decomposition} allows us to leverage the principle of optimism to control our regret. More specifically, 
if a box is queried frequently, its query dataset grows and its index uncertainty decreases. If a box is queried rarely, then it contributes only
rarely to the second term in Theorem~\ref{thm:optimistic-regret-decomposition}. This mirrors the classic self-correcting principle behind UCB analysis, adapted here to govern reservation indices and reward rather than the direct feedback in bandits.

Consequently, Theorem~\ref{thm:optimistic-regret-decomposition} motivates us to construct high-probability confidence bounds for both reward and reservation indices and use them to define the optimistic estimators $\widetilde{\mu}_t(x_t,\omega_{at})$ and $\widetilde{\sigma}_{at}$. The goal is to ensure that, with high probability, these estimators upper bound their population counterparts while remaining sufficiently accurate. Together, optimism and error control provide the key ingredients for deriving the final regret guarantee of the \DOW policy, whose high-level structure is summarized as Algorithm~\ref{alg:DOW}. 
The full implementation, including the computation of optimistic reservation indices and reward estimates, is deferred to Algorithm~\ref{alg:DOW:detailed} in the Appendix.

\begin{algorithm}[!t]
\caption{\DOW Policy Framework}
\label{alg:DOW}
\textbf{Initialization.}
Observe an initial context $x_0$. Query each box $a\in[A]$ once and record the observed output-cost pairs $\{(\omega_{a0},c_{a0})\}_{a\in[A]}$. 

\For{$t=1,2,\ldots,T$}{
    Observe context $x_t$\;

    Construct optimistic reservation index estimates
    $\{\widetilde\sigma_{at}\}_{a\in[A]}$\;
    Sort
    boxes so that $\widetilde\sigma_{(1)t}\ge
        \widetilde\sigma_{(2)t}\ge
        \cdots\ge
        \widetilde\sigma_{(A)t}$. Set $\widetilde\sigma_{(A+1)t}:=-\infty$\;

    \For{$k=1,2,\ldots,A$}{
        Query box $(k)$ and observe its output and cost
        $(\omega_{(k)t},c_{(k)t})$\;

        Compute the optimistic reward estimate
        $\widetilde\mu_t(x_t,\omega_{(k)t})$\;
        
        \textbf{if} $\max_{l\in[k]}
            \widetilde\mu_t(x_t,\omega_{(l)t})\ge \widetilde\sigma_{(k+1)t}$, stop querying and \textbf{break}\;
    }

    Set $\mathcal{A}_t$ as queried boxes. Deploy
    $
        a_t\in
        \argmax_{a\in\mathcal A_t}
        \widetilde\mu_t(x_t,\omega_{at}),
    $
    and observe reward $r_t$\;
}
\end{algorithm}

\section{Results under Known Reward Function}\label{sec:known:theta-*}
To build intuition, we first study the case in which the reward function $\mu^*(x,\omega)$ is known. This setting is practically relevant when the reward model can be estimated offline using a substantially larger historical dataset.  If the offline data are sufficiently rich, the resulting estimate of the reward function $\mu^*(\cdot,\cdot)$ may be accurate enough relative to the $\sqrt{T}$ scale of the online regret analysis, to be treated as fixed in the subsequent online decision problem. Importantly, even with a known reward function, the conditional distribution of the output $\omega_{at}$ given $x_t$ typically remains unknown.

This case therefore captures regimes in which the main online learning challenge is not reward estimation, but learning the reservation indices and managing exploration across boxes. Indeed, once box $a$ is
opened in period $t$, the decision-maker observes its output $\omega_{at}$ and cost $c_{at}$, and can therefore compute the conditional expected reward
$\mu^*(x_t,\omega_{at})$. 
Thus, in Algorithm~\ref{alg:DOW}, we would set $\widetilde{\mu}_t(x_t,\omega_{at})=\mu^*(x_t,\omega_{at}).$ 
The selection rule then chooses the opened box with the largest true conditional expected reward, and the regret component arising from reward estimation disappears. The remaining learning problem is then focused solely on estimating $\sigma_a^*(x_t)$ for each $a$.

This is analogous to the bandit-feedback setting in \citet{atsidakou2024contextual}, in that feedback is obtained only from opened boxes.\endnote{The authors also study a full information setting where the DM observes the rewards of all boxes at the end of each round.}
\citet{atsidakou2024contextual} develop a black-box reduction from contextual Pandora's Box to online regression, using a specially designed linear-quadratic loss to learn all the reservation indices. In contrast, our approach follows the principle of optimism from the UCB literature. As we illustrate, under suitable regularity conditions, it is not necessary to learn all boxes' index functions uniformly well. Instead, it is enough to maintain optimistic index estimates
and refine the estimates of boxes that remain relevant to the query process. We next formalize this intuition. Recall from Section~\ref{sec:model} that we impose the parametric structure on the indices
\begin{equation}\label{eq:index:a}
    \sigma_a^*(x)=\Lambda(\rho_a^\top\psi(x)),\qquad \forall x\in\mathcal{X},
\end{equation}
where $\Lambda$ is a known monotonic link function, and $\psi(\cdot)$ is a known feature map. 
We impose the following regularity condition on the parametric reservation index function \eqref{eq:index:a}:
\begin{assumption}[Regularity conditions for reservation indices]
\label{ass:reg_res_index}
The link function $\Lambda:\mathbb R\to[-1,1]$ is continuously differentiable,
strictly increasing, and $L$-Lipschitz, $\displaystyle\lim_{s\rightarrow-\infty}\Lambda(s)=-1$, $\displaystyle\lim_{s\rightarrow+\infty}\Lambda(s)=1$. Moreover, $\displaystyle\inf_{\rho\in\mathcal B,\;x\in\mathcal X}\Lambda'(\rho^\top\psi(x))\ge \mu_1>0, \|\psi(x)\|_2\le \bar C_\psi\ \text{for all } x\in\mathcal X$. The parameter space $\mathcal B\subset\mathbb R^m$ is convex and compact where $\rho_a\in\mathcal{B}$, with
diameter $d_{\mathcal B}$: $\sup_{\rho,\rho'\in\mathcal B}\|\rho-\rho'\|_2
    \le d_{\mathcal B}$.
\end{assumption}
The range restriction $\Lambda:\mathbb{R}\to[-1,1]$ is consistent with the fact that $c_{at}\in(0,1)$ and $\mu^*\in[0,1]$, which imply $\sigma_{at}^*\in[-1,1]$. In addition, Assumption~\ref{ass:reg_res_index} imposes only mild regularity conditions. It accommodates the linear reservation index specification considered by \citet{atsidakou2024contextual}, which corresponds to taking $\Lambda$ as the identity map on the relevant, strictly bounded domain of $\rho^\top\psi(x)$, while allowing it to smoothly flatten to $\pm1$ as $s\to\pm\infty$. Assumption~\ref{ass:reg_res_index} yields the point identification of $\rho_a^\top\psi(x_t)$:
\begin{proposition}\label{prop:uniqueness:rho:a:star}
Under Assumption \ref{ass:reg_res_index}, $\rho_a^\top\psi(x_t)$ is point identified through \eqref{eq:known-theta-queried-moment} for each $a\in[A]$. 
\end{proposition}
At the beginning of period $t$,
Algorithm~\ref{alg:DOW} estimates $\rho_a$ by $\hat{\rho}_{at}$ and constructs an
upper confidence bound for the linear index $\rho_a^\top\psi(x_t)$. Let
$\mathcal{R}_{at}(x_t)$ denote the corresponding confidence radius, and define $\widetilde{\sigma}_{at}
=
\Lambda\!\left(
    \hat{\rho}_{at}^{\top}\psi(x_t)
    +
    \mathcal{R}_{at}(x_t)
\right)$. Let $\mathcal{E}_{\rho}$ denote the high-probability event on which $\left|
    \hat{\rho}_{at}^{\top}\psi(x_t)
    -
    \rho_a^\top\psi(x_t)
\right|
\leq
\mathcal{R}_{at}(x_t),\
\forall a\in[A],\; t\in[T]$. Then on $\mathcal{E}_{\rho}$, by the monotonicity of $\Lambda$, $\widetilde{\sigma}_{at}
\geq
\Lambda(\rho_a^\top\psi(x_t))
=
\sigma_{at}^*, \ \forall a\in[A],\; t\in[T]$,
allowing us to invoke the regret decomposition of Theorem~\ref{thm:optimistic-regret-decomposition}. 
Recall that $\widetilde{\mu}_t=\mu^*$ in the known-$\mu^*$ case, implying the reward estimation term in
Theorem~\ref{thm:optimistic-regret-decomposition} is zero. Therefore, under the event $\mathcal{E}_{\rho}$,
Theorem~\ref{thm:optimistic-regret-decomposition} implies that the period-$t$ regret satisfies
\begin{equation}\label{eq:knowntheta_regret_structure}
\Delta_t(\tilde{\pi})
\leq
\mathbb{E}\left[
    \sum_{a\in\mathcal{A}_t}
    \left(\widetilde{\sigma}_{at}-\sigma_{at}^*\right)
    \,\bigg|\, x_t,\mathcal{F}_{t-1}
\right],
\end{equation}
where $\mathcal{A}_t$ is the random set of boxes opened by Algorithm~\ref{alg:DOW} in
period $t$. Inequality~\eqref{eq:knowntheta_regret_structure} illustrates that the regret depends only on the index estimation errors of boxes that are actually opened. Thus,
frequently opened boxes yield larger sample sizes that shrink their confidence radii, while rarely opened boxes contribute negligibly to the overall regret. Once we obtain a valid
high-probability confidence radius for the linear index $\rho_a^\top\psi(x_t)$, the cumulative regret can be controlled by the standard elliptical-potential argument used for contextual bandits \citep[e.g.,][]{abbasi2011improved}. However, unlike in standard contextual bandits, an opened box does not produce a noisy
observation of its reservation index $\sigma_{at}^*$ but instead produces a tuple $(x_t,\omega_{at},c_{at})$. Therefore, the index parameter must be
learned indirectly from the defining equation \eqref{eq:known-theta-queried-moment} of the reservation index. 

\subsection{Index Estimation and Loss Function}
We next describe how to estimate the box-specific reservation index function $\sigma_a^*(\cdot)$. Given \eqref{eq:known-theta-queried-moment} and \eqref{eq:index:a}, our estimation problem is moment-based and is thus inspired by the generalized method of moments (GMM). For this purpose, we define the known-$\mu^*$ moment function, for any $x\in\mathcal{X}$, $\omega\in\Omega$, $c\in[0,1]$, and
$\rho\in\mathcal{B}$, as
\begin{equation}\label{eq:known-theta-moment-function}
m^*(x,\omega,c;\rho)
:=
c-\left(\mu^*(x,\omega)-\Lambda(\rho^\top\psi(x))\right)^+ .
\end{equation}
We have
$\mathbb{E}[
    m^*(x_t,\omega_{at},c_{at};\rho_a)
    \,|\,
    \mathcal{F}_{t-1},x_t,\ a\in\mathcal{A}_t]
=0$ by \eqref{eq:known-theta-queried-moment}. Let $\mathcal{S}_{at}:=\{s<t:\ a\in\mathcal{A}_s\},\ n_{at}:=|\mathcal{S}_{at}|$ denote the set and number of past periods in which box $a$ was queried before period $t$. Order the elements in $\mathcal{S}_{at}$ as
$\tau_{a,1}<\tau_{a,2}<\cdots<\tau_{a,n_{at}}$. As we discuss below, our estimation problem is a special case of the GMM framework with conditional moment conditions. 
By setting the moment conditions based on the features $\psi(x_s)$, 
we can formulate it as an $M$-estimation problem that is more suitable for our online learning setting.
A standard GMM approach would instead form the empirical vector moment $\bar g_{at}(\rho)
    :=
    \frac{1}{n_{at}}
    \sum_{s\in\mathcal{S}_{at}}
    \psi(x_s)m^*(x_s,\omega_{as},c_{as};\rho)$ and minimize a quadratic criterion such as
$\bar g_{at}(\rho)^\top W_{at}\bar g_{at}(\rho)$ for some positive definite weighting matrix $W_{at}$ at each period $t$. However, this approach focuses on estimating the coefficients $\rho_a$, 
which requires a full-rank condition on the local GMM curvature.
Such a requirement can be overly stringent in our setting, as it requires that the observations in which box $a$ is queried contain sufficiently many ``active'' samples, namely those satisfying $\mu^*(x_t,\omega_{at})>\Lambda(\rho_a^\top\psi(x_t))$, 
and that the corresponding feature vectors $\psi(x_t)$ span the parameter space in a well-conditioned way conditional on $a\in\mathcal{A}_t$.

Our objective, however, is not to recover the entire vector $\rho_a$. For the UCB analysis, it is enough to construct a confidence interval for the scalar quantity $\rho_a^\top\psi(x_t)$ at the current context. 
We therefore adopt an $M$-estimation formulation, whose first-order condition recovers the sample moment. For a queried sample of box $a$ in period $s$, define
\begin{equation}\label{eq:known-theta--loss}
    \ell_{as}(\rho)
    :=
    \int_{0}^{\rho^\top\psi(x_s)}
    \left[
        c_{as}
        -
        \left(\mu^*(x_s,\omega_{as})-\Lambda(u)\right)^+
    \right]du.
\end{equation}
Note that $\ell_{as}(\rho)$ is convex in $\rho$, as $\Lambda(\cdot)$ is strictly increasing. Also, the derivative of
$\ell_{as}$ satisfies
\begin{equation}\label{eq:known-theta-loss-gradient}
\nabla_{\rho}\ell_{as}(\rho)=\psi(x_s)\left[
c_{as}
-
\left(\mu^*(x_s,\omega_{as})-\Lambda(\rho^\top\psi(x_s))\right)^+
\right]
= \psi(x_s)m^*(x_s,\omega_{as},c_{as};\rho).
\end{equation}
Therefore, the first-order condition of the empirical loss minimization is equivalent
to the sample analogue of the vector moment equation. The known-$\mu^*$
estimator of $\rho_a$ for period $t$ is
\begin{equation}\label{eq:known-theta-rho-estimator}
    \hat{\rho}_{at}
    \in
    \argmin_{\rho\in\mathcal{B}}
    \sum_{s\in\mathcal{S}_{at}}
    \ell_{as}(\rho).
\end{equation}
We next state the regularity condition that yields curvature of the population loss.
The condition requires that $\rho^\top\psi(x)$ remains in a bounded
region, and for every point in this region, there is some probability mass that the reward of the output is above its corresponding reservation index.

\begin{assumption}[Local mass around reservation thresholds]
\label{ass:losscurvature}
There exist constants $\bar\iota>0$ and $\kappa>0$ such that, for all
$x\in\mathcal{X}$ and $\rho\in\mathcal{B}$, $|\psi(x)^\top\rho|\leq \bar\iota$,
and for all $a\in[A]$, $x\in\mathcal{X}$, $|u|\leq \bar\iota$, we have $\mathbb{P}_{\omega\sim p_a(\cdot\mid x)}
    \left(
        \mu^*(x,\omega)>\Lambda(u)
    \right)
    \geq \kappa $.
\end{assumption}
Assumption~\ref{ass:losscurvature} rules out the cases in which the reservation index lies in a region of the reward with no probability mass above it. Together with the lower derivative
bound $\Lambda'(\cdot)\geq \mu_1$ from Assumption~\ref{ass:reg_res_index}, it ensures that the
reservation index moment crosses zero with slope bounded away from zero. This is the
one-dimensional source of curvature for the loss function.
\begin{lemma}[Population curvature of the loss function]\label{lem:known-theta-loss-curvature}
Under Assumptions~\ref{ass:reg_res_index}, \ref{ass:losscurvature}, for any $t\in[T]$, $a\in[A]$ and $\rho\in\cB$,
\begin{equation}\label{eq:glm-population-curvature}
\mathbb{E}[\ell_{at}(\rho)-\ell_{at}(\rho_a)\mid x_t,\mathcal{F}_{t-1}]
\ge \frac{\kappa\mu_1}{2}\bigl(\psi(x_t)^\top(\rho-\rho_a)\bigr)^2.
\end{equation}
\end{lemma}
The curvature bound in Lemma~\ref{lem:known-theta-loss-curvature} illustrates the advantage of using the $M$-estimation formulation. 
It shows that the population loss increases quadratically in the prediction
error $\psi(x_t)^\top(\rho-\rho_a)$. 
As a preview of the subsequent analysis, we will compare this quadratic curvature against the empirical fluctuation of the additive loss
$\sum_{s\in\mathcal{S}_{at}}\ell_{as}(\rho)$. 
By controlling this fluctuation uniformly over $\rho\in\mathcal{B}$ via martingale concentration, we can absorb it into the curvature term. This procedure yields a high-probability bound on $|(\hat{\rho}_{at}-\rho_a)^\top\psi(x_t)|$, providing the confidence radius needed to construct the optimistic index $\widetilde{\sigma}_{at}$ in \eqref{eq:knowntheta_regret_structure}.

\subsection{Confidence Bound for a Fixed Box}\label{subsection:indices:known:theta}
Next, we construct confidence radii for the reservation indices. Because $\Lambda(\cdot)$ is monotone and $\sigma_{at}^*=\Lambda(\rho_a^\top\psi(x_t))$, for each fixed box $a \in [A]$ we focus on constructing a confidence interval for the scalar index $\rho_a^\top\psi(x_t)$ at the realized context $x_t$. We define $H_{at}(\rho)$, the empirical fluctuation of the loss at $\rho$, as 
\begin{equation}\label{eq:H:at:rho:known:theta}
H_{at}(\rho):=\sum_{s \in \mathcal{S}_{at}}\left\{\ell_{as}(\rho)-\ell_{as}(\rho_a)-\mathbb{E}\!\left[\ell_{as}(\rho)-\ell_{as}(\rho_a)\mid \mathcal{F}_{s-1},x_s\right]\right\}.
\end{equation}
By optimality of $\hat{\rho}_{at}$, the empirical excess loss
$\sum_{s\in\mathcal S_{at}} \left\{\ell_{as}(\hat{\rho}_{at}) - \ell_{as}(\rho_a) \right\}$ is non-positive. 
This sum can be decomposed into its conditional expectation and empirical fluctuation as
$$
\sum_{s\in\mathcal S_{at}}
\left\{\ell_{as}(\hat{\rho}_{at})-\ell_{as}(\rho_a)\right\}
=
    \sum_{s\in\mathcal S_{at}}
    \mathbb E\left[\ell_{as}(\hat{\rho}_{at})-\ell_{as}(\rho_a)
        \,\middle|\,
        \mathcal F_{s-1},x_s
    \right]+ H_{at}(\hat{\rho}_{at}).
$$
Lemma \ref{lem:known-theta-loss-curvature} implies that the conditional expectation is bounded below by the quadratic form
$\kappa\mu_1/2
\sum_{s\in\mathcal S_{at}}
\left\{\psi(x_s)^\top(\hat{\rho}_{at}-\rho_a)\right\}^2$.
Consequently, to establish a confidence interval for $\rho_a^\top \psi(x_t)$, it suffices to bound the empirical fluctuation $H_{at}(\hat{\rho}_{at})$ using a comparable quadratic term and a logarithmic complexity term, as established in the following lemma.

\begin{lemma}\label{lemma:H:empirical:process}
Suppose Assumptions~\ref{ass:reg_res_index}, \ref{ass:losscurvature} hold. 
Fix $a\in[A]$, and given any constant $c_0>0$, with probability at least $1-\delta/3$, uniformly over all $t\in[T]$, for any $\rho_t$ adapted to $\mathcal{F}_{t-1}$
we have 
$$-H_{at}(\rho_t) \leq\frac{c_0}{8}\sum_{s \in \mathcal{S}_{at}}\{(\rho_t-\rho_a)^\top\psi(x_{s})\}^2+\left(\frac{144}{c_0}+C_0\right)(m\log T+\log(T/\delta)),$$
where $C_0>0$ is a constant depending on $d_{\mathcal{B}},\bar{C}_{\psi},\overline{\iota}$. 
\end{lemma}
Lemma \ref{lemma:H:empirical:process} is derived through standard concentration arguments. 
For a fixed $\rho$, $H_{at}(\rho)$ form a martingale difference sequence whose conditional variance is controlled by the
quadratic term that appears in the population curvature lemma. A peeling
argument over the size of this quadratic term and an $\epsilon$-net argument over $\mathcal B$ yield the uniform bound, which therefore also holds for adapted $\rho_t$. 

We now build the confidence interval for the scalar index $\rho_a^\top\psi(x_t)$ using Lemmas \ref{lem:known-theta-loss-curvature} and \ref{lemma:H:empirical:process}. Define the regularized
empirical design matrix
\begin{equation}
\label{eq:known-theta-V}
    V_{at}(\eta_1)
    :=
    \eta_1I_m
    +
    \sum_{s\in\mathcal{S}_{at}}
    \psi(x_s)\psi(x_s)^\top ,
\end{equation}
where $\eta_1>0$ is a regularization parameter and $I_m$ is the \(m\times m\) identity matrix. Intuitively, $V_{at}(\eta_1)$ summarizes the contexts in which box $a$ has been queried before period $t$. The next proposition establishes a high-probability bound on $|(\hat{\rho}_{at}-\rho_a)^\top \psi(x_t)|$.
\begin{proposition}\label{prop:rho:a:confidence:theta:known}
Suppose Assumptions~\ref{ass:reg_res_index}, \ref{ass:losscurvature} hold. When $\mu^*$ is known, given any $\delta>0$, with probability at least $1-\delta$, uniformly over all $a\in[A]$ and $t\in[T]$, we have 
\begin{equation}\label{eq:index:error:theta:known}
    |(\hat{\rho}_{at}-\rho_a)^\top\psi(x_t)|\leq B_{at}^*\|\psi(x_t)\|_{V_{at}(\eta_1)^{-1}},
\end{equation}
where $B_{at}^*=\sqrt{C_1[m\log(T)+\log(TA/\delta)]+\eta_1d_{\mathcal{B}}^2}$, and $C_1$ is a constant depending on $\kappa,\mu_1,\bar{\iota},\bar{C}_{\psi},d_{\mathcal{B}}$ defined in Assumptions~\ref{ass:reg_res_index}, \ref{ass:losscurvature}. 
\end{proposition}
Proposition~\ref{prop:rho:a:confidence:theta:known} 
has the usual online-learning interpretation \citep[see, e.g.,][]{abbasi2011improved,lattimore2020bandit} as the
confidence bounds used in bandit UCB algorithms. 
The confidence bound is large when the current context points in a direction that has not been well explored for box $a$, and small when the historical opened-box contexts provide enough information in that direction. 

Motivated by Proposition~\ref{prop:rho:a:confidence:theta:known}, we define the optimistic
index in the known-$\mu^*$ case as
\begin{equation}\label{eq:tilde:sigma:theta:known}
\widetilde{\sigma}_{at}
=
\Lambda\left(
    \hat{\rho}_{at}^\top\psi(x_t)
    +B_{at}^*
    \|\psi(x_t)\|_{V_{at}(\eta_1)^{-1}}
\right),
\end{equation}
Since $\Lambda$ is increasing, Proposition~\ref{prop:rho:a:confidence:theta:known} implies $\widetilde{\sigma}_{at}\geq \sigma_{at}^*$ uniformly over $a\in[A], t\in[T]$ with probability at least $1-\delta$. This, in turn, is used to establish the uniform optimism event required for the regret guarantee. In implementation, one may calibrate this radius using a bootstrap approximation to the distribution of $(\hat{\rho}_{at}-\rho_a)^\top\psi(x_t)$, which may yield less conservative confidence bounds. Such a bootstrap procedure is not covered by the present theory and proving its validity would require a uniform bootstrap approximation under adaptive sampling. 

\subsection{Regret under Known Reward Function}
Now we bound the cumulative regret when $\mu^*$ is known so that $\widetilde{\mu}_t=\mu^*$. When the
indices are optimistic, i.e., $\widetilde{\sigma}_{at}\geq\sigma_{at}^*, \forall a\in[A],\ t\in[T]$, Theorem~\ref{thm:optimistic-regret-decomposition} gives $\Delta_t(\tilde{\pi})
\leq
\mathbb{E}\left[
    \sum_{a\in\mathcal{A}_t}
    \left(\widetilde{\sigma}_{at}-\sigma_{at}^*\right)
    \,\middle|\, x_t,\mathcal{F}_{t-1}
\right]$. Using the Lipschitz continuity of $\Lambda$ and the confidence bound in
Proposition~\ref{prop:rho:a:confidence:theta:known}, we obtain $\widetilde{\sigma}_{at}-\sigma_{at}^*\leq2LB_{at}^*\|\psi(x_t)\|_{V_{at}(\eta_1)^{-1}}$ on the high-probability event, where $B_{at}^*\leq\mathrm{O}(\sqrt{m\log(T)+\log(TA)})$. Thus the cumulative regret is controlled by
$\sum_{t=1}^T\sum_{a\in\mathcal{A}_t}\|\psi(x_t)\|_{V_{at}(\eta_1)^{-1}}$.
Every time box $a$ is queried, the corresponding
feature vector $\psi(x_t)$ is added to its design matrix. Therefore, the uncertainty in the same
direction decreases over time. The standard elliptical-potential argument implies that for
each box $a\in[A]$,
\begin{equation}\label{eq:bound:index:known:theta}
    \sum_{t=1}^TB_{at}^*\|\psi(x_t)\|_{V_{at}(\eta_1)^{-1}}
    \mathbb{I}\{a\in\mathcal{A}_t\}
    \leq
    \widetilde{O}(m\sqrt{T}).
\end{equation}
Summing this bound over the $A$ boxes yields the following regret guarantee:
\begin{proposition}\label{prop:cumulative:regret:theta:known}
    Suppose Assumptions \ref{ass:reg_res_index}, \ref{ass:losscurvature} hold. When the reward function $\mu^*$ is known, we have $\displaystyle\mathbb{E}\bigg[\sum_{t=1}^T\Delta_t(\tilde{\pi})\bigg]\leq\widetilde{O}\left(Am\sqrt{T}\right)$.
\end{proposition}
Compared with the $\widetilde O(T^{5/6})$ bound of \citet{atsidakou2024contextual} for a more general contextual Pandora’s Box setting, our $\widetilde O(\sqrt T)$ rate relies on the additional local-curvature structure in Assumption~\ref{ass:losscurvature}, which enables UCB-style confidence control for the reservation indices. 
\begin{remark}\label{rmk:x:predictive}
The bound in \eqref{eq:bound:index:known:theta} is pathwise and therefore Proposition~\ref{prop:cumulative:regret:theta:known} does not require $x_t$ to be independent. 
It holds for any sequence of contexts, including $x_t$ that are predictable with respect to $\mathcal{F}_{t-1}$. 
The independence assumption is, however, needed for the analysis when the reward function is unknown.
\proofend
\end{remark}

\section{Results under Unknown Reward Function}\label{sec:estimations}
We now drop the assumption that $\mu^*$ is known and study the problem where $\mu^*$ must be learned online. The preceding known-$\mu^*$ result guides the analysis when $\mu^*$ is unknown. Relative to the previous section, two additional challenges arise: (i) The reward estimation term in Theorem~\ref{thm:optimistic-regret-decomposition} is no longer zero; (ii) the reservation index estimation can no longer use the true conditional expected reward $\mu^*(x_s,\omega_{as})$ when defining the loss. Recall from Section~\ref{sec:model} that we impose the generalized linear parametric structure on the expected reward function:
\begin{equation}\label{eq:reward:function}
    \mu^*(x,\omega)=G(\theta_*^\top \phi(x,\omega)),\qquad\forall (x,\omega)\in\mathcal{X}\times\Omega,
\end{equation}
where $G:\mathbb{R}\rightarrow[0,1]$ is strictly increasing. 
In Section~\ref{sec:estimating:rewards} we address challenge (i) using a standard penalized maximum likelihood estimator (MLE) construction based on \eqref{eq:reward:function}, in Section~\ref{sec:estimating:indices} we address challenge (ii) by showing that the main intuitions from Proposition~\ref{prop:rho:a:confidence:theta:known} continue to work. In particular, Proposition~\ref{prop:confidence:rho} extends Proposition~\ref{prop:rho:a:confidence:theta:known} to this setting by accounting for the plug-in error for the loss. Similarly, the index estimator is controlled by population curvature, empirical concentration, and now an additional perturbation term from reward estimation. Once the reward parameter is learned at the standard online rate, this perturbation is small enough to preserve the overall $\widetilde O(\sqrt T)$ regret rate. 

However, challenge (ii) with unknown reward is more delicate because reward learning uses selected outputs while index estimation is based on queried outputs. Hence controlling the plug-in loss requires an additional assumption ensuring that the reward-design matrix is sufficiently informative for the queried features that enter the index estimator.

We begin by imposing the following regularity condition on the expected reward function:
\begin{assumption}[Regularity conditions for expected reward function]\label{ass:G:phi:regularity}
(i) The function $G:\mathbb{R}\rightarrow [0,1]$ is continuously differentiable and strictly increasing, with Lipschitz constant $L$. Furthermore, $\underline{\mu}:=\inf_{\theta\in\Theta,x\in\mathcal{X},\omega\in\Omega}G'(\theta^\top \phi(x,\omega))>0$, and for some absolute constant $L_1>0$, $|G''(\theta^\top \phi(x,\omega))|\leq L_1$, $\forall x\in\mathcal{X}, \omega\in\Omega, \theta\in\Theta$, where $\theta_*\in\Theta$, and $\Theta\subset\mathbb{R}^d$ is compact and convex. (ii) For an absolute constant $\bar{C}_{\phi},\bar{\alpha}>0$, $\|\phi(x,\omega)\|_2\leq\bar{C}_{\phi}$, $\forall\omega\in\Omega, x\in\mathcal{X}$ and $\|\theta_*\|_2\leq \bar{\alpha}$. 
\end{assumption}
Assumption~\ref{ass:G:phi:regularity} is standard in the generalized linear contextual bandit literature \citep[e.g.,][]{filippi2010parametric,li2017provably,zhou2019learning,lee2024unified}. The Lipschitz constant $L$ controls how errors in $\hat\theta_{t-1}$ propagate into the reward prediction, the lower bound $\underline\mu$ on $G'$ ensures uniform strong convexity of the objective, and the boundedness of $\phi$ and $\theta_*$ keeps the analysis in the standard regime. Common examples of $G$ include any known cumulative distribution function (CDF) such as the logistic and probit CDFs. 

We next impose an assumption on the random reward observed by the DM: 
\begin{assumption}\label{ass:reward:perturbed}
Given any $x\in\mathcal{X}$ and $\omega\in\Omega$, conditioning on the pair of input context and output of the selected API $(x_t,\omega_{a_tt})=(x,\omega)$ during period $t$, the DM observes the random reward $r_t=G(\theta_*^\top\phi(x,\omega))+\zeta_t$, where $\mathbb{E}[\zeta_{t}|\omega_{a_tt},x_t,\mathcal{F}_{t-1}]=\mathbb{E}[\zeta_{t}|\omega_{a_tt},\{\omega_{a't}\}_{a'\neq a_t},x_t,\mathcal{F}_{t-1}]=0$ for any $t\in[T]$ almost surely, and $r_t\in[0,\gamma_0]$ a.s. for some constant $\gamma_{0}>1$.
\end{assumption}
Assumption~\ref{ass:reward:perturbed} requires that the shock $\zeta_t$ depends only on the context $x_t$ and the deployed output; conditional on these, the outputs of the other boxes (queried or not) carry no additional information about $\zeta_t$.

\subsection{Estimating the Reward Parameter}\label{sec:estimating:rewards}
The reward estimation essentially follows a generalized linear contextual bandit framework. At each period $t$ the deployment $(x_t,\omega_{a_tt})$ yields a reward observation $r_t$, and the dataset $\{(x_k,\omega_{a_kk},r_k)\}_{k=1}^{t-1}$ accumulates over time. We estimate $\theta_*$ by a penalized (projected) maximum likelihood estimator
\begin{equation}\label{eq:hat:theta:def}
\hat{\theta}_{t-1}:=\argmin_{\theta\in\Theta}\left\|\sum_{k=1}^{t-1}\!\big\{r_k-G\big(\theta^\top\phi(x_k,\omega_{a_kk})\big)\big\}\phi(x_k,\omega_{a_kk})-\eta_0\theta\right\|_{\Phi_{t-1}^{-1}},\qquad\mbox{where}
\end{equation}
\begin{equation}\label{eq:design:matrix:t}
    \Phi_{t-1}=\eta_0I_d+\sum_{k=1}^{t-1}\phi(x_k,\omega_{a_kk})\phi(x_k,\omega_{a_kk})^\top 
\end{equation}
is the regularized design matrix of reward features collected through period $t-1$, $\eta_0>0$ is a regularization parameter and $I_d$ is the \(d\times d\) identity matrix. This construction follows \citet{filippi2010parametric} and enables us to invoke the self-normalized concentration bound for vector-valued martingales \citep{abbasi2011improved}, thereby controlling the reward-estimation error induced by $\hat{\theta}_{t-1}$. The following standard result provides a high-probability bound on the prediction error induced by using $\hat{\theta}_{t-1}$ to evaluate the reward of a candidate context-output pair $(x_t,\omega_{at})$. 
\begin{lemma}\label{lemma:G:diff:hat:theta:main}
Suppose Assumptions~\ref{ass:G:phi:regularity} and ~\ref{ass:reward:perturbed} hold and fix any $\delta\in(0,\min\{1,2d/e\})$. With probability at least $1-\delta/2$, uniformly over all $a\in[A]$ and $t\geq 2$, $\big|(\hat\theta_{t-1}-\theta_*)^\top\phi(x_t,\omega_{at})\big|\;\leq\;\beta_t\;\big\|\phi(x_t,\omega_{at})\big\|_{\Phi_{t-1}^{-1}}$, where $\beta_t=\mathrm{O}\left(\sqrt{d\log(t)+\log(1/\delta)}\right)$ defined as in~\eqref{eq:beta:t-1}. 
\end{lemma}
This is the ellipsoidal confidence bound for generalized linear bandits, which has a similar structure as Proposition~\ref{prop:rho:a:confidence:theta:known} for the index parameter estimation error bound. The matrix $\Phi_{t-1}$ plays the role that $V_{at}(\eta_1)$ played for the indices, and $\beta_t$ grows logarithmically in $t$. The proof of Lemma~\ref{lemma:G:diff:hat:theta:main} adapts \citet{filippi2010parametric} via a self-normalized martingale inequality. With this bound, we construct the optimistic reward as
\begin{equation}\label{eq:mu:tilde}
\widetilde{\mu}_t(x_t,\omega_{at})\;:=\;G\!\left(\hat\theta_{t-1}^\top\phi(x_t,\omega_{at})+\beta_t\,\big\|\phi(x_t,\omega_{at})\big\|_{\Phi_{t-1}^{-1}}\right).
\end{equation}
Because $G$ is monotonically increasing by Assumption~\ref{ass:G:phi:regularity}, Lemma~\ref{lemma:G:diff:hat:theta:main} implies that with probability at least $1-\delta/2$, uniformly over all $a\in[A]$ and $t\geq 2$, $\widetilde{\mu}_t(x_t,\omega_{at})\;\geq\;\mu^*(x_t,\omega_{at})$.

\subsection{Estimating Indices}\label{sec:estimating:indices}
Next, we estimate the reservation indices and construct their associated confidence intervals for the regime where the reward function is unknown. As an unknown-reward analogue to Section~\ref{subsection:indices:known:theta}, we replace the true reward $\mu^*(x_s,\omega_{as})$ by its plug-in estimate $G(\hat\theta_{t-1}^\top\phi(x_s,\omega_{as}))$ for each $s \leq t$.
The main goal is to show that this substitution preserves the same confidence-bound structure as Proposition~\ref{prop:rho:a:confidence:theta:known}, up to an additional perturbation term controlled by the reward estimation error bound. Fix a period $t$ and a box $a\in[A]$. For any $\rho\in\mathcal B$ and any past queried sample $s\in\mathcal S_{at}$, define the plug-in loss
\begin{equation}\label{eq:plugin:loss}
\hat{\ell}_{as,t}(\rho)
:=
\int_{0}^{\rho^\top\psi(x_s)}
\left[
c_{as}
-
\left\{
G(\hat\theta_{t-1}^\top\phi(x_s,\omega_{as}))
-
\Lambda(u)
\right\}^{+}
\right]du .
\end{equation}
This is the same loss as \eqref{eq:known-theta--loss}, with $\mu^*$ replaced by the current reward estimate. The plug-in index estimator is
\begin{equation}\label{eq:hat:rho:at}
\hat\rho_{at}
\in
\argmin_{\rho\in\mathcal B}
\sum_{s\in\mathcal S_{at}}
\hat{\ell}_{as,t}(\rho).
\end{equation}
The first-order condition of \eqref{eq:hat:rho:at} is the empirical analogue of the Weitzman moment equation with the reward function evaluated at $\hat\theta_{t-1}$. The new challenge compared to Section~\ref{subsection:indices:known:theta} is that the empirical loss in \eqref{eq:hat:rho:at} is not the oracle loss. To isolate this difference, define the plug-in perturbation
\begin{equation}\label{eq:J:at:main}
J_{at}(\rho)
:=
\sum_{s\in\mathcal S_{at}}
\left\{
\left[\hat{\ell}_{as,t}(\rho)-\ell_{as}(\rho)\right]
-
\left[\hat{\ell}_{as,t}(\rho_a)-\ell_{as}(\rho_a)\right]
\right\},
\end{equation}
where $\ell_{as}(\rho)$ is the known-$\mu^*$ loss in \eqref{eq:known-theta--loss}. The term $J_{at}(\rho)$ measures how the estimation error of $\hat\theta_{t-1}$ propagates into the index-estimation loss.

We impose the following additional assumption for the unknown-reward case. It ensures that the reward-design matrix contains enough information in all directions needed to control the plug-in perturbation. We use $\mathrm{polylog}(T)$ to denote a quantity bounded by a polynomial in $\log T$ (i.e. $\mathrm{polylog}(T) = \mathrm{O}((\log T)^c)$ for some constant $c>0$). 
\begin{assumption}[Anti-concentration of reward features]\label{ass:projection:anti:concentration}
There exists a sequence $M_T\geq 1$ with
$M_T=\sqrt{\mathrm{polylog}(T)}/A$ such that for any
$t\in[T]$, $a\in[A]$, $v\in\mathbb S^{d-1}$ and $0<\epsilon\le (2AM_T)^{-1}$, $\mathbb P\!\left(
|v^\top \phi(x_t,\omega_{at})|\le \epsilon\right)\le
M_T\epsilon
\ \text{a.s.}$, where $\mathbb S^{d-1}$ is the unit sphere in $\mathbb{R}^d$.
\end{assumption}
Recall that the contexts $x_t$ are independent across time, implying that the distribution of the reward feature $\phi(x_t,\omega_{at})$ is independent of $\mathcal{F}_{t-1}$. 
Assumption~\ref{ass:projection:anti:concentration} leverages this independence to impose an anti-concentration condition that rules out degenerate reward features.
Informally, every one-dimensional projection of $\phi(x_t,\omega_{at})$ must have nontrivial variation near zero. This condition is imposed because $\theta_*$ is learned only from selected outputs, whereas the index loss for box $a$ uses queried outputs. Thus, the reward features must be sufficiently informative for evaluating the queried features that enter the index estimator. For example, it is satisfied when the embedding distribution is a truncated Gaussian, a truncated multivariate Student's $t$, or a uniform distribution on a hyperrectangle; see Lemmas~\ref{lemma:example:gaussian:assumption}--\ref{lemma:example:uniform} in Appendix~\ref{appendix:minimum:eigenvalue} in the electronic companion.
\begin{remark}\label{rmk:x:t:unknown:reward}
Our regret analysis can be extended to the case where $x_t$ is predictable with respect to $\mathcal{F}_{t-1}$ with a more restrictive version of Assumption \ref{ass:projection:anti:concentration}.
Specifically, it would require that under the same conditions as stated in Assumption \ref{ass:projection:anti:concentration}, we have
\begin{equation}\label{eq:ext_anti_concent}
\mathbb P\!\left(
|v^\top \phi(x_t,\omega_{at})|\le \epsilon \, \mid \, \mathcal F_{t-1}, x_t \right)\le
M_T\epsilon
\qquad \text{a.s.}
\end{equation}

However, such an assumption implies structural restrictions on the reward features. For example, suppose $G$ is the identity map, $\phi(x_t,\omega_{at})=(x_t,\omega_{at})$ and for simplicity, suppose both $x_t$ and $\omega_{at}$ are one-dimensional. If $x_t$ is predictable from $\mathcal F_{t-1}$, then conditional on $\mathcal F_{t-1}$ the context coordinate is fixed. In particular, if $x_1$ follows a Bernoulli distribution with $\mathbb{P}(x_1=1)=0.5$ and $x_t=x_{t-1}$ for $t\ge2$, then the projection along the context direction $v=(1,0)$ is identically zero when $x_1=0$, so \eqref{eq:ext_anti_concent} fails.
Thus, for general predictive $x_t$, even a simple linear embedding can fail to satisfy the anti-concentration property. \proofend
\end{remark}
The following lemma bounds the additional perturbation caused by using $\hat\theta_{t-1}$ in the index loss.
\begin{lemma}[Plug-in loss perturbation]\label{lemma:J:at:bound}
Suppose Assumptions~\ref{ass:reg_res_index} and~\ref{ass:G:phi:regularity} hold. Then, for any $c_0>0$,
$$-J_{at}(\hat\rho_{at})\le
\frac{c_0}{8}
\sum_{s\in\mathcal S_{at}}
\left\{
\psi(x_s)^\top(\hat\rho_{at}-\rho_a)
\right\}^2 +
\frac{2L^2}{c_0}
\sum_{s\in\mathcal S_{at}}
\left\{
(\hat\theta_{t-1}-\theta_*)^\top
\phi(x_s,\omega_{as})
\right\}^2.$$
\end{lemma}
Lemma~\ref{lemma:J:at:bound} shows why Assumption~\ref{ass:projection:anti:concentration} is needed. The first term on the right-hand side has the same quadratic form as the population curvature in Lemma~\ref{lem:known-theta-loss-curvature} and can therefore be absorbed into the curvature argument. The second term is the reward estimation error evaluated on queried outputs. By Lemma~\ref{lemma:G:diff:hat:theta:main}, this term is controlled by the inverse-design norms
$\|\phi(x_s,\omega_{as})\|_{\Phi_{t-1}^{-1}}$, where $\Phi_{t-1}$ is constructed from previously selected outputs, while $\phi(x_s,\omega_{as})$ corresponds to queried outputs used for estimating the index of box $a$. Thus, there is a potential mismatch between the selected outputs that inform reward learning and the queried outputs that enter index estimation. Assumption~\ref{ass:projection:anti:concentration} rules out severe mismatches of this form by ensuring that $\Phi_{t-1}$ is sufficiently well conditioned, so these inverse-design norms do not remain large. 

We now state the confidence bound for the index estimator: 
\begin{proposition}[Index confidence bound with unknown reward function]
\label{prop:confidence:rho:main:context}
Suppose Assumptions~\ref{ass:reg_res_index}--\ref{ass:projection:anti:concentration} 
hold. Fix any $\delta>0$. Then, with probability at least $1-2\delta/3$, uniformly over all
$a\in[A]$ and all
$t\in[T]$,
\begin{equation}\label{eq:index:error:unknown:theta}
\left|
(\hat\rho_{at}-\rho_a)^\top\psi(x_t)
\right|
\le B_{at}
\left\|\psi(x_t)\right\|_{V_{at}(\eta_1)^{-1}},
\end{equation}
where $B_{at}
=
C_2
\sqrt{
m\log(T)
+
\log(AT/\delta)
+
[d\log(t)+\log(1/\delta)]\sum_{s\in\mathcal{S}_{at}}\|\phi(x_{s},\omega_{as})\|_{\Phi_{t-1}^{-1}}^2
}$, $V_{at}(\eta_1)$ is defined as \eqref{eq:known-theta-V}, and $C_2$ is a constant depending on
$\eta_0,\eta_1,\kappa,\mu_1,\bar\iota,d_{\mathcal B},
\bar C_\psi,L,\bar C_\phi,\gamma_0,\underline\mu$, and $\bar\alpha$.
\end{proposition}

The proof follows the same localized empirical-process argument as Proposition~\ref{prop:rho:a:confidence:theta:known}. By the optimality of $\hat\rho_{at}$ in \eqref{eq:hat:rho:at}, $\sum_{s\in\mathcal S_{at}}
\left\{
\hat\ell_{as,t}(\hat\rho_{at})
-
\hat\ell_{as,t}(\rho_a)
\right\}\leq0$. Decomposing this into the oracle population curvature, the empirical fluctuation, and the plug-in perturbation gives
$$\sum_{s\in\mathcal S_{at}}
\mathbb E\!\left[
\ell_{as}(\hat\rho_{at})-\ell_{as}(\rho_a)
\,\middle|\,
\mathcal F_{s-1},x_s
\right] + H_{at}(\hat\rho_{at}) + J_{at}(\hat\rho_{at})\le0.$$
The first term is lower bounded by the curvature result in Lemma~\ref{lem:known-theta-loss-curvature}, the second is controlled by the empirical-process bound in Lemma~\ref{lemma:H:empirical:process}, and the third is controlled by Lemma~\ref{lemma:J:at:bound} together with the reward confidence bound in Lemma~\ref{lemma:G:diff:hat:theta:main}. Combining these bounds yields \eqref{eq:index:error:unknown:theta}. The detailed proof is provided in Appendix~\ref{appendix:index:unknown:theta} in the electronic companion. Motivated by Proposition~\ref{prop:confidence:rho:main:context}, define 
\begin{equation}\label{eq:indices:tilde}
\widetilde{\sigma}_{at}
:=
\Lambda\!\left(
\hat\rho_{at}^\top\psi(x_t)
+B_{at}
\left\|\psi(x_t)\right\|_{V_{at}(\eta_1)^{-1}}
\right).
\end{equation}
Since $\Lambda$ is increasing, Proposition~\ref{prop:confidence:rho:main:context} implies that, with probability at least $1-\delta$,
$\widetilde{\sigma}_{at}
\ge
\Lambda(\rho_a^\top\psi(x_t))
=
\sigma_{at}^*$ uniformly over all $a\in[A]$ and
$t\in[T]$.

\subsection{Regret Analysis under Unknown Reward Function}\label{sec:regret}
We now combine the reward and index confidence bounds to control the cumulative regret of \DOW when the reward function is unknown. The preceding subsections established two high-probability optimism events: Lemma~\ref{lemma:G:diff:hat:theta:main} gives optimism for the reward estimate $\widetilde{\mu}_t$, and Proposition~\ref{prop:confidence:rho:main:context} gives optimism for the reservation index estimate $\widetilde{\sigma}_{at}$. 
Therefore, when both events hold, the regret decomposition in Theorem~\ref{thm:optimistic-regret-decomposition} applies and separates the period-$t$ regret into a reward estimation term and an index-estimation term for all $a\in[A]$ and $t\in[T]$. This structure parallels the known-reward analysis in Section~\ref{sec:known:theta-*}. When the reward function is unknown, an additional generalized linear bandit term appears from estimating the reward parameter. 

As shown in Appendix~\ref{appendix:minimum:eigenvalue} in the electronic companion, Assumption~\ref{ass:projection:anti:concentration} implies that, with high probability, $\lambda_{\min}(\Phi_{t-1})
\ge
\frac{t-1}{16\,\mathrm{polylog}(T)}$
uniformly for all $t\ge \tau:=C\mathrm{polylog}(T)\log(Ad)$ for some absolute constant $C$.
This eigenvalue growth allows the reward estimation error in $J_{at}(\rho)$ to be controlled uniformly over $t\geq\tau$, so that the cumulative regret is controlled after $t\geq\tau$. For these initial periods before $\lceil\tau\rceil$, the per-period regret is bounded by $(2A+1)$. Since $\tau=\mathrm{O}(\log(Ad)\mathrm{polylog}(T))$, the cumulative regret up to period $\lceil\tau\rceil$ is at most $\mathrm{O}(A\log(Ad)\mathrm{polylog}(T))$, and is therefore dominated by the regret accumulated over periods $t\ge \tau$.
\begin{theorem}\label{thm:expected:total:regret}
Suppose Assumptions~\ref{ass:reg_res_index}--\ref{ass:projection:anti:concentration} hold. Then $\displaystyle\mathbb{E}\bigg[\sum_{t=1}^T \Delta_t(\tilde{\pi})\bigg]
    \leq
    \widetilde{O}\left([d+A(m+\sqrt{md})]\sqrt{T}\right)$. In particular, the regret bound is asymptotically minimized by choosing $\eta_1
\asymp
\max\left\{m^{1/2},(dm)^{1/4},d^{1/3}\right\},\ \eta_0
\asymp
\min\left\{
\frac{d}{A\sqrt m},
\left(\frac{d}{A\sqrt m}\right)^{2/3}
\right\}$.
\end{theorem}
The terms in Theorem~\ref{thm:expected:total:regret} correspond directly to the two components in Theorem~\ref{thm:optimistic-regret-decomposition}. The term $\widetilde O(d\sqrt T)$ is the reward estimation regret, matching the usual dimension dependence for generalized linear contextual bandits with reward feature dimension $d$. The term $\widetilde O(A(m+\sqrt{dm})\sqrt T)$ is the index-estimation regret, where $A$ is the number of boxes and $m$ is the dimension of the index feature $\psi$. Relative to the known-reward case in Proposition~\ref{prop:cumulative:regret:theta:known}, the additional term $\widetilde{O}(A\sqrt{dmT})$ reflects the effect of reward estimation error on index estimation.
On the high-probability optimism event, Theorem~\ref{thm:optimistic-regret-decomposition} gives
$\displaystyle\mathbb{E}\bigg[\sum_{t=1}^T \Delta_t(\tilde\pi)\bigg]
\leq
\mathbb{E}\bigg[\sum_{t=1}^T
\left\{
\widetilde{\mu}_t(x_t,\omega_{a_t t})-\mu^*(x_t,\omega_{a_t t})
\right\}\bigg]
+
\mathbb{E}\bigg[\sum_{t=1}^T
\sum_{a\in\mathcal A_t}
\left(
\widetilde{\sigma}_{at}-\sigma_{at}^*
\right)\bigg]$.
The first sum is controlled by the reward confidence radius from Lemma~\ref{lemma:G:diff:hat:theta:main} and an elliptical-potential argument \citep{abbasi2011improved} for the reward-design matrix $\Phi_{t-1}$. The second sum is controlled by the index confidence radius from Proposition~\ref{prop:confidence:rho:main:context} and a separate elliptical-potential argument for each box-specific design matrix $V_{at}(\eta_1)$. Thus the unknown-reward analysis adds a standard reward-learning term while preserving the same index-learning rate as in the known-reward benchmark. The detailed proof is provided in Appendix~\ref{appendix:regret:unknown:theta} in the electronic companion. 

\section{Numerical Results}\label{sec:numerical}
In this section, we illustrate the empirical performance of the proposed algorithms. We conduct simulations in three environments (different output distributions), each with two regimes (specific parameters for reward and costs). We compare variants of our proposed policies, FrugalGPT \citep{chenfrugalgpt} and contextual-bandit baselines. In each environment, there are $A=3$ boxes (APIs) and $T=30{,}000$ periods, with contexts $x_t\overset{\mathrm{i.i.d.}}{\sim}\mathrm{Uniform}([-1,1]^{100})$. 
API \(a\) generates $\omega_{at}=s_a(x_t)\xi_{at}, s_a(x)=0.5+\mathcal{L}(\alpha_a^\top x)$,
where $\alpha_a\in\mathbb{R}^{100}$, \(\mathcal{L}(z)=1/(1+e^{-z})\), and 
$\xi_{at}\in\mathbb{R}^{100}$ are drawn independently from $P_\xi$ across boxes and periods. 

The three environments differ in the choice of \(P_\xi\). In environment~(i), \(P_\xi\) is \(\mathrm{Unif}([-1,1]^{100})\); 
in environment~(ii), \(P_\xi\) is \(N(0,I_{100})\), then truncated to the Euclidean
ball of radius $\sqrt{q_{0.99}(\chi_{100}^2)}\approx 11.65$; in environment~(iii),
let \(t_5\) denote the Student-\(t\) distribution with \(5\) degrees of freedom, and \(\sqrt{0.6}(T_1,\ldots,T_{100})^\top\) follows $P_\xi$, where
\(T_j\overset{\mathrm{i.i.d.}}{\sim}t_5\), then truncated to the Euclidean ball
of radius $\sqrt{q_{0.99}\left(0.6\sum_{j=1}^{100}T_j^2\right)}
\approx 13.56$.
Here, \(q_{0.99}(Z)\) denotes the \(99\)th percentile of the random variable \(Z\). The factor \(\sqrt{0.6}\) in environment~(iii) normalizes each coordinate to have unit variance prior to truncation. For the reward function $\mu^*$, we use the linear
feature map $\phi(x,\omega)=Qx+D\omega$,
where \(Q,D\in\mathbb{R}^{10\times 100}\), and define $\mu^*(x,\omega)
=G\!\left(\theta_*^\top\phi(x,\omega)\right)$,
with \(\theta_*\in\mathbb{R}^{10}\), $G(\cdot)$ set to the standard logistic function $\mathcal{L}(\cdot)$ defined earlier. The realized reward from a context-output pair $(x,\omega)$ is set to be drawn from
$\mathrm{Bernoulli}\!\left(\mu^*(x, \omega)\right)$. 
Moreover, the cost function of box $a$ given context $x$ is $c_a(x,\omega)
=
c_{0,a}G(\beta_{0,a}^\top x)
+
c_{1,a}G(\beta_{1,a}^\top|\omega|)$, where \(\beta_{0,a},\beta_{1,a}\in\mathbb{R}^{100}\) and \(|\omega|\)
denotes the component-wise absolute value. 

Within each environment, we consider two regimes. In \emph{Regime A}, we set
\(\alpha_a=0\), \(Q=0\), and \(\beta_{0,a}=\beta_{1,a}=0\). Thus, the output distributions, rewards, 
and costs are context-invariant, 
implying that \(\sigma_a^*(x)\equiv \sigma_a^*\). Consequently, the oracle uses a fixed cascading order of the APIs.
In \emph{Regime~B}, we allow \(\alpha_a\), \(Q\), \(\beta_{0,a}\) and \(\beta_{1,a}\) to be nonzero, so that the reservation indices, and hence the oracle cascading order, can vary with the context. 
The specific choices of \(\alpha_a\), \(Q\), \(\beta_{0,a}\), and \(\beta_{1,a}\) in Regime B, 
as well as those of \(\theta_*\) and \(D\) in both regimes, are described in Section~\ref{sec:sim_results}. We consider these two regimes because one of our baselines, FrugalGPT, uses a fixed cascading order and may therefore perform significantly worse when the oracle order varies with the context.

\paragraph{Computation of the oracle reservation indices.} Although the oracle reservation index depends on the high-dimensional context and
output, its computation can be reduced to three scalar inputs: the
context-dependent component of the reward predictor, the scale of the
projected output shock, and the expected query cost.
To formalize this reduction, for each replication define $\widetilde{\ell}(x):=\theta_*^\top Qx$ and $\bar c_a(x):=
\E_{\xi\sim P_\xi}
\left[c_a\bigl(x,s_a(x)\xi\bigr)\right]$. The oracle reservation index \(\sigma_a^*(x_t)\) is the unique solution
in \(\sigma\) to
\begin{equation}\label{eq:sigma}
\E_{\xi\sim P_\xi}
\left[
\left\{
G\!\left(
\widetilde{\ell}(x_t)
+s_a(x_t)\theta_*^\top D\xi
\right)
-\sigma
\right\}^{+}
\right]
=
\bar c_a(x_t).
\end{equation}
Because  \(D\), \(\theta_*\), and \(P_\xi\) are fixed within each replication,
\(\sigma_a^*(x_t)\) depends on \(a\) and \(x_t\) only through
$\left(
\widetilde{\ell}(x_t),
\widetilde{s}_a(x_t),
\bar c_a(x_t)
\right)$, 
with \(\widetilde{s}_a(x):=s_a(x)\|D^\top\theta_*\|_2.\)
We exploit this three-parameter
representation as follows. For each
replication, we draw a Monte Carlo sample
\(\xi_1,\ldots,\xi_M\overset{\mathrm{i.i.d.}}{\sim}P_\xi\), where
\(M=4\times10^5\). In Regime~A, the three inputs are context-invariant, so
we approximate the expectation in \eqref{eq:sigma} using the Monte Carlo
sample and solve the resulting equation once for each box. In Regime~B, we
discretize the joint range of
\(
(\widetilde{\ell},\widetilde{s}_a,\bar c_a)
\).
At each grid point, we approximate the expectation in \eqref{eq:sigma} and
solve the resulting one-dimensional equation in \(\sigma\). During the
simulation, we approximate \(\sigma_a^*(x_t)\) by trilinearly interpolating
the precomputed solutions at the eight vertices of the grid cell containing
$\left(
\widetilde{\ell}(x_t),
\widetilde{s}_a(x_t),
\bar c_a(x_t)
\right)$.

\subsection{Policies Compared}
We compare the proposed DOW policy with contextual-bandit and FrugalGPT
baselines, and Weitzman-based variants that differ in their use
of optimism. The oracle implements Weitzman's policy by computing the oracle reservation indices specified above.
Its per-period utility serves as the
benchmark in each table, and all regrets are measured relative to it. Before the online phase, each non-oracle policy receives \(T_0\) offline
calibration periods. For every policy other than FrugalGPT, exactly one box is queried in each calibration period, and only that box’s cost, output and realized reward are observed. The boxes are queried in round-robin order, so each box is sampled approximately $T_0/A$  
times.\endnote{Because FrugalGPT needs to tune the cascading order and thresholds for all boxes, we perform a full-information calibration, where all three boxes are queried in every calibration period, and their costs, outputs and rewards are observed.} We consider
\(T_0\in\{0,50,100,1000\}\) where $T_0=0$ is a cold-start setting.

All Weitzman-based policies adopt the same parametric model of indices using the index-feature vector $\psi$. 
For the experiment, we set 
$\psi(x)=(1,x)$ and set the link function as $\Lambda(z)=\tanh(z)$.
In the \emph{unknown-reward} setting, the policy also estimates $\theta_*$ by ridge-penalized logistic maximum likelihood as defined in
\eqref{eq:hat:theta:def}. 
We remark that our theoretical regret guarantees assume that the reservation-index functions are correctly specified by the generalized-linear model. The indices in our numerical studies do not necessarily satisfy this assumption, so the formal guarantees do not apply directly to these experiments. We intentionally retain this simpler specification to examine the empirical robustness of DOW under model misspecification. The results therefore complement, rather than constitute a numerical verification of, the theoretical guarantees.

We compare the following Weitzman-based policies:

\noindent \emph{(i) Static Plug-in Weitzman:} estimates $\hat{\rho}_a$, $\hat{\theta}$ once using the $T_0$ calibration periods and
keeps these estimates fixed throughout the online phase. This policy is omitted when $T_0=0$ since it requires
calibration data.

\noindent \emph{(ii) Online Plug-in Weitzman with known/unknown reward:} updates $\hat{\rho}_a$ and, in the unknown-reward setting,
$\hat{\theta}$ as new observations become available. It applies no
optimism to either $\rho_a$ or $\theta_*$.

\noindent \emph{(iii) DOW with known/unknown reward.}
DOW follows Algorithm~\ref{alg:DOW}. Its ordering and stopping decisions use
the optimistic reservation index $\tanh\!\left(
\hat{\rho}_a^\top\psi(x_t)+c_\rho\|\psi(x_t)\|_{V_{a,t}^{-1}}
\right)$.
In the unknown-reward setting, DOW evaluates a queried box $a$ using the
optimistic reward $G\!\left(
\hat{\theta}_{t-1}^\top\phi(x_t,\omega_{at})+c_\theta
\|\phi(x_t,\omega_{at})\|_{\Phi_{t-1}^{-1}}
\right)$, where \(V_{a,t}\) and \(\Phi_{t-1}\) are the regularized design matrices for
the index and reward regressions, respectively. The hyperparameters
\(c_\rho\) and \(c_\theta\) determine the widths of the corresponding
confidence bounds. In the known-reward setting, \(c_\theta\) is not used.

In addition, we implement a contextual linear bandit and FrugalGPT as follows:
\paragraph{Contextual linear bandit (LinUCB).}
This policy treats each box as an arm and queries exactly one box per period,
yielding the net payoff \(r_{at}-c_{at}\). For each arm \(a\), it fits a
ridge regression of the net payoff on \(\psi(x)\). Let $A_{a,t}
=
I_m+
\sum_{\substack{s<t:\\ a_s=a}}
\psi(x_s)\psi(x_s)^\top$ denote the design matrix and $\hat{\eta}_{a,t}$ the corresponding
coefficient estimate. The \emph{online} version selects $\argmax_{a\in[A]}
\left\{
\hat{\eta}_{a,t}^\top\psi(x_t)
+
c_{\mathrm{b}}\|\psi(x_t)\|_{A_{a,t}^{-1}}
\right\}$
and updates the model after each period, where \(c_{\mathrm{b}}\) controls
the confidence bonus. The \emph{static} version estimates $\{\hat{\eta}_a\}_{a\in[A]}$ using only the \(T_0\) calibration
periods and subsequently selects $\argmax_{a\in[A]}
\hat{\eta}_a^\top\psi(x_t)$ without optimism or further updates.

\paragraph{FrugalGPT.}
We implement an analogous policy of FrugalGPT \citep{chenfrugalgpt} in our setting. 
FrugalGPT uses a fixed cascade consisting of an ordered sequence of boxes $\pi=(a_1,\ldots,a_L)$, $1\leq L\leq A$, so the cascade need not include all available boxes, and acceptance thresholds $\tau_1,\ldots,\tau_{L-1}\in[0,1]$. In period $t$, the boxes are queried sequentially according to $\pi$, with each query incurring its associated cost. At each nonterminal position $j<L$, after observing the response from box $a_j$, the policy stops querying if $\mu^*(x_t, \omega_{a_jt})\geq\tau_j$; otherwise, it proceeds to the next box. If the final box $a_L$ is reached, querying terminates automatically. Upon termination, the policy selects the queried response with the highest expected reward. 
The cascade and its thresholds are selected by exhaustive search on the calibration data. The candidate set contains all ordered subsets of the
three boxes: \(3\) of length one, \(6\) of length two, and \(6\) of length
three, for a total of \(15\) prefixes. For a prefix of length \(L\), each
threshold is selected from the grid $\{0,0.1,\ldots,1\}$, and all \(11^{L-1}\) threshold combinations are evaluated. This yields $3+6(11)+6(11^2)=795$ candidate configurations. We select the configuration with the highest
empirical mean utility over the \(T_0\) calibration periods and deploy it
unchanged throughout the \(T\) online periods.

Our implementation differs from \citet{chenfrugalgpt} in two respects. First, because of access to $\mu^*(\cdot, \cdot)$, the policy operates under the \emph{known-reward} setting in our experiments. Second, the policy proposed in \citet{chenfrugalgpt} returns the response from the last box queried, whereas our implementation selects the response from the queried boxes with the highest reward. Under the known-reward setting, this modification only improves the policy's performance.

\subsection{Simulation Setup and Results}\label{sec:sim_results}
Each instance configuration \((\text{environment},\text{regime},T_0)\)  is evaluated over
\(R=50\) independent replications. At the beginning of each replication, we
draw all fixed parameters using the replication-specific random seed. In both Regimes A and B, we draw
\(\theta_*\in\mathbb{R}^{10}\) uniformly from the sphere $\left\{
\theta\in\mathbb{R}^{10}:\|\theta\|_2=r
\right\}$, where the radius $r$ is calibrated so that $\mathbb{P}\left(
-1\leq\theta_*^\top\phi(x_t,\omega_{at})\leq1
\right)
\approx 0.5$. This calibration ensures that a substantial fraction of $\theta_*^\top\phi(x_t,\omega_{at})$ lies in a region where the derivative of the link function is bounded away from zero, which is qualitatively consistent with Assumption~\ref{ass:G:phi:regularity}.
We generate the matrix $D$ by drawing each of its entries independently from $\mathrm{Uniform}[0,1]$, after which each row is normalized to have unit Euclidean norm. 
In addition, we draw $c_{0,a},c_{1,a}\sim\mathrm{Uniform}[0,u]$ independently for all $a\in[A]$. The upper bound $u$ is calibrated so that, on average, approximately $1.5$ boxes are opened per period.
In regime B, the entries of \(Q\) are drawn independently from \(\mathrm{Uniform}[0,1]\), after which each row is normalized to have unit
Euclidean norm. For each $a\in[A]$, the entries of
\(\alpha_a,\beta_{0,a},\beta_{1,a}\in\mathbb{R}^{100}\) are drawn independently from \(\mathrm{Uniform}[0,1]\), and each vector is similarly normalized. 
All parameters remain unchanged throughout a replication, while the contexts, box outputs, and realized rewards vary across periods.

Within each replication, all policies face the same realizations of contexts, box outputs, and rewards. For the Weitzman-based policies, the
hyperparameters \(c_\rho\) and \(c_\theta\) are selected from $\{0.001,0.1,1.0\}$, as is \(c_{\mathrm{b}}\) for online LinUCB. 
Let \(U_t^\pi\) denote the period-\(t\) utility of policy \(\pi\), defined as
the reward from the accepted box minus the total query cost incurred in that
period, and let \(U_t^*\) denote the corresponding oracle utility. In Tables~\ref{tab:regimeA} and~\ref{tab:regimeB}, corresponding to Regimes~A and~B respectively, we report
the time-averaged regret $\overline{\mathrm{Reg}}_T =
\frac{1}{T}\sum_{t=1}^T
\left(U_t^*-U_t^\pi\right)$, scaled by \(10^3\) for readability, together with $1.96$ standard errors. 
In each table, the panel header provides the oracle's per-period utility (also scaled by
\(10^3\)), its average queries per period, and the parameter $u$ (upper-bound for $c_{0,a}$ and $c_{1,a}$) which is calibrated separately across environments to yield approximately $1.5$
oracle queries per period. For each tuned policy, environment, and value of $T_0$, we report the mean regret achieved using the best-performing hyperparameter in $\{0.001,0.1,1.0\}$. Bold entries
indicate the lowest regret within an information setting (reward known/unknown, value of $T_0$), and ``---'' indicates that a policy is not applicable for that column. 

Several findings stand out.
First, Tables~\ref{tab:regimeA}--\ref{tab:regimeB} show that the Weitzman-based policies generally outperform the baselines across both regimes and all three environments. 
Second, both the known- and unknown-reward variants of DOW achieve low regret even under the cold-start setting with \(T_0=0\). 
Third, in Regime~A, FrugalGPT is comparable to DOW when \(T_0=1000\), but performs substantially worse than the online Weitzman policies at smaller values of $T_0$. In Regime~B, it remains substantially worse even at \(T_0=1000\). This pattern is consistent with the structural limitation of committing to a single fixed cascade when the oracle ordering varies with the context, as documented in Table~\ref{tab:oracle-order-stats}.  
Fourth, the one-shot contextual bandits are not competitive in this
setting. The oracle queries \(1.30\)--\(1.78\) boxes per period on average,
so a one-box policy forfeits the option value of obtaining additional
outputs before making its final selection.  
Fifth, the regret gap between DOW under known- and unknown-reward variants is small, qualitatively consistent with our theoretical analysis under Assumption~\ref{ass:projection:anti:concentration}. Across the online Weitzman's policies, the largest increase in regret is approximately $0.17\%$ of the oracle's utility.
Finally, in Regime~A, DOW has lower mean regret than Online Plug-in Weitzman in all three environments, with the advantage most pronounced when \(T_0\leq100\). In contrast, the differences are generally smaller in Regime~B. One possible explanation is that, in the context-invariant Regime~A, a plug-in policy is more susceptible to early lock-in, whereas DOW’s optimism induces continued exploration. In contrast, contextual variation in Regime~B may generate incidental exploration.

\section{Conclusion}\label{sec:conclusion}
Motivated by LLM cascading, we introduce an online contextual Pandora's Box model in which a decision-maker sequentially queries APIs to generate outputs at a cost and deploys a single output, observing only its downstream reward. This output-mediated feedback departs from the classical setting, where opening a box reveals its reward directly, and captures a key feature of LLM cascading systems: API-specific heterogeneity enters only through the distributions of generated outputs and costs, while downstream value is governed by a shared reward evaluator on the context-output pair. 
Rather than estimating the full conditional output and cost distributions, we impose a generalized linear structure directly on the reservation indices and the shared reward function. In simulations, we further demonstrate the advantages of our \DOW policy over both a contextual-bandit benchmark and the existing FrugalGPT heuristic. Our \DOW policy is built by combining GMM-type estimation of the reservation indices with UCB-style confidence bounds for both the indices and the reward evaluator. Through a regret decomposition under optimism, which separates cumulative regret into reward-estimation errors for deployed outputs and index-estimation errors for queried APIs, we derive a $\widetilde{\mathrm{O}}\big([d+A(m+\sqrt{dm})]\sqrt{T}\big)$ regret bound. Our work opens several avenues for future research. First, because our theoretical results rely on generalized linear specifications, extending the moment-based UCB framework to richer nonparametric or neural representations is a natural next step.
Achieving this while preserving $\widetilde{O}(\sqrt{T})$ regret would further broaden the model's applicability. Second, our current formulation assumes a standard LLM cascading structure where APIs are queried sequentially and a single output is selected and then deployed. Expanding this framework to allow for batched querying under latency budgets, or to support output ensembling and synthesis, would significantly change the query-selection dynamics. Addressing these richer action spaces raises new questions about the structure of the learning policy.

\begin{APPENDICES}
\renewcommand*{\theHsection}{app.\arabic{section}}
\renewcommand*{\theHsubsection}
  {app.\arabic{section}.\arabic{subsection}}
\renewcommand*{\theHsubsubsection}
  {app.\arabic{section}.\arabic{subsection}.\arabic{subsubsection}}

\section{Detailed \DOW Algorithm}\label{appendix:algorithm}
\begin{algorithm}[!htp]
\caption{\DOW Policy Algorithm}
\label{alg:DOW:detailed}
\textbf{Initialization.}
Observe an initial context $x_0$. Query each box $a\in[A]$ once and record the observed output-cost pairs $\{(\omega_{a0},c_{a0})\}_{a\in[A]}$. 

\For{$t=1,2,\ldots,T$}{
    Observe context $x_t$. Initialize $\mathcal A_t=\emptyset$ and $M_t=-\infty$\;
    For any $a\in[A]$, compute $\hat{\rho}_{at}$ by \eqref{eq:hat:rho:at} and set $\widetilde{\sigma}_{at}=\Lambda(\hat{\rho}_{at}^\top\psi(x_t)+B_{at}\|\psi(x_t)\|_{V_{at}(\eta_1)^{-1}})$, where $B_{at}$ is defined as in Proposition~\ref{prop:confidence:rho:main:context}, $V_{at}(\eta_1)$ is defined as \eqref{eq:known-theta-V}\;
    Order the
    boxes so that $\widetilde\sigma_{(1)t}\ge
        \widetilde\sigma_{(2)t}\ge
        \cdots\ge
        \widetilde\sigma_{(A)t}$.
    Set $\widetilde\sigma_{(A+1)t}:=-\infty$\;

    \For{$k=1,2,\ldots,A$}{
        Query box $(k)$ and observe its output and cost
        $(\omega_{(k)t},c_{(k)t})$\;

        Add $(k)$ to the queried set:
        $\mathcal A_t\leftarrow \mathcal A_t\cup\{(k)\}$\;

        Compute the optimistic reward estimate
        $\widetilde\mu_t(x_t,\omega_{(k)t})=G\!\left(\hat\theta_{t-1}^\top\phi(x_t,\omega_{(k)t})+\beta_t\,\big\|\phi(x_t,\omega_{(k)t})\big\|_{\Phi_{t-1}^{-1}}\right)$,
        where $\Phi_{t-1}$ is defined as in \eqref{eq:design:matrix:t} and $\beta_t$ is as given in Lemma~\ref{lemma:G:diff:hat:theta:main}\;
        Update
            $M_t \leftarrow
            \max_{a\in\mathcal A_t}
            \widetilde\mu_t(x_t,\omega_{at})$\;
        \textbf{if} $M_t\ge \widetilde\sigma_{(k+1)t}$, stop querying and \textbf{break}\;
    }

    Select and deploy
    $
        a_t\in
        \argmax_{a\in\mathcal A_t}
        \widetilde\mu_t(x_t,\omega_{at}),
    $
    and observe reward $r_t$\;
}
\end{algorithm}
\end{APPENDICES}

%\clearpage
\renewcommand{\notesname}{Endnotes} 

\begingroup
\parindent 0pt
\parskip 0pt
\def\enotesize{\normalsize}
\theendnotes
\endgroup

\section*{AI Disclosure}
The authors used generative artificial intelligence tools, including GPT, Gemini, and Claude, to edit and polish earlier versions of the drafts.
All AI-generated text and suggestions were checked, revised, and approved by the authors. The authors take full responsibility for the accuracy, integrity, and originality of the submitted work.

\section*{Acknowledgment}
We thank Ali Makhdoumi for helpful discussions on the general theory and intuition of Pandora’s Box, and 
Shreyas Sekar for an early conversation on AI-assisted coding and model selection for reducing LLM API costs that eventually led us to the FrugalGPT paper by \cite{chenfrugalgpt}.
We also thank the participants from 2026 Marketplace Innovation Workshop for comments and discussion.

\bibliographystyle{informs2014}
\bibliography{reference} 

@inproceedings{kleinberg2016descending,
  title={Descending price coordinates approximately efficient search},
  author={Kleinberg, Robert and Waggoner, Bo and Weyl, E Glen},
  booktitle={Extended abstract in the Proceedings of the 17th ACM Conference on Electronic Commerce (EC’16)},
  year={2016}
}

@article{freedman1975tail,
  title={On tail probabilities for martingales},
  author={Freedman, David A},
  journal={the Annals of Probability},
  pages={100--118},
  year={1975},
  publisher={JSTOR}
}

@inproceedings{li2017provably,
  title={Provably optimal algorithms for generalized linear contextual bandits},
  author={Li, Lihong and Lu, Yu and Zhou, Dengyong},
  booktitle={International Conference on Machine Learning},
  pages={2071--2080},
  year={2017},
  organization={PMLR}
}

@article{zhou2019learning,
  title={Learning in generalized linear contextual bandits with stochastic delays},
  author={Zhou, Zhengyuan and Xu, Renyuan and Blanchet, Jose},
  journal={Advances in Neural Information Processing Systems},
  volume={32},
  year={2019}
}

@techreport{tropp2011user,
  title={User-friendly tail bounds for matrix martingales},
  author={Tropp, Joel A},
  year={2011}
}

@article{newey1994large,
  title={Large sample estimation and hypothesis testing},
  author={Newey, Whitney K and McFadden, Daniel},
  journal={Handbook of econometrics},
  volume={4},
  pages={2111--2245},
  year={1994},
  publisher={Elsevier}
}

@article{filippi2010parametric,
  title={Parametric bandits: The generalized linear case},
  author={Filippi, Sarah and Cappe, Olivier and Garivier, Aur{\'e}lien and Szepesv{\'a}ri, Csaba},
  journal={Advances in neural information processing systems},
  volume={23},
  year={2010}
}

@article{lee2024unified,
  title={A Unified Confidence Sequence for Generalized Linear Models, with Applications to Bandits},
  author={Lee, Junghyun and Yun, Se-Young and Jun, Kwang-Sung},
  journal={Advances in Neural Information Processing Systems},
  volume={37},
  pages={124640--124685},
  year={2024}
}

@article{rusmevichientong2010linearly,
  title={Linearly parameterized bandits},
  author={Rusmevichientong, Paat and Tsitsiklis, John N},
  journal={Mathematics of Operations Research},
  volume={35},
  number={2},
  pages={395--411},
  year={2010},
  publisher={INFORMS}
}

@article{jun2017scalable,
  title={Scalable generalized linear bandits: Online computation and hashing},
  author={Jun, Kwang-Sung and Bhargava, Aniruddha and Nowak, Robert and Willett, Rebecca},
  journal={Advances in Neural Information Processing Systems},
  volume={30},
  year={2017}
}

@inproceedings{atsidakou2024contextual,
  title={Contextual pandora’s box},
  author={Atsidakou, Alexia and Caramanis, Constantine and Gergatsouli, Evangelia and Papadigenopoulos, Orestis and Tzamos, Christos},
  booktitle={Proceedings of the AAAI Conference on Artificial Intelligence},
  volume={38},
  pages={10944--10952},
  year={2024}
}

@book{lattimore2020bandit,
  title={Bandit algorithms},
  author={Lattimore, Tor and Szepesv{\'a}ri, Csaba},
  year={2020},
  publisher={Cambridge University Press}
}

@article{liu2025improved,
  title={Improved Regret and Contextual Linear Extension for Pandora's Box and Prophet Inequality},
  author={Liu, Junyan and Chen, Ziyun and Wang, Kun and Luo, Haipeng and Ratliff, Lillian J},
  journal={arXiv preprint arXiv:2505.18828},
  year={2025}
}

@inproceedings{agarwal2024semi,
  title={Semi-bandit learning for monotone stochastic optimization},
  author={Agarwal, Arpit and Ghuge, Rohan and Nagarajan, Viswanath},
  booktitle={2024 IEEE 65th Annual Symposium on Foundations of Computer Science (FOCS)},
  pages={1260--1274},
  year={2024},
  organization={IEEE}
}

@article{abbasi2011improved,
  title={Improved algorithms for linear stochastic bandits},
  author={Abbasi-Yadkori, Yasin and P{\'a}l, D{\'a}vid and Szepesv{\'a}ri, Csaba},
  journal={Advances in neural information processing systems},
  volume={24},
  year={2011}
}

@article{weitzman1979optimal,
  title={OPTIMAL SEARCH FOR THE BEST ALTERNATIVE.},
  author={Weitzman, Martin L},
  journal={Econometrica},
  volume={47},
  number={3},
  year={1979}
}

@inproceedings{gatmiry2024bandit,
  title={Bandit algorithms for prophet inequality and pandora's box},
  author={Gatmiry, Khashayar and Kesselheim, Thomas and Singla, Sahil and Wang, Yifan},
  booktitle={Proceedings of the 2024 Annual ACM-SIAM Symposium on Discrete Algorithms (SODA)},
  pages={462--500},
  year={2024},
  organization={SIAM}
}

@article{hager2024evaluation,
  title={Evaluation and mitigation of the limitations of large language models in clinical decision-making},
  author={Hager, Paul and Jungmann, Friederike and Holland, Robbie and Bhagat, Kunal and Hubrecht, Inga and Knauer, Manuel and Vielhauer, Jakob and Makowski, Marcus and Braren, Rickmer and Kaissis, Georgios and others},
  journal={Nature medicine},
  volume={30},
  number={9},
  pages={2613--2622},
  year={2024},
  publisher={Nature Publishing Group US New York}
}

@article{hao2025large,
  title={Large language model integrations in cancer decision-making: a systematic review and meta-analysis},
  author={Hao, Yuexing and Qiu, Zhiwen and Holmes, Jason and L{\"o}ckenhoff, Corinna E and Liu, Wei and Ghassemi, Marzyeh and Kalantari, Saleh},
  journal={NPJ Digital Medicine},
  volume={8},
  number={1},
  pages={450},
  year={2025},
  publisher={Nature Publishing Group UK London}
}

@article{thirunavukarasu2023large,
  title={Large language models in medicine},
  author={Thirunavukarasu, Arun James and Ting, Darren Shu Jeng and Elangovan, Kabilan and Gutierrez, Laura and Tan, Ting Fang and Ting, Daniel Shu Wei},
  journal={Nature medicine},
  volume={29},
  number={8},
  pages={1930--1940},
  year={2023},
  publisher={Nature Publishing Group US New York}
}

@article{chen2025manager,
  title={A manager and an AI walk into a bar: does ChatGPT make biased decisions like we do?},
  author={Chen, Yang and Kirshner, Samuel N and Ovchinnikov, Anton and Andiappan, Meena and Jenkin, Tracy},
  journal={Manufacturing \& Service Operations Management},
  volume={27},
  number={2},
  pages={354--368},
  year={2025},
  publisher={INFORMS}
}

@article{chen2024large,
  title={Large language model in creative work: The role of collaboration modality and user expertise},
  author={Chen, Zenan and Chan, Jason},
  journal={Management Science},
  volume={70},
  number={12},
  pages={9101--9117},
  year={2024},
  publisher={INFORMS}
}

@incollection{simchi2026large,
  title={Large language models for supply chain decisions},
  author={Simchi-Levi, David and Mellou, Konstantina and Menache, Ishai and Pathuri, Jeevan},
  booktitle={AI in Supply Chains: Perspectives from Global Thought Leaders},
  pages={93--104},
  year={2026},
  publisher={Springer}
}

@inproceedings{yang2023against,
  title={Against opacity: Explainable ai and large language models for effective digital advertising},
  author={Yang, Qi and Ongpin, Marlo and Nikolenko, Sergey and Huang, Alfred and Farseev, Aleksandr},
  booktitle={Proceedings of the 31st ACM International Conference on Multimedia},
  pages={9299--9305},
  year={2023}
}

@article{reisenbichler2025applying,
  title={Applying large language models to sponsored search advertising},
  author={Reisenbichler, Martin and Reutterer, Thomas and Schweidel, David A},
  journal={Marketing Science},
  year={2025},
  publisher={INFORMS}
}

@article{chenfrugalgpt,
  title={FrugalGPT: How to Use Large Language Models While Reducing Cost and Improving Performance},
  author={Chen, Lingjiao and Zaharia, Matei and Zou, James},
  journal={Transactions on Machine Learning Research}, 
  year={2025},
}

@article{hu2024routerbench,
  title={Routerbench: A benchmark for multi-llm routing system},
  author={Hu, Qitian Jason and Bieker, Jacob and Li, Xiuyu and Jiang, Nan and Keigwin, Benjamin and Ranganath, Gaurav and Keutzer, Kurt and Upadhyay, Shriyash Kaustubh},
  journal={arXiv preprint arXiv:2403.12031},
  year={2024}
}

@article{jiang2023llm,
  title={Llm-blender: Ensembling large language models with pairwise ranking and generative fusion},
  author={Jiang, Dongfu and Ren, Xiang and Lin, Bill Yuchen},
  journal={arXiv preprint arXiv:2306.02561},
  year={2023}
}

@article{nie2024online,
  title={Online cascade learning for efficient inference over streams},
  author={Nie, Lunyiu and Ding, Zhimin and Hu, Erdong and Jermaine, Christopher and Chaudhuri, Swarat},
  journal={arXiv preprint arXiv:2402.04513},
  year={2024}
}

@article{mei2025omnirouter,
  title={Omnirouter: Budget and performance controllable multi-llm routing},
  author={Mei, Kai and Xu, Wujiang and Guo, Minghao and Lin, Shuhang and Zhang, Yongfeng},
  journal={ACM SIGKDD Explorations Newsletter},
  volume={27},
  number={2},
  pages={107--116},
  year={2025},
  publisher={ACM New York, NY, USA}
}

@article{chen2024cascade,
  title={Cascade speculative drafting for even faster llm inference},
  author={Chen, Ziyi and Yang, Xiaocong and Lin, Jiacheng and Sun, Chenkai and Chang, Kevin C and Huang, Jie},
  journal={Advances in Neural Information Processing Systems},
  volume={37},
  pages={86226--86242},
  year={2024}
}

@article{zhang2024efficient,
  title={Efficient contextual llm cascades through budget-constrained policy learning},
  author={Zhang, Xuechen and Huang, Zijian and Taga, Ege Onur and Joe-Wong, Carlee and Oymak, Samet and Chen, Jiasi},
  journal={Advances in Neural Information Processing Systems},
  volume={37},
  pages={91691--91722},
  year={2024}
}

@inproceedings{fang2024llm,
  title={Llm-ensemble: Optimal large language model ensemble method for e-commerce product attribute value extraction},
  author={Fang, Chenhao and Li, Xiaohan and Fan, Zezhong and Xu, Jianpeng and Nag, Kaushiki and Korpeoglu, Evren and Kumar, Sushant and Achan, Kannan},
  booktitle={Proceedings of the 47th International ACM SIGIR Conference on Research and Development in Information Retrieval},
  pages={2910--2914},
  year={2024}
}

@inproceedings{hu2025efficient,
  title={Efficient dynamic ensembling for multiple LLM experts},
  author={Hu, Jinwu and Wang, Yufeng and Zhang, Shuhai and Zhou, Kai and Chen, Guohao and Hu, Yu and Xiao, Bin and Tan, Mingkui},
  booktitle={Proceedings of the Thirty-Fourth International Joint Conference on Artificial Intelligence, IJCAI},
  pages={16--22},
  year={2025}
}

@inproceedings{yue2024large,
  title={Large Language Model Cascades with Mixture of Thought Representations for Cost-Efficient Reasoning},
  author={Yue, Murong and Zhao, Jie and Zhang, Min and Du, Liang and Yao, Ziyu},
  booktitle={The Twelfth International Conference on Learning Representations},
  year={2024}
}

@inproceedings{shnitzerlarge,
  title={Large Language Model Routing with Benchmark Datasets},
  author={Shnitzer, Tal and Ou, Anthony and Silva, M{\'\i}rian and Soule, Kate and Sun, Yuekai and Solomon, Justin and Thompson, Neil and Yurochkin, Mikhail},
  booktitle={First Conference on Language Modeling},
  year={2023}
}

@article{hari2023tryage,
  title={Tryage: Real-time, intelligent routing of user prompts to large language models},
  author={Hari, Surya Narayanan and Thomson, Matt},
  journal={arXiv preprint arXiv:2308.11601},
  year={2023}
}

@inproceedings{lu2024routing,
  title={Routing to the expert: Efficient reward-guided ensemble of large language models},
  author={Lu, Keming and Yuan, Hongyi and Lin, Runji and Lin, Junyang and Yuan, Zheng and Zhou, Chang and Zhou, Jingren},
  booktitle={Proceedings of the 2024 Conference of the North American Chapter of the Association for Computational Linguistics: Human Language Technologies (Volume 1: Long Papers)},
  pages={1964--1974},
  year={2024}
}

@inproceedings{vsakota2024fly,
  title={Fly-swat or cannon? cost-effective language model choice via meta-modeling},
  author={{\v{S}}akota, Marija and Peyrard, Maxime and West, Robert},
  booktitle={Proceedings of the 17th ACM International Conference on Web Search and Data Mining},
  pages={606--615},
  year={2024}
}

@article{doval2018whether,
  title={Whether or not to open Pandora's box},
  author={Doval, Laura},
  journal={Journal of Economic Theory},
  volume={175},
  pages={127--158},
  year={2018},
  publisher={Elsevier}
}

@inproceedings{chawla2020pandora,
  title={Pandora's box with correlations: Learning and approximation},
  author={Chawla, Shuchi and Gergatsouli, Evangelia and Teng, Yifeng and Tzamos, Christos and Zhang, Ruimin},
  booktitle={2020 IEEE 61st Annual Symposium on Foundations of Computer Science (FOCS)},
  pages={1214--1225},
  year={2020},
  organization={IEEE}
}

@inproceedings{boodaghians2020pandora,
  title={Pandora's box problem with order constraints},
  author={Boodaghians, Shant and Fusco, Federico and Lazos, Philip and Leonardi, Stefano},
  booktitle={Proceedings of the 21st ACM Conference on Economics and Computation},
  pages={439--458},
  year={2020}
}

@inproceedings{ezra2026contract,
  title={Contract design for sequential actions},
  author={Ezra, Tomer and Feldman, Michal and Schlesinger, Maya},
  booktitle={Proceedings of the 2026 Annual ACM-SIAM Symposium on Discrete Algorithms (SODA)},
  pages={6537--6570},
  year={2026},
  organization={SIAM}
}

@inproceedings{gergatsouli2022online,
  title={Online learning for min sum set cover and pandora’s box},
  author={Gergatsouli, Evangelia and Tzamos, Christos},
  booktitle={International Conference on Machine Learning},
  pages={7382--7403},
  year={2022},
  organization={PMLR}
}

@article{fu2020learning,
  title={Learning utilities and equilibria in non-truthful auctions},
  author={Fu, Hu and Lin, Tao},
  journal={Advances in Neural Information Processing Systems},
  volume={33},
  pages={14231--14242},
  year={2020}
}

@inproceedings{kveton2020randomized,
  title={Randomized exploration in generalized linear bandits},
  author={Kveton, Branislav and Zaheer, Manzil and Szepesvari, Csaba and Li, Lihong and Ghavamzadeh, Mohammad and Boutilier, Craig},
  booktitle={International Conference on Artificial Intelligence and Statistics},
  pages={2066--2076},
  year={2020},
  organization={PMLR}
}

@inproceedings{ding2021efficient,
  title={An efficient algorithm for generalized linear bandit: Online stochastic gradient descent and thompson sampling},
  author={Ding, Qin and Hsieh, Cho-Jui and Sharpnack, James},
  booktitle={International Conference on Artificial Intelligence and Statistics},
  pages={1585--1593},
  year={2021},
  organization={PMLR}
}

@inproceedings{kim2023double,
  title={Double doubly robust thompson sampling for generalized linear contextual bandits},
  author={Kim, Wonyoung and Lee, Kyungbok and Paik, Myunghee Cho},
  booktitle={Proceedings of the AAAI Conference on Artificial Intelligence},
  volume={37},
  pages={8300--8307},
  year={2023}
}

@article{hansen1982large,
  title={Large sample properties of generalized method of moments estimators},
  author={Hansen, Lars Peter},
  journal={Econometrica: Journal of the econometric society},
  pages={1029--1054},
  year={1982},
  publisher={JSTOR}
}

@article{arellano1991some,
  title={Some tests of specification for panel data: Monte Carlo evidence and an application to employment equations},
  author={Arellano, Manuel and Bond, Stephen},
  journal={The review of economic studies},
  volume={58},
  number={2},
  pages={277--297},
  year={1991},
  publisher={Wiley-Blackwell}
}

@article{chamberlain1987asymptotic,
  title={Asymptotic efficiency in estimation with conditional moment restrictions},
  author={Chamberlain, Gary},
  journal={Journal of econometrics},
  volume={34},
  number={3},
  pages={305--334},
  year={1987},
  publisher={Elsevier}
}

@article{cheng2024gmm,
  title={GMM estimation for high-dimensional panel data models},
  author={Cheng, Tingting and Dong, Chaohua and Gao, Jiti and Linton, Oliver},
  journal={Journal of Econometrics},
  volume={244},
  number={1},
  pages={105853},
  year={2024},
  publisher={Elsevier}
}

@article{lin2010gmm,
  title={GMM estimation of spatial autoregressive models with unknown heteroskedasticity},
  author={Lin, Xu and Lee, Lung-fei},
  journal={Journal of Econometrics},
  volume={157},
  number={1},
  pages={34--52},
  year={2010},
  publisher={Elsevier}
}

@article{andrews2022optimal,
  title={Optimal decision rules for weak GMM},
  author={Andrews, Isaiah and Mikusheva, Anna},
  journal={Econometrica},
  volume={90},
  number={2},
  pages={715--748},
  year={2022},
  publisher={Wiley Online Library}
}

@article{hansen2021inference,
  title={Inference for iterated GMM under misspecification},
  author={Hansen, Bruce E and Lee, Seojeong},
  journal={Econometrica},
  volume={89},
  number={3},
  pages={1419--1447},
  year={2021},
  publisher={Wiley Online Library}
}

@article{lai1985asymptotically,
  title={Asymptotically efficient adaptive allocation rules},
  author={Lai, Tze Leung and Robbins, Herbert},
  journal={Advances in applied mathematics},
  volume={6},
  number={1},
  pages={4--22},
  year={1985},
  publisher={Academic Press}
}

@article{auer2002finite,
  title={Finite-time analysis of the multiarmed bandit problem},
  author={Auer, Peter and Cesa-Bianchi, Nicolo and Fischer, Paul},
  journal={Machine learning},
  volume={47},
  number={2},
  pages={235--256},
  year={2002},
  publisher={Springer}
}

@article{fan2025fragility,
  title={The fragility of optimized bandit algorithms},
  author={Fan, Lin and Glynn, Peter W},
  journal={Operations Research},
  volume={73},
  number={6},
  pages={3173--3198},
  year={2025},
  publisher={INFORMS}
}

@inproceedings{garivier2011kl,
  title={The KL-UCB algorithm for bounded stochastic bandits and beyond},
  author={Garivier, Aur{\'e}lien and Capp{\'e}, Olivier},
  booktitle={Proceedings of the 24th annual conference on learning theory},
  pages={359--376},
  year={2011},
  organization={JMLR Workshop and Conference Proceedings}
}

@article{audibert2009exploration,
  title={Exploration--exploitation tradeoff using variance estimates in multi-armed bandits},
  author={Audibert, Jean-Yves and Munos, R{\'e}mi and Szepesv{\'a}ri, Csaba},
  journal={Theoretical Computer Science},
  volume={410},
  number={19},
  pages={1876--1902},
  year={2009},
  publisher={Elsevier}
}

@article{rakhlin2013optimization,
  title={Optimization, learning, and games with predictable sequences},
  author={Rakhlin, Sasha and Sridharan, Karthik},
  journal={Advances in Neural Information Processing Systems},
  volume={26},
  year={2013}
}

@article{mao2025model,
  title={Model-Free Nonstationary Reinforcement Learning: Near-Optimal Regret and Applications in Multiagent Reinforcement Learning and Inventory Control},
  author={Mao, Weichao and Zhang, Kaiqing and Zhu, Ruihao and Simchi-Levi, David and Ba{\c{s}}ar, Tamer},
  journal={Management Science},
  volume={71},
  number={2},
  pages={1564--1580},
  year={2025},
  publisher={INFORMS}
}

@article{gao2022joint,
  title={Joint learning and optimization for multi-product pricing (and ranking) under a general cascade click model},
  author={Gao, Xiangyu and Jasin, Stefanus and Najafi, Sajjad and Zhang, Huanan},
  journal={Management Science},
  volume={68},
  number={10},
  pages={7362--7382},
  year={2022},
  publisher={INFORMS}
}

@article{cheung2022inventory,
  title={Inventory balancing with online learning},
  author={Cheung, Wang Chi and Ma, Will and Simchi-Levi, David and Wang, Xinshang},
  journal={Management Science},
  volume={68},
  number={3},
  pages={1776--1807},
  year={2022},
  publisher={INFORMS}
}

@article{gupta2024language,
  title={Language model cascades: Token-level uncertainty and beyond},
  author={Gupta, Neha and Narasimhan, Harikrishna and Jitkrittum, Wittawat and Rawat, Ankit Singh and Menon, Aditya Krishna and Kumar, Sanjiv},
  journal={arXiv preprint arXiv:2404.10136},
  year={2024}
}

@article{jaillet2025online,
  title={Online scheduling for llm inference with kv cache constraints},
  author={Jaillet, Patrick and Jiang, Jiashuo and Mellou, Konstantina and Molinaro, Marco and Podimata, Chara and Zhou, Zijie},
  journal={arXiv preprint arXiv:2502.07115},
  year={2025}
}

@article{ao2025optimizing,
  title={Optimizing llm inference: Fluid-guided online scheduling with memory constraints},
  author={Ao, Ruicheng and Luo, Gan and Simchi-Levi, David and Wang, Xinshang},
  journal={arXiv preprint arXiv:2504.11320},
  year={2025}
}

@inproceedings{yu2022orca,
  title     = {Orca: A Distributed Serving System for Transformer-Based Generative Models},
  author    = {Yu, Gyeong-In and Jeong, Joo Seong and Kim, Geon-Woo and Kim, Soojeong and Chun, Byung-Gon},
  booktitle = {Proceedings of the 16th USENIX Symposium on Operating Systems Design and Implementation},
  pages     = {521--538},
  year      = {2022}
}

@inproceedings{kwon2023efficient,
  title     = {Efficient Memory Management for Large Language Model Serving with PagedAttention},
  author    = {Kwon, Woosuk and Li, Zhuohan and Zhuang, Siyuan and Sheng, Ying and Zheng, Lianmin and Yu, Cody Hao and Gonzalez, Joseph E. and Zhang, Hao and Stoica, Ion},
  booktitle = {Proceedings of the ACM SIGOPS 29th Symposium on Operating Systems Principles},
  year      = {2023}
}

@inproceedings{agrawal2024sarathi,
  title     = {Taming Throughput-Latency Tradeoff in LLM Inference with Sarathi-Serve},
  author    = {Agrawal, Amey and Kedia, Nitin and Panwar, Ashish and Mohan, Jayashree and Kwatra, Nipun and Gulavani, Bhargav S. and Tumanov, Alexey and Ramjee, Ramachandran},
  booktitle = {Proceedings of the 18th USENIX Symposium on Operating Systems Design and Implementation},
  year      = {2024}
}

@article{huang2026optimal,
  title   = {Optimal Bayesian Stopping for Efficient Inference of Consistent LLM Answers},
  author  = {Huang, Jingkai and Ma, Will and Zhou, Zhengyuan},
  journal = {arXiv preprint arXiv:2602.05395},
  year    = {2026}
}

@article{li2026asymptotically,
  title   = {Asymptotically Optimal Sequential Testing with Heterogeneous LLMs},
  author  = {Li, Guokai and Liang, Jiaxin and Liu, Mo and Lei, Yanzhe and Jasin, Stefanus and Yang, Fenghua and Baxi, Preet},
  journal = {arXiv preprint arXiv:2604.01086},
  year    = {2026}
}

@unpublished{guo2026congestion,
  title  = {Congestion-Aware Static LLM Cascades: Analysis of a Steady-State Framework},
  author = {Guo, Yuan and Jasin, Stefanus and Lin, Chen-An},
  note   = {Working paper},
  year   = {2026}
}

\clearpage

\begin{table}[htp!]
\centering
\scriptsize
\caption{Regime A: time-averaged regret scaled by $10^3$, against the oracle under the three environments with different \(P_\xi\).}
\label{tab:regimeA}
\begin{tabular}{llcccc}
\toprule
 & & $T_0{=}0$ & $T_0{=}50$ & $T_0{=}100$ & $T_0{=}1000$\\
\midrule
\multicolumn{6}{l}{(i) Uniform\quad $u=0.23$;\; oracle utility $\times10^3$: 522.53 $\pm$ 12.63;\; 1.69 boxes queried}\\
\cmidrule(lr){1-6}
\multirow{2}{*}{\emph{Reward known}} & Online Plug-in Weitzman & 9.31$\pm$0.35 & 9.03$\pm$0.37 & 8.61$\pm$0.41 & 6.00$\pm$0.41\\
 & DOW & \textbf{7.59$\pm$0.39} & \textbf{7.66$\pm$0.41} & \textbf{7.45$\pm$0.37} & 5.58$\pm$0.42\\
\arrayrulecolor{black!25}\cmidrule(lr){2-6}\arrayrulecolor{black}
\multirow{2}{*}{\emph{Reward unknown}} & Online Plug-in Weitzman & 10.17$\pm$0.35 & 9.43$\pm$0.33 & 9.17$\pm$0.41 & 6.13$\pm$0.40\\
 & DOW & \textbf{7.97$\pm$0.37} & \textbf{8.03$\pm$0.38} & \textbf{7.83$\pm$0.37} & \textbf{5.74$\pm$0.41}\\
\arrayrulecolor{black!25}\cmidrule(lr){2-6}\arrayrulecolor{black}
\multirow{4}{*}{\emph{Baselines}} & Static Plug-in Weitzman & --- & 105.92$\pm$6.39 & 88.14$\pm$5.04 & 37.86$\pm$1.69\\
 & LinUCB (online) & 109.03$\pm$5.77  & 110.02$\pm$5.58 & 109.96$\pm$5.57 & 110.96$\pm$5.50\\
 & LinUCB (static) & --- & 133.78$\pm$6.23 & 134.63$\pm$6.34 & 131.44$\pm$5.78\\
 & FrugalGPT (reward known) & --- & 24.32$\pm$5.53 & 16.88$\pm$4.25 & \textbf{5.55$\pm$1.08} \\
\midrule
\multicolumn{6}{l}{(ii) Truncated Gaussian\quad $u=0.2$;\; oracle utility $\times10^3$: 550.79 $\pm$ 13.08;\; 1.78 boxes queried}\\
\cmidrule(lr){1-6}
\multirow{2}{*}{\emph{Reward known}} & Online Plug-in Weitzman & 8.59$\pm$0.31 & 8.11$\pm$0.37 & 7.94$\pm$0.35 & 5.14$\pm$0.38\\
 & DOW & \textbf{6.65$\pm$0.39} & \textbf{6.55$\pm$0.35} & \textbf{6.46$\pm$0.37} & \textbf{4.71$\pm$0.38}\\
\arrayrulecolor{black!25}\cmidrule(lr){2-6}\arrayrulecolor{black}
\multirow{2}{*}{\emph{Reward unknown}} & Online Plug-in Weitzman & 9.02$\pm$0.34 & 8.46$\pm$0.35 & 8.15$\pm$0.33 & 5.40$\pm$0.38\\
 & DOW & \textbf{7.17$\pm$0.31} & \textbf{6.95$\pm$0.35} & \textbf{6.82$\pm$0.35} & \textbf{5.03$\pm$0.40}\\
\arrayrulecolor{black!25}\cmidrule(lr){2-6}\arrayrulecolor{black}
\multirow{4}{*}{\emph{Baselines}} & Static Plug-in Weitzman & --- & 104.54$\pm$5.56 & 86.90$\pm$4.73 & 35.26$\pm$1.65\\
 & LinUCB (online) & 124.28$\pm$6.20  & 124.95$\pm$5.99 & 125.18$\pm$6.04 & 125.88$\pm$5.96\\
 & LinUCB (static) & --- & 143.68$\pm$6.64 & 143.81$\pm$6.52 & 142.00$\pm$6.28\\
 & FrugalGPT (reward known) & --- & 25.79$\pm$5.83 & 19.86$\pm$3.75 & 6.20$\pm$1.30 \\
\midrule
\multicolumn{6}{l}{(iii) Truncated Student's $t_5$\quad $u=0.27$;\; oracle utility $\times10^3$: 489.72 $\pm$ 15.49;\; 1.61 boxes queried}\\
\cmidrule(lr){1-6}
\multirow{2}{*}{\emph{Reward known}} & Online Plug-in Weitzman & 9.75$\pm$0.34 & 9.37$\pm$0.38 & 9.27$\pm$0.41 & 6.57$\pm$0.44\\
 & DOW & \textbf{8.11$\pm$0.44} & \textbf{8.08$\pm$0.42} & \textbf{8.01$\pm$0.44} & 6.14$\pm$0.47\\
\arrayrulecolor{black!25}\cmidrule(lr){2-6}\arrayrulecolor{black}
\multirow{2}{*}{\emph{Reward unknown}} & Online Plug-in Weitzman & 10.02$\pm$0.38 & 9.67$\pm$0.39 & 9.40$\pm$0.34 & 6.80$\pm$0.45\\
 & DOW & \textbf{8.53$\pm$0.42} & \textbf{8.46$\pm$0.43} & \textbf{8.29$\pm$0.39} & \textbf{6.39$\pm$0.46}\\
\arrayrulecolor{black!25}\cmidrule(lr){2-6}\arrayrulecolor{black}
\multirow{4}{*}{\emph{Baselines}} & Static Plug-in Weitzman & --- & 92.76$\pm$6.94 & 78.47$\pm$5.99 & 37.43$\pm$2.26\\
 & LinUCB (online) & 95.44$\pm$5.37  & 95.42$\pm$5.49 & 95.53$\pm$5.18 & 96.45$\pm$5.35\\
 & LinUCB (static) & --- & 123.29$\pm$7.60 & 121.85$\pm$7.24 & 119.67$\pm$6.77\\
 & FrugalGPT (reward known) & --- & 30.20$\pm$5.57 & 20.54$\pm$4.54 & \textbf{5.13$\pm$1.62} \\
\bottomrule
\end{tabular}
\end{table}

\begin{table}[htp!]\centering
\scriptsize
\caption{Regime B: time-averaged regret scaled by $10^3$, against the oracle under the three environments with different \(P_\xi\).}
\label{tab:regimeB}
\begin{tabular}{llcccc}
\toprule
 & & $T_0{=}0$ & $T_0{=}50$ & $T_0{=}100$ & $T_0{=}1000$\\
\midrule
\multicolumn{6}{l}{(i) Uniform\quad $u=0.23$;\; oracle utility $\times10^3$: 430.78 $\pm$ 18.25;\; 1.30 boxes queried}\\
\cmidrule(lr){1-6}
\multirow{2}{*}{\emph{Reward known}} & Online Plug-in Weitzman & 6.97$\pm$0.30 & \textbf{6.77$\pm$0.32} & \textbf{6.71$\pm$0.34} & \textbf{4.96$\pm$0.44}\\
 & DOW & \textbf{6.91$\pm$0.35} & 6.85$\pm$0.30 & 6.75$\pm$0.33 & 4.99$\pm$0.43\\
\arrayrulecolor{black!25}\cmidrule(lr){2-6}\arrayrulecolor{black}
\multirow{2}{*}{\emph{Reward unknown}} & Online Plug-in Weitzman & \textbf{7.04$\pm$0.33} & 6.96$\pm$0.27 & 6.84$\pm$0.36 & \textbf{5.06$\pm$0.43}\\
 & DOW & 7.06$\pm$0.30 & \textbf{6.85$\pm$0.27} & \textbf{6.83$\pm$0.33} & 5.11$\pm$0.43\\
\arrayrulecolor{black!25}\cmidrule(lr){2-6}\arrayrulecolor{black}
\multirow{4}{*}{\emph{Baselines}} & Static Plug-in Weitzman & --- & 77.49$\pm$5.86 & 71.08$\pm$4.89 & 31.40$\pm$1.83\\
 & LinUCB (online) & 42.65$\pm$4.64  & 43.25$\pm$4.65 & 44.12$\pm$4.85 & 43.38$\pm$4.71\\
 & LinUCB (static) & --- & 89.78$\pm$7.92 & 87.79$\pm$7.59 & 81.20$\pm$6.15\\
 & FrugalGPT (reward known) & --- & 43.34$\pm$8.31 & 32.97$\pm$5.21 & 17.38$\pm$2.28 \\
\midrule
\multicolumn{6}{l}{(ii) Truncated Gaussian\quad $u=0.2$;\; oracle utility $\times10^3$: 455.57 $\pm$ 18.08;\; 1.48 boxes queried}\\
\cmidrule(lr){1-6}
\multirow{2}{*}{\emph{Reward known}} & Online Plug-in Weitzman & 8.52$\pm$0.34 & 8.50$\pm$0.35 & 8.23$\pm$0.39 & 6.16$\pm$0.49\\
 & DOW & \textbf{7.80$\pm$0.49} & \textbf{7.70$\pm$0.42} & \textbf{7.69$\pm$0.47} & \textbf{5.94$\pm$0.49}\\
\arrayrulecolor{black!25}\cmidrule(lr){2-6}\arrayrulecolor{black}
\multirow{2}{*}{\emph{Reward unknown}} & Online Plug-in Weitzman & 8.87$\pm$0.36 & 8.50$\pm$0.37 & 8.56$\pm$0.40 & 6.28$\pm$0.49\\
 & DOW & \textbf{8.13$\pm$0.48} & \textbf{8.17$\pm$0.51} & \textbf{7.83$\pm$0.44} & \textbf{6.20$\pm$0.51}\\
\arrayrulecolor{black!25}\cmidrule(lr){2-6}\arrayrulecolor{black}
\multirow{4}{*}{\emph{Baselines}} & Static Plug-in Weitzman & --- & 78.70$\pm$5.72 & 70.20$\pm$4.97 & 34.35$\pm$1.82\\
 & LinUCB (online) & 68.29$\pm$6.32  & 68.31$\pm$6.15 & 67.76$\pm$6.15 & 68.99$\pm$6.18\\
 & LinUCB (static) & --- & 99.51$\pm$8.33 & 99.49$\pm$7.98 & 95.20$\pm$7.14\\
 & FrugalGPT (reward known) & --- & 31.87$\pm$5.40 & 26.99$\pm$3.84 & 14.39$\pm$1.42 \\
\midrule
\multicolumn{6}{l}{(iii) Truncated Student's $t_5$\quad $u=0.27$;\; oracle utility $\times10^3$: 413.64 $\pm$ 21.50;\; 1.32 boxes queried}\\
\cmidrule(lr){1-6}
\multirow{2}{*}{\emph{Reward known}} & Online Plug-in Weitzman & 8.77$\pm$0.37 & 8.63$\pm$0.42 & 8.24$\pm$0.38 & \textbf{6.02$\pm$0.47}\\
 & DOW & \textbf{8.35$\pm$0.38} & \textbf{8.37$\pm$0.48} & \textbf{8.02$\pm$0.44} & 6.04$\pm$0.46\\
\arrayrulecolor{black!25}\cmidrule(lr){2-6}\arrayrulecolor{black}
\multirow{2}{*}{\emph{Reward unknown}} & Online Plug-in Weitzman & 8.83$\pm$0.37 & \textbf{8.62$\pm$0.38} & 8.64$\pm$0.39 & \textbf{6.21$\pm$0.45}\\
 & DOW & \textbf{8.78$\pm$0.36} & 8.72$\pm$0.41 & \textbf{8.48$\pm$0.48} & 6.24$\pm$0.45\\
\arrayrulecolor{black!25}\cmidrule(lr){2-6}\arrayrulecolor{black}
\multirow{4}{*}{\emph{Baselines}} & Static Plug-in Weitzman & --- & 79.79$\pm$6.34 & 74.62$\pm$5.61 & 38.28$\pm$2.44\\
 & LinUCB (online) & 48.33$\pm$6.78  & 47.71$\pm$6.76 & 47.92$\pm$6.70 & 47.97$\pm$6.90\\
 & LinUCB (static) & --- & 98.30$\pm$9.02 & 98.96$\pm$9.39 & 90.22$\pm$7.32\\
 & FrugalGPT (reward known) & --- & 38.00$\pm$8.35 & 21.94$\pm$4.34 & 13.59$\pm$1.38 \\
\bottomrule
\end{tabular}
\end{table}

\begin{table}[htp!]
\centering
\scriptsize
\caption{
This table compares the oracle-order variation across 50 replications across Regimes~A and B. 
In each replication, we draw a full set of parameters and sample $5{,}000$ i.i.d.\ contexts. We then record the number of distinct permutations observed and the largest and smallest empirical probability masses. The table averages these quantities across replications and reports $\pm1.96$ standard errors. Note that because the reservation indices in Regime~A are context-free by construction, exactly one permutation occurs in every replication. }
\label{tab:oracle-order-stats}
\begin{tabular}{llccc}
\toprule
Environment & Regime & \# permutations & max.\ prob.\ mass
& min.\ prob.\ mass\\
\midrule
\multirow{2}{*}{(i) Uniform}
& A & $1.00\pm0.00$ & $1.00\pm0.00$ & $1.00\pm0.00$\\
& B & $2.58\pm0.45$ & $0.87\pm0.05$ & $0.32\pm0.12$\\
\arrayrulecolor{black!25}\cmidrule(lr){1-5}\arrayrulecolor{black}

\multirow{2}{*}{(ii) Truncated Gaussian}
& A & $1.00\pm0.00$ & $1.00\pm0.00$ & $1.00\pm0.00$\\
& B & $2.36\pm0.39$ & $0.86\pm0.05$ & $0.28\pm0.11$\\
\arrayrulecolor{black!25}\cmidrule(lr){1-5}\arrayrulecolor{black}

\multirow{2}{*}{(iii) Truncated Student's $t_5$}
& A & $1.00\pm0.00$ & $1.00\pm0.00$ & $1.00\pm0.00$\\
& B & $2.20\pm0.38$ & $0.90\pm0.04$ & $0.38\pm0.12$\\
\bottomrule
\end{tabular}
\end{table}

\clearpage
\begin{APPENDICES}
\renewcommand{\thesection}{EC.\arabic{section}} 
\renewcommand{\thesubsection}{EC.\arabic{section}.\arabic{subsection}} 

\renewcommand*{\theHsection}{ec.\arabic{section}}
\renewcommand*{\theHsubsection}
  {ec.\arabic{section}.\arabic{subsection}}
\renewcommand*{\theHsubsubsection}
  {ec.\arabic{section}.\arabic{subsection}.\arabic{subsubsection}}

\pagenumbering{arabic} 
\renewcommand{\thepage}{ec\arabic{page}}

\begin{center}
\Large\bfseries Electronic Companions to ``Online Pandora's Box for \\
Contextual LLM Cascading''
\end{center}

\section{Proofs for Section~\ref{sec:policy} and Proposition~\ref{prop:uniqueness:rho:a:star}}\label{appendix:proof:prop:1:2}
In this Appendix we prove two results in Section~\ref{sec:policy}: the optimality of the oracle reservation index policy under full information (Proposition~\ref{prop:oracle-weitzman}), a regret decomposition under optimistic reward and index estimators (Theorem~\ref{thm:optimistic-regret-decomposition}), and one result in Section~\ref{sec:known:theta-*}: point identification of $\rho_a^\top\psi(x_t)$ from the queried samples for box $a$ before period $t$ (Proposition~\ref{prop:uniqueness:rho:a:star}).
The following lemma establishes that the query decision is independent of the stochastic output conditioning on the historical data and the current context at each period $t$. Although the proof is immediate, the result is used repeatedly in subsequent arguments, so we state it explicitly here.
\begin{lemma}\label{lemma:indep:output:query}
Given any $t\in[T]$, $x_t\in\mathcal{X}$ and $a\in[A]$, we have $\mathbb{I}\{a\in\mathcal{A}_t\}\indep(\omega_{at},c_{at})|x_t,\mathcal{F}_{t-1}$.
\end{lemma}

\medskip
\proof{Proof of Lemma~\ref{lemma:indep:output:query}}
For any $x_t\in\mathcal X$ and $a\in[A]$, under both the oracle reservation index policy and \DOW algorithm, the event $\{a\in\mathcal A_t\}$ is determined by the sequential construction, which depends on the history $\mathcal F_{t-1}$, the current context $x_t$, and the outputs of boxes queried before $a$. Under the conditional independence of $\{\omega_{a't}:a'\in[A]\}$ given $x_t$ according to \eqref{eq:law:outputs}, these preceding outputs are independent of $\omega_{at}$. Further, $c_{at}=c_a(x_t,\omega_{at})$. Hence, $\mathbb{I}\{a\in\mathcal{A}_t\}\indep(\omega_{at},c_{at})|x_t,\mathcal{F}_{t-1}$ as claimed. \proofend

\medskip
\proof{Proof of Proposition~\ref{prop:oracle-weitzman}} For $t \in [T]$, let $Q_a^t = \mathbb{I}\{a \in \mathcal{A}_t\}$ and $S_a^t = \mathbb{I}\{ a_t = a\}$. Since $\omega_{at}\sim p_a(\cdot|x_t)$, \eqref{eq:oracle-weitzman-index} implies  
$\mathbb E\left[
        \left\{\mu^*(x_t,\omega_{at})-\sigma^*_a(x_t)\right\}^+
        \,\middle|\, x_t,\mathcal{F}_{t-1}
    \right]
    =
    \mathbb E[c_{at}\mid x_t,\mathcal{F}_{t-1}], \forall x_t.$
Thus 
$$\begin{array}{rl}
&\quad\displaystyle\mathbb{E}\left[\sum_{a\in[A]}S_a^t\mu^*(x_t,\omega_{at})-\sum_{a\in[A]}Q_a^tc_{at}\bigg|x_t,\mathcal{F}_{t-1}\right]\\
&\displaystyle=_{(i)}\mathbb{E}\left[\sum_{a\in[A]}S_a^t\mu^*(x_t,\omega_{at})-\sum_{a\in[A]}Q_a^t\{\mu^*(x_t,\omega_{at})- \sigma_{at}^*\}^+\bigg|x_t,\mathcal{F}_{t-1}\right]\\
& 
\displaystyle=_{(ii)}\mathbb{E}\left[\sum_{a\in[A]}S_a^t\min\{ \sigma_{at}^*,\mu^*(x_t,\omega_{at})\}+\sum_{a\in[A]}(S_a^t-Q_a^t)\{ \mu^*(x_t,\omega_{at}) - \sigma_{at}^*\}^+\bigg|x_t,\mathcal{F}_{t-1}\right]\\
&\displaystyle\leq_{(iii)} \mathbb{E}\left[\sum_{a\in[A]}S_a^t\min\{ \sigma_{at}^*,\mu^*(x_t,\omega_{at})\}\bigg|x_t,\mathcal{F}_{t-1}\right] \displaystyle \leq_{(iv)} \mathbb{E}\left[ \max_{a\in [A]} \min\{ \sigma_{at}^*,\mu^*(x_t,\omega_{at})\}\bigg|x_t,\mathcal{F}_{t-1}\right],
\end{array}$$
where (i) holds because $Q_a^t\indep(\omega_{at},c_{at})\mid x_t,\mathcal{F}_{t-1}$ by Lemma~\ref{lemma:indep:output:query} and the definition of $\sigma_{at}^*$ in \eqref{eq:oracle-weitzman-index} imply $$\begin{array}{rl}
\E[Q_a^tc_{at}|x_t,\mathcal{F}_{t-1}]\!\!\!&\displaystyle=\E[Q_a^t|x_t,\mathcal{F}_{t-1}]\E[c_{at}|x_t,\mathcal{F}_{t-1}]=\E[Q_a^t|x_t,\mathcal{F}_{t-1}]\E[\{\mu^*(x_t,\omega_{at})-\sigma_{at}^*\}^{+}|x_t,\mathcal{F}_{t-1}]\\
&\displaystyle=\E[Q_a^t\{\mu^*(x_t,\omega_{at})-\sigma_{at}^*\}^{+}|x_t,\mathcal{F}_{t-1}],
\end{array}$$
(ii) holds by the identity $b=\min\{a,b\}+\{b-a\}^+$, (iii) since $S_a^t \leq Q_a^t$, and (iv) since $\sum_{a\in[A]}S_a^t = 1$ and $S_a^t \geq 0$. Inequalities (iii) and (iv) hold as equalities when using Weitzman's principle through indices defined by \eqref{eq:oracle-weitzman-index}. 

It remains to show that the policy stated in the proposition attains
the upper bound for inequality (iv). Let $K:=\min\{k\in[A]:{\displaystyle\max_{1\leq i\leq k}}\mu^*(x_t,\omega_{(i)t})\geq\sigma_{(k+1)t}^*\}$ be the stopping position of period $t$. 

First, every queried but unselected box $(a)$ such that $a<K$ satisfies $\displaystyle\mu^*(x_t,\omega_{(a)t})\leq\max_{1\leq i\leq a}\mu^*(x_t,\omega_{(i)t})<\sigma_{(a+1)t}^*\leq\sigma_{(a)t}^*$. If box $(K)$ is not selected, then $\mu^*(x_t,\omega_{(K)t})\leq\max_{1\leq i\leq K-1}\mu^*(x_t,\omega_{(i)t})<\sigma_{(K)t}^*$. Hence $(S_a^t-Q_a^t)
\{\mu^*(x_t,\omega_{at})-\sigma_{at}^*\}^+=0$ for every $a\in[A]$ and inequality (iii) is an equality. 

Second, let $\displaystyle J\in\argmax_{1\leq i\leq K}\mu^*(x_t,\omega_{(i)t})$ be the rank of the selected box. If $J<K$, then $\displaystyle\mu^*(x_t,\omega_{(J)t})=\max_{1\leq i\leq K}\mu^*(x_t,\omega_{(i)t})$ and $\mu^*(x_t,\omega_{(J)t})<\sigma_{(J)t}^*$, so $\displaystyle\min\{\mu^*(x_t,\omega_{(J)t}),\sigma_{(J)t}^*\}=\max_{1\leq i\leq K}\mu^*(x_t,\omega_{(i)t})$, while for unqueried box $i>K$, $\displaystyle\min\{\mu^*(x_t,\omega_{(i)t}),\sigma_{(i)t}^*\}\leq\sigma_{(i)t}^*\leq\sigma_{(K+1)t}^*\leq\max_{1\leq i\leq K}\mu^*(x_t,\omega_{(i)t})$. If $J=K$, then $\displaystyle\min\{\mu^*(x_t,\omega_{(K)t}),\sigma_{(K)t}^*\}\geq\max_{1\leq i\leq K-1}\mu^*(x_t,\omega_{(i)t})$ and $\displaystyle\min\{\mu^*(x_t,\omega_{(K)t}),\sigma_{(K)t}^*\}\geq\sigma_{(K+1)}^*(x_t)$. Hence $\min\{\mu^*(x_t,\omega_{(J)t}),\sigma_{(J)}^*(x_t)\}$ dominates the value of $\min\{\sigma_{at}^*,\mu^*(x_t,\omega_{at})\}$ for both the queried and every unqueried boxes. Thus pathwise, $\sum_{a\in[A]}S_a^t
\min\{\sigma_{at}^*,\mu^*(x_t,\omega_{at})\}
=
\max_{a\in[A]}
\min\{\sigma_{at}^*,\mu^*(x_t,\omega_{at})\}$. So inequality (iv) is also an equality for the stated policy.

So the utility of the DM $U_t$ satisfies
\begin{equation}
U_t=\mathbb{E}\left[\sum_{a\in[A]}S_a^t\mu^*(x_t,\omega_{at})-\sum_{a\in[A]}Q_a^tc_{at}\bigg|x_t,\mathcal{F}_{t-1}\right]=\mathbb{E}\left[\max_{a\in [A]} \min\{ \sigma_{at}^*,\mu^*(x_t,\omega_{at})\}\bigg|x_t,\mathcal{F}_{t-1}\right].
\end{equation}
This concludes the proof. \proofend

\medskip
\proof{Proof of Theorem~\ref{thm:optimistic-regret-decomposition}}
Define $U_t(\tilde{\pi}):=\mu^*(x_t,\omega_{a_tt})-\sum_{a\in\mathcal{A}_t}c_{at}$, where $\tilde{\pi}$ is the \DOW policy.
Let $\tilde{c}_{at}:=(\widetilde{\mu}_t(x_t,\omega_{at})-\widetilde{\sigma}_{at})^{+}$ be defined as an auxiliary pseudo-cost for any $a\in[A], t\in[T]$. Let $\tilde{U}_t(\tilde{\pi})$ denote the realized utility of $\tilde{\pi}$ under cost $\tilde{c}_{at}$, i.e. $\tilde{U}_t(\tilde{\pi}):=\max_{a\in\mathcal{A}_t}\widetilde{\mu}_t(x_t,\omega_{at})-\sum_{a\in\mathcal{A}_t}\tilde{c}_{at}.$ 
By \eqref{eq:max:utility} in Proposition \ref{prop:oracle-weitzman}, we have
$$\begin{array}{rcl}
\displaystyle\mathbb{E}[U_t(\pi^*)\mid x_t,\mathcal{F}_{t-1}]&=&\displaystyle\mathbb{E}\big[\max_{a\in[A]}\min\{\mu^*(x_t,\omega_{at}),\sigma_{at}^*\}\mid x_t,\mathcal{F}_{t-1}\big],\\
\displaystyle\mathbb{E}\left[\tilde{U}_t(\tilde{\pi})\mid x_t,\mathcal{F}_{t-1}\right]&=&\displaystyle\mathbb{E}\big[\max_{a\in[A]}\min\{\widetilde{\mu}_t(x_t,\omega_{at}),\widetilde{\sigma}_{at}\}\mid x_t,\mathcal{F}_{t-1}\big]. 
\end{array}$$ 
Note that for the same execution of $\tilde{\pi}$ pathwise,
$U_t(\tilde{\pi})=\tilde{U}_t(\tilde{\pi})-(\widetilde{\mu}_t(x_t,\omega_{a_tt})-\mu^*(x_t,\omega_{a_tt}))-\sum_{a\in\mathcal{A}_t}(c_{at}-\tilde{c}_{at}).$
Therefore, 
\begin{equation}\label{eq:Delta:t}
\begin{array}{rl}
\Delta_t(\tilde{\pi})&=\mathbb{E}[U_t(\pi^*)\mid x_t,\mathcal{F}_{t-1}]-\mathbb{E}[U_t(\tilde{\pi})\mid x_t,\mathcal{F}_{t-1}]\\
&\displaystyle=\mathbb{E}\big[\max_{a\in[A]}\min\{\mu^*(x_t,\omega_{at}),\sigma_{at}^*\}\mid x_t,\mathcal{F}_{t-1}\big]-\mathbb{E}\big[\max_{a\in[A]}\min\{\widetilde{\mu}_t(x_t,\omega_{at}),\widetilde{\sigma}_{at}\}\mid x_t,\mathcal{F}_{t-1}\big]\\
&\displaystyle\quad+\mathbb{E}[\widetilde{\mu}_t(x_t,\omega_{a_tt})-\mu^*(x_t,\omega_{a_tt})\mid x_t,\mathcal{F}_{t-1}]+\mathbb{E}\bigg[\sum_{a\in\mathcal{A}_t}(c_{at}-\tilde{c}_{at})\mid x_t,\mathcal{F}_{t-1}\bigg].
\end{array}
\end{equation}
Since $\widetilde{\mu}_t(x_t,\omega_{at})\geq\mu^*(x_t,\omega_{at})$ and $\widetilde{\sigma}_{at}\geq\sigma_{at}^*$ a.s., 
$\min\{\widetilde{\mu}_t(x_t,\omega_{at}),\widetilde{\sigma}_{at}\}\geq\min\{\mu^*(x_t,\omega_{at}),\sigma_{at}^*\}$ a.s. for every $a\in[A]$. So 
\begin{equation}\label{eq:term:1:Delta:t}
\mathbb{E}\big[\max_{a\in[A]}\min\{\mu^*(x_t,\omega_{at}),\sigma_{at}^*\}\mid x_t,\mathcal{F}_{t-1}\big]-\mathbb{E}\big[\max_{a\in[A]}\min\{\widetilde{\mu}_t(x_t,\omega_{at}),\widetilde{\sigma}_{at}\}\mid x_t,\mathcal{F}_{t-1}\big]\leq0.
\end{equation}
Further, for each $a\in[A]$, 
$$\begin{array}{rl}
\mu^*(x_t,\omega_{at})-\sigma_{at}^*&\displaystyle=(\widetilde{\mu}_t(x_t,\omega_{at})-\widetilde{\sigma}_{at})+(\mu^*(x_t,\omega_{at})-\widetilde{\mu}_t(x_t,\omega_{at}))+(\widetilde{\sigma}_{at}-\sigma_{at}^*)\\
&\displaystyle\leq(\widetilde{\mu}_t(x_t,\omega_{at})-\widetilde{\sigma}_{at})+(\widetilde{\sigma}_{at}-\sigma_{at}^*).
\end{array}$$
Taking positive parts and using subadditivity of the function $t\mapsto (t)^+$ gives 
$$(\mu^*(x_t,\omega_{at})-\sigma_{at}^*)^{+}\leq(\widetilde{\mu}_t(x_t,\omega_{at})-\widetilde{\sigma}_{at})^{+}+(\widetilde{\sigma}_{at}-\sigma_{at}^*)^{+}=(\widetilde{\mu}_t(x_t,\omega_{at})-\widetilde{\sigma}_{at})^{+}+(\widetilde{\sigma}_{at}-\sigma_{at}^*).$$
Hence 
\begin{equation}\label{eq:diff:c:x}
\begin{array}{rl}
\mathbb{E}[c_{at}-\tilde{c}_{at}\mid x_t,\mathcal{F}_{t-1}]&\displaystyle=\mathbb{E}[(\mu^*(x_t,\omega_{at})-\sigma_{at}^*)^{+}\mid x_t,\mathcal{F}_{t-1}]-\mathbb{E}[(\widetilde{\mu}_t(x_t,\omega_{at})-\widetilde{\sigma}_{at})^{+}\mid x_t,\mathcal{F}_{t-1}]\\
&\displaystyle\leq\widetilde{\sigma}_{at}-\sigma_{at}^*.
\end{array}
\end{equation}
Thus 
\begin{equation}\label{eq:c:diff:bound}
\begin{array}{rl}
&\displaystyle\quad\mathbb{E}\left[\sum_{a\in\mathcal{A}_t}(c_{at}-\tilde{c}_{at})\ \bigg|\ x_t,\mathcal{F}_{t-1}\right]=\sum_{a\in[A]}\mathbb{E}[(c_{at}-\tilde{c}_{at})\mathbb{I}\{a\in\mathcal{A}_t\} \mid x_t,\mathcal{F}_{t-1}]\\
&\displaystyle=_{(i)}\sum_{a\in[A]}\mathbb{E}[c_{at}-\tilde{c}_{at}\mid x_t,\mathcal{F}_{t-1}]\mathbb{E}[\mathbb{I}\{a\in\mathcal{A}_t\}\mid x_t,\mathcal{F}_{t-1}]\\
    &\displaystyle\leq_{(ii)}\sum_{a\in[A]}(\widetilde{\sigma}_{at}-\sigma_{at}^*)\mathbb{E}[\mathbb{I}\{a\in\mathcal{A}_t\}\mid x_t,\mathcal{F}_{t-1}]\leq\mathbb{E}\left[\sum_{a\in\mathcal{A}_t}(\widetilde{\sigma}_{at}-\sigma_{at}^*)\ \bigg|\ x_t,\mathcal{F}_{t-1}\right].
\end{array}
\end{equation}
where (i) of \eqref{eq:c:diff:bound} follows since $c_{at}=c_a(x_t,\omega_{at}), \tilde{c}_{at}=(\widetilde{\mu}_t(x_t,\omega_{at})-\widetilde{\sigma}_{at})^{+}$ and $\mathbb{I}\{a\in\mathcal{A}_t\}\indep \omega_{at}|x_t,\mathcal{F}_{t-1}$ by Lemma~\ref{lemma:indep:output:query}, (ii) follows from \eqref{eq:diff:c:x}.
The result then follows from \eqref{eq:Delta:t}, \eqref{eq:term:1:Delta:t}, \eqref{eq:c:diff:bound}. 
\proofend

\medskip
\proof{Proof of Proposition~\ref{prop:uniqueness:rho:a:star}}
For any $x_t\in\mathcal{X}$ and $a\in[A]$, we have  $\mathbb{I}\{a\in\mathcal{A}_t\}\indep\omega_{at}|x_t,\mathcal{F}_{t-1}$ by Lemma~\ref{lemma:indep:output:query}. Since $c_{at}=c_a(x_t,\omega_{at})$ we have
$$\begin{array}{rl}
&\quad\displaystyle\mathbb{E}\left[\{ \mu^*(x_t,\omega_{at})-\Lambda(\rho_a^\top\psi(x_t))\}^{+}-c_{at}|x_t,\mathcal{F}_{t-1},a\in\mathcal{A}_t\right]\\ 
&\displaystyle = \mathbb{E}\left[\{ \mu^*(x_t,\omega_{at})-\Lambda(\rho_a^\top\psi(x_t))\}^{+}-c_{at}|x_t,\mathcal{F}_{t-1}\right]. 
\end{array}$$
For any given $a\in[A]$ and $x$, define 
$F_{a}(x,s):=\mathbb{E}\left[\{\mu^*(x,\omega_{at})-\Lambda(s)\}^{+}-c_{at}|x_t=x,\mathcal{F}_{t-1},a\in\mathcal{A}_t\right]$.
So we have $F_a\left(x,\rho_a^\top \psi(x)\right)=0, \forall x\in\mathcal{X}$. Note that $\Lambda(\cdot)$ is strictly increasing and takes values in $[-1,1]$, and $c_{at}\in(0,1)$, so for every $x\in\mathcal{X}$, $F_a(x,s)$ is decreasing and continuous in $s$. Note that 
$$\begin{array}{rl}
\lim_{s\rightarrow-\infty}F_a(x,s)&=\lim_{s\rightarrow-\infty}\mathbb{E}\left[\{\mu^*(x,\omega_{at})-\Lambda(s)\}^{+}-c_{at}|x_t=x,\mathcal{F}_{t-1}, a\in\mathcal{A}_t\right]\\
&\displaystyle=_{(i)}\mathbb{E}\left[\{\mu^*(x,\omega_{at})+1\}^{+}-c_{at}|x_t=x,\mathcal{F}_{t-1}, a\in\mathcal{A}_t\right]\\
&\displaystyle\geq\mathbb{E}\left[1-c_{at}|x_t=x,\mathcal{F}_{t-1}, a\in\mathcal{A}_t\right]=_{(ii)}\mathbb{E}\left[1-c_{at}|x_t=x,\mathcal{F}_{t-1}\right]>_{(iii)}0
\end{array}$$
where (i) follows from Assumption~\ref{ass:reg_res_index}, (ii) follows because $c_{at}=c_a(x_t,\omega_{at})$ and $\omega_{at}\indep\mathbb{I}\{a\in\mathcal{A}_t\}|x_t,\mathcal{F}_{t-1}$ by Lemma~\ref{lemma:indep:output:query}, (iii) follows because $c_{at}\in(0,1)$. Thus $\lim_{s\rightarrow-\infty}F_a(x,s)>0$. Similarly, Assumption~\ref{ass:reg_res_index} implies $\lim_{s\rightarrow+\infty}F_a(x,s)<0$. Thus by intermediate value theorem, there exists a $\xi_a(x)\in\mathbb{R}$ such that $F_a(x,\xi_a(x))=0$. Now suppose that there exists $s_1<s_2$ such that $F_a(x,s_1)=F_a(x,s_2)=0$. Then we have 
\begin{equation}\label{eq:F:a:0}
0=F_a(x,s_1)-F_a(x,s_2)=\mathbb{E}[(\mu^*(x,\omega_{at})-\Lambda(s_1))^{+}-(\mu^*(x,\omega_{at})-\Lambda(s_2))^{+}\mid x_t=x,\mathcal{F}_{t-1},a\in\mathcal{A}_t].
\end{equation}
This implies that 
$$\begin{array}{rl}
\displaystyle\mathbb{E}[(\mu^*(x,\omega_{at})-\Lambda(s_1))^{+}|x_t=x,\mathcal{F}_{t-1},a\in\mathcal{A}_t]&\displaystyle=\mathbb{E}[(\mu^*(x,\omega_{at})-\Lambda(s_2))^{+}|x_t=x,\mathcal{F}_{t-1},a\in\mathcal{A}_t]\\
&\displaystyle=\mathbb{E}[c_{at}|x_t=x,\mathcal{F}_{t-1},a\in\mathcal{A}_t]>0.
\end{array}$$ 
Since $s_1<s_2$ and $\Lambda$ is strictly increasing, there must exist an event $\mathcal{E}$ where $\mu^*(x,\omega_{at})>\Lambda(s_2)>\Lambda(s_1)$ and $\mathbb{P}(\mathcal{E}|x_t=x,\mathcal{F}_{t-1},a\in\mathcal{A}_t)>0$, and 
$$\begin{array}{rl}
F_a(x,s_1)-F_a(x,s_2)&=\mathbb{E}[(\mu^*(x,\omega_{at})-\Lambda(s_1))^{+}-(\mu^*(x,\omega_{at})-\Lambda(s_2))^{+}\mid x_t=x,\mathcal{F}_{t-1},a\in\mathcal{A}_t]\\
&\geq\mathbb{E}[(\Lambda(s_2)-\Lambda(s_1))\mathbb{I}\{\mathcal{E}\}|x_t=x,\mathcal{F}_{t-1},a\in\mathcal{A}_t]>0,
\end{array}$$
which contradicts \eqref{eq:F:a:0}. So the conditional moment restriction pins down the optimal index $\sigma_a^*(x)=\Lambda(\xi_a(x))$ uniquely, and $\xi_a(x)=\rho_a^\top\psi(x)$ for any $x\in\mathcal{X}$. \proofend

\section{Regret under Known Reward Function}
In this section, we provide regret analysis when $\mu^*$ is known. We begin with the technical lemmas used in the analysis.

\subsection{Technical Lemmas}
For any $a\in[A], t\in[T]$, let
\begin{equation}\label{eq:moment:t}
m_{at}(u):=\E\bigl[(\mu^*(x_t,\omega_{at})-\Lambda(u))^+ - c_{at}\mid \Ft_{t-1},x_t\bigr],
\end{equation}
so that $m_{at}(\rho_a^\top\psi(x_t))=0$ by (\ref{eq:known-theta-queried-moment}) and (\ref{eq:index:a}). 

\begin{lemma}\label{lem:glm-monotonicity}
Assume that, for every $t$, for some constant $\eta>0$,
\begin{equation}\label{eq:glm-product-lower-bound}
\Prob\bigl(\mu^*(x_t,\omega_{at})>\Lambda(u)\mid \Ft_{t-1},x_t\bigr)\,\Lambda'(u)\ge \eta, \ \forall |u|\leq\overline{\iota},
\end{equation}
Then, for every $u\in[-\overline{\iota},\overline{\iota}]$,
\begin{equation}\label{eq:glm-strong-monotonicity}
-(u-\rho_a^\top\psi(x_t))m_{at}(u)\ge \eta (u-\rho_a^\top\psi(x_t))^2
\qquad \text{a.s.}
\end{equation}
\end{lemma}
Lemma~\ref{lem:glm-monotonicity} says that the reservation-value moment crosses zero with a
slope bounded away from zero. If the candidate threshold $u$ is above the true threshold, then the expected excess value is too small relative to cost, so the moment
$c-(\mu^*-\Lambda(u))^+$ is positive. If $u$ is below the true threshold, then the expected
excess value is too large relative to cost, so the moment is negative. This monotonicity is what
ultimately makes the loss locally curved around $\rho_a$. 
\proof{Proof of Lemma~\ref{lem:glm-monotonicity}}
Fix $t$ and $u\in I=[-\bar{\iota},\bar{\iota}]$. For any real number $a$ and any continuously differentiable function $g$, the map $v\mapsto (a-g(v))^+$ is absolutely continuous and has a.e. derivative $\frac{d}{dv}(a-g(v))^+=-\one\{a>g(v)\}g'(v)$. Applying fundamental theorem of calculus to  $g=\Lambda$, for any $u,v\in\R$,
$(\mu^*(x_t,\omega_{at})-\Lambda(u))^+-(\mu^*(x_t,\omega_{at})-\Lambda(v))^+
=-\int_v^u \one\{\mu^*(x_t,\omega_{at})>\Lambda(r)\}\Lambda'(r)\,dr$.
Taking expectations conditional on $x_t, \Ft_{t-1}$, $m_{at}(u)-m_{at}(v)
=-\int_v^u \Prob\bigl(\mu^*(x_t,\omega_{at})>\Lambda(r)\mid \Ft_{t-1},x_t\bigr)\Lambda'(r)\,dr.$ Since $m_{at}(\rho_a^\top\psi(x_t))=0$, setting $v=\rho_a^\top\psi(x_t)$ gives
$m_{at}(u)
=-\int_{\rho_a^\top\psi(x_t)}^{u} \Prob\bigl(\mu^*(x_t,\omega_{at})>\Lambda(r)\mid \Ft_{t-1},x_t\bigr)\Lambda'(r)\,dr.$
If $u\ge\rho_a^\top\psi(x_t)$, then \eqref{eq:glm-product-lower-bound} implies
$m_{at}(u)
\le -\int_{\rho_a^\top\psi(x_t)}^{u} \eta\,dr
=-\eta\left(u-\rho_a^\top\psi(x_t)\right)$,
which proves \eqref{eq:glm-strong-monotonicity}. If $u\le\rho_a^\top\psi(x_t)$, then
$$m_{at}(u)
=\int_{u}^{\rho_a^\top\psi(x_t)} \Prob\bigl(\mu^*(x_t,\omega_{at})>\Lambda(r)\mid \Ft_{t-1},x_t\bigr)\Lambda'(r)\,dr
\ge_{(i)} \int_{u}^{\rho_a^\top\psi(x_t)} \eta\,dr
=\eta\left(\rho_a^\top\psi(x_t)-u\right),$$
where (i) follows from \eqref{eq:glm-product-lower-bound} and \eqref{eq:glm-strong-monotonicity} follows again.
\proofend

We next prove Lemma~\ref{lem:known-theta-loss-curvature} using Lemma~\ref{lem:glm-monotonicity}:
\begin{lemma*}[Restatement of Lemma~\ref{lem:known-theta-loss-curvature}]
Under Assumptions~\ref{ass:reg_res_index}, \ref{ass:losscurvature}, for any $t\in[T]$, $a\in[A]$ and $\rho\in\cB$, $\mathbb{E}[\ell_{at}(\rho)-\ell_{at}(\rho_a)\mid x_t,\mathcal{F}_{t-1}]
\ge \frac{1}{2}\kappa\mu_1\bigl(\psi(x_t)^\top(\rho-\rho_a)\bigr)^2$.
\end{lemma*}
\proof{Proof of Lemma~\ref{lem:known-theta-loss-curvature}}
The scalar map
$\displaystyle z\mapsto \int_0^z\Bigl(c_{at}-(\mu^*(x_t,\omega_{at})-\Lambda(u))^+\Bigr)\,du$
is differentiable, with derivative $c_{at}-(\mu^*(x_t,\omega_{at})-\Lambda(z))^+$. 
Therefore
\begin{equation}\label{eq:glm-gradient-loss}
\nabla \ell_{at}(\rho)
=\psi(x_t)\Bigl(c_{at}-(\mu^*(x_t,\omega_{at})-\Lambda(\rho^\top\psi(x_t)))^+\Bigr).
\end{equation}
The derivative of $z\mapsto c_{at}-(\mu^*(x_t,\omega_{at})-\Lambda(z))^+$ is $\one\{\mu^*(x_t,\omega_{at})>\Lambda(z)\}\Lambda'(z)\ge 0$ a.e., so the scalar map is nondecreasing and thus $\ell_{at}(\rho)$ is convex in $\rho^\top\psi(x_t)$, hence convex in $\rho$. Taking conditional expectations in \eqref{eq:glm-gradient-loss} gives
$\nabla\E[\ell_{at}(\rho)\mid x_t,\mathcal{F}_{t-1}]
= -\psi(x_t)\,m_{at}(\psi(x_t)^\top\rho)$.
By convexity of $\cB$, the entire segment $\psi(x_t)^\top\rho_a+s\psi(x_t)^\top(\rho-\rho_a)= \psi(x_t)^\top(\rho_a+s(\rho-\rho_a)),\forall s\in[0,1]$
lies in $I=[-\bar{\iota},\bar{\iota}]$ by Assumption~\ref{ass:losscurvature}. Using the fundamental theorem of calculus along the line segment from $\rho_a$ to $\rho$,
$$\begin{array}{rl}
\mathbb{E}[\ell_{at}(\rho)-\ell_{at}(\rho_a)\mid\mathcal{F}_{t-1},x_t]
&\displaystyle=\int_0^1(\rho-\rho_a)^\top\nabla \E[\ell_{at}(\rho_a+s(\rho-\rho_a))\mid\mathcal{F}_{t-1},x_t]\,ds\\
&\displaystyle=-\int_0^1 \bigl(\psi(x_t)^\top(\rho-\rho_a)\bigr)m_{at}\bigl(\psi(x_t)^\top\rho_a+s \psi(x_t)^\top(\rho-\rho_a)\bigr)\,ds.
\end{array}$$
Applying Lemma~\ref{lem:glm-monotonicity} with $u=\psi(x_t)^\top\rho_a+s\psi(x_t)^\top(\rho-\rho_a)$ yields
$-\bigl(s\psi(x_t)^\top(\rho-\rho_a)\bigr)m_{at}\bigl(\psi(x_t)^\top\rho_a+s \psi(x_t)^\top(\rho-\rho_a)\bigr)
\ge \kappa\mu_1 s^2\bigl(\psi(x_t)^\top(\rho-\rho_a)\bigr)^2,$
where the last inequality follows because $\Prob\bigl(\mu^*(x_t,\omega_{at})>\Lambda(u)\mid \Ft_{t-1},x_t\bigr)\,\Lambda'(u)\ge\kappa\mu_1$  according to Assumptions~\ref{ass:reg_res_index}, \ref{ass:losscurvature}. 
For $s>0$, divide by $s$ to obtain
$-\bigl(\psi(x_t)^\top(\rho-\rho_a)\bigr)m_{at}\bigl(\psi(x_t)^\top\rho_a+s \psi(x_t)^\top(\rho-\rho_a)\bigr)
\ge \kappa\mu_1 s\bigl(\psi(x_t)^\top(\rho-\rho_a)\bigr)^2.$
Integrating over $s\in[0,1]$ yields
$\E[\ell_{at}(\rho)-\ell_{at}(\rho_a)\mid\mathcal{F}_{t-1},x_t]
\ge \int_0^1 \kappa\mu_1 s\bigl(\psi(x_t)^\top(\rho-\rho_a)\bigr)^2\,ds
=\frac{\kappa\mu_1}{2}\bigl(\psi(x_t)^\top(\rho-\rho_a)\bigr)^2$
as claimed.
\proofend

\medskip
\begin{lemma}\label{lemma:curvature:hat:rho:at}
Fix any $a\in[A]$ and $t\in[T]$, for any $\rho_t$ adapted to $\mathcal{F}_{t-1}$, we have 
$$\frac{1}{n_{at}}\sum_{s\in\mathcal{S}_{at}}\mathbb{E}\left[\ell_{as}(\rho_t)-\ell_{as}(\rho_a)\mid\mathcal{F}_{s-1},x_s\right]\geq\frac{\kappa\mu_1}{2}\frac{1}{n_{at}}\sum_{s\in\mathcal{S}_{at}}[\psi(x_{s})^\top(\rho_t-\rho_a)]^2-\frac{1}{n_{at}}.$$
\end{lemma}
\proof{Proof of Lemma~\ref{lemma:curvature:hat:rho:at}}
Set $\epsilon=\frac{1}{2n_{at}\bar{C}_{\psi}(\kappa\mu_1\bar{\iota}+1)}$ and take a Euclidean $\epsilon$-net $\mathcal{N}_{\epsilon}$ of $\mathcal{B}$. According to Assumptions~\ref{ass:reg_res_index} and \ref{ass:losscurvature}, for any $s\in[T]$ we have  $\Prob\bigl(G(\theta_*^\top\phi(x_s,\omega_{as}))>\Lambda(u)\mid x_s,\mathcal{F}_{s-1}\bigr)\Lambda'(u)\ge \kappa\mu_1>0, \ \forall |u|\leq\overline{\iota}$. 
Lemma~\ref{lem:known-theta-loss-curvature} implies that given any $s\in[T]$ and $\bar{\rho}\in\mathcal{N}_{\epsilon}$, it always holds that
$\displaystyle\mathbb{E}\left[\ell_{as}(\bar{\rho})-\ell_{as}(\rho_a)\mid\mathcal{F}_{s-1},x_s\right]\geq\frac{\kappa\mu_1}{2}[\psi(x_s)^\top(\bar{\rho}-\rho_a)]^2$.
Since $\mathcal{B}$ has diameter $d_{\mathcal{B}}$ by Assumption~\ref{ass:reg_res_index}, $\mathcal{N}_{\epsilon}$ is finite. So after intersecting finitely many probability-one events, we have with probability one, simultaneously for all $\rho\in\mathcal{N}_{\epsilon}$ and $s\in[T]$,
\begin{equation}\label{eq:curvature:rho:fixed}
    \mathbb{E}\left[\ell_{as}(\rho)-\ell_{as}(\rho_a)\mid\mathcal{F}_{s-1},x_s\right]\geq\frac{\kappa\mu_1}{2}[\psi(x_s)^\top(\rho-\rho_a)]^2.
\end{equation}
Choose $\rho_{\epsilon}\in\mathcal{N}_{\epsilon}$ such that $\|\rho_t-\rho_{\epsilon}\|_2\leq\epsilon$. Note that 
\begin{equation}\label{eq:diff:ell:rho:hat}
\begin{array}{rl}
&\quad\displaystyle\left|[\ell_{as}(\rho_t)-\ell_{as}(\rho_a)]-[\ell_{as}(\rho_\epsilon)-\ell_{as}(\rho_a)]\right|=|\ell_{as}(\rho_t)-\ell_{as}(\rho_\epsilon)|\\
&\displaystyle=_{(i)}\left|\int_{\rho_\epsilon^\top\psi(x_{s})}^{\rho_t^\top\psi(x_{s})}[c_{as}-(G(\theta_*^\top\phi(x_{s},\omega_{as}))-\Lambda(u))^{+}]du\right|\leq_{(ii)}2|(\rho_t-\rho_\epsilon)^\top\psi(x_{s})|\leq_{(iii)}2\bar{C}_{\psi}\epsilon,
\end{array}
\end{equation}
where (i) follows from the definition of $\ell_{as}(\rho)$, (ii) follows since by definition, $c_{as}\in[0,1]$, $G(\theta_*^\top\phi(x_{s},\omega_{as}))\in[0,1]$, $\Lambda(u)\in[-1,1]$, so 
$|c_{as}-(G(\theta_*^\top\phi(x_{s},\omega_{as}))-\Lambda(u))^{+}|\leq 2$.
(iii) follows by Cauchy-Schwarz inequality, the fact that $\|\rho-\rho_{\epsilon}\|_2\leq\epsilon$ and $\|\psi(x_t)\|_2\leq\bar{C}_{\psi}$ by Assumption~\ref{ass:reg_res_index}. Hence 
$$\begin{array}{rl} &\displaystyle\quad\frac{1}{n_{at}}\sum_{s\in\mathcal{S}_{at}}\mathbb{E}\left[\ell_{as}(\rho_t)-\ell_{as}(\rho_a)\mid\mathcal{F}_{s-1},x_s\right]\\
&\displaystyle=\frac{1}{n_{at}}\sum_{s\in\mathcal{S}_{at}}\mathbb{E}\left[\ell_{as}(\rho_t)-\ell_{as}(\rho_\epsilon)+\ell_{as}(\rho_\epsilon)-\ell_{as}(\rho_a)\mid\mathcal{F}_{s-1},x_s\right]\\
&\displaystyle\geq_{(1)}\frac{\kappa\mu_1}{2}\frac{1}{n_{at}}\sum_{s\in\mathcal{S}_{at}}[\psi(x_s)^\top(\rho_\epsilon-\rho_a)]^2-2\bar{C}_{\psi}\epsilon\\
&\displaystyle=\frac{\kappa\mu_1}{2}\frac{1}{n_{at}}\sum_{s\in\mathcal{S}_{at}}[\psi(x_s)^\top(\rho_\epsilon-\rho_t)+\psi(x_s)^\top(\rho_t-\rho_a)]^2-2\bar{C}_{\psi}\epsilon\\
&\displaystyle\geq_{(2)}\frac{\kappa\mu_1}{2}\frac{1}{n_{at}}\sum_{s\in\mathcal{S}_{at}}\left\{[\psi(x_s)^\top(\rho_t-\rho_a)]^2-2\epsilon\bar{C}_{\psi}|\psi(x_s)^\top(\rho_t-\rho_a)|\right\}-2\bar{C}_{\psi}\epsilon\\
&\displaystyle\geq_{(3)}\frac{\kappa\mu_1}{2}\frac{1}{n_{at}}\sum_{s\in\mathcal{S}_{at}}\left\{[\psi(x_s)^\top(\rho_t-\rho_a)]^2-4\epsilon\bar{C}_{\psi}\bar{\iota}\right\}-2\bar{C}_{\psi}\epsilon\\
&\displaystyle=\frac{\kappa\mu_1}{2}\frac{1}{n_{at}}\sum_{s\in\mathcal{S}_{at}}[\psi(x_s)^\top(\rho_t-\rho_a)]^2-2(\kappa\mu_1\bar{\iota}+1)\bar{C}_{\psi}\epsilon=_{(4)}\frac{\kappa\mu_1}{2}\frac{1}{n_{at}}\sum_{s\in\mathcal{S}_{at}}[\psi(x_s)^\top(\rho_t-\rho_a)]^2-\frac{1}{n_{at}},
\end{array}$$
where (1) follows from \eqref{eq:curvature:rho:fixed} and \eqref{eq:diff:ell:rho:hat}, (2) follows from applying Cauchy-Schwarz inequality to $\psi(x_s)^\top(\rho_\epsilon-\rho_t)$ and the fact that $\|\rho_\epsilon-\rho_t\|_2\leq\epsilon$, $\|\psi(x_s)\|_2\leq\bar{C}_{\psi}$, (3) follows since $\psi(x_s)^\top\rho_t\in[-\bar{\iota},\bar{\iota}], \psi(x_s)^\top\rho_{a}\in[-\bar{\iota},\bar{\iota}]$, (4) follows since $\epsilon=1/(2n_{at}\bar{C}_{\psi}(\kappa\mu_1\bar{\iota}+1))$.
\proofend

\begin{lemma}[Freedman's Inequality \citep{freedman1975tail}]\label{lemma:freedman}
Consider a real-valued martingale $Y_k=\sum_{j=1}^kX_j$ with $Y_0=0$ and difference sequence $\{X_k:k=1,2,3,\ldots\}$. Assume that $X_k\leq R$ almost surely for $k\geq1$, where $R$ is a constant. Let $W_k:=\sum_{j=1}^k\mathbb{E}[X_j^2\mid\mathcal{F}_{j-1}]$ for $k\geq1$. Then for all $t\geq1$ and $\sigma^2>0$, 
$\displaystyle\mathbb{P}\left(\exists k\geq0: Y_k\geq t,\ \mbox{and}\ W_k\leq\sigma^2\right)\leq\exp\left\{-\frac{t^2/2}{\sigma^2+Rt/3}\right\}$.
\end{lemma}

\begin{lemma*}[Restatement of Lemma~\ref{lemma:H:empirical:process}]\label{lemma:H:empirical:process:appendix}
Suppose Assumptions~\ref{ass:reg_res_index}, \ref{ass:losscurvature} hold. Then given any constant $c_0>0$ and any $a\in[A]$, with probability at least $1-\delta/3$, uniformly over all $t\in[T]$, for any $\rho_t$ adapted to $\mathcal{F}_{t-1}$, we have 
$$\begin{array}{rl}
&\quad\displaystyle-\sum_{s\in\mathcal{S}_{at}}\left\{\ell_{as}(\rho_t)-\ell_{as}(\rho_a)-\mathbb{E}[\ell_{as}(\rho_t)-\ell_{as}(\rho_a)\mid\mathcal{F}_{s-1},x_{s}]\right\}\\
&\displaystyle\leq\frac{c_0}{8}\sum_{s\in\mathcal{S}_{at}}\{(\rho_t-\rho_a)^\top\psi(x_{s})\}^2+\left(\frac{144}{c_0}+2\min\{2\overline{\iota},\bar{C}_{\psi}d_{\mathcal{B}}\}+6\right)\Gamma_{at}(\delta)+4+\frac{c_0}{8},
\end{array}$$
where $\Gamma_{at}(\delta)\!:=m\log(1+2d_{\mathcal{B}}\bar{C}_{\psi}T)+\log\left(\left\lceil\log_2\left(1+n_{at}\min\{9\overline{\iota}^2,d_{\mathcal{B}}^2\bar{C}_{\psi}^2\}\right)\right\rceil+1\right)+\log\left(6T/\delta\right)$.
\end{lemma*}
\proof{Proof of Lemma~\ref{lemma:H:empirical:process}}
For any $s\in\mathcal{S}_{at}$, define $\mathcal{G}_{s-1}:=\sigma(\mathcal{F}_{s-1},x_{s})$. Let
$$H_{at}(\rho):=\sum_{s\in\mathcal{S}_{at}}\left\{\ell_{as}(\rho)-\ell_{as}(\rho_a)-\mathbb{E}[\ell_{as}(\rho)-\ell_{as}(\rho_a)\mid\mathcal{G}_{s-1}]\right\}.$$
Note that $-2\leq c_{at}-(G(\theta_*^\top\phi(x_t,\omega_{at}))-\Lambda(u))^{+}\leq1$, so for any $\rho\in\mathcal{B}$,
$$\begin{array}{rl}
&\quad\displaystyle\left|\ell_{as}(\rho)-\ell_{as}(\rho_a)-\mathbb{E}[\ell_{as}(\rho)-\ell_{as}(\rho_a)\mid\mathcal{G}_{s-1}]\right|\\
&\displaystyle\ =\bigg|\int_{\rho_a^\top\psi(x_{s})}^{\rho^\top\psi(x_{s})}[c_{as}-(G(\theta_*^\top\phi(x_{s},\omega_{as}))-\Lambda(u))^{+}]du\\
&\displaystyle\quad\quad-\mathbb{E}\bigg[\int_{\rho_a^\top\psi(x_{s})}^{\rho^\top\psi(x_{s})}[c_{as}-(G(\theta_*^\top\phi(x_{s},\omega_{as}))-\Lambda(u))^{+}]du\mid\mathcal{G}_{s-1}\bigg]\bigg|\\
&\displaystyle\leq3|(\rho-\rho_a)^\top\psi(x_{s})|\leq\min\{6\overline{\iota},3\bar{C}_{\psi}d_{\mathcal{B}}\}. 
\end{array}$$
Let 
$\displaystyle W_t(\rho):=\sum_{s\in\mathcal{S}_{at}}\mathbb{E}\left[\{\ell_{as}(\rho)-\ell_{as}(\rho_a)-\mathbb{E}[\ell_{as}(\rho)-\ell_{as}(\rho_a)\mid\mathcal{G}_{s-1}]\}^2\mid\mathcal{G}_{s-1}\right]$.
So for any $\rho\in\mathcal{B}$, 
\begin{equation}\label{eq:W:t:bound}
    W_t(\rho)\leq9\sum_{s\in\mathcal{S}_{at}}\{(\rho-\rho_a)^\top\psi(x_{s})\}^2,
\end{equation}
where $(\rho-\rho_a)^\top\psi(x_{s})$ is $\mathcal{G}_{s-1}$-measurable. Note that for any $\rho\in\mathcal{B}$, $\rho^\top\psi(x_{s})\in [-\overline{\iota},\overline{\iota}]$ and $|(\rho-\rho_a)^\top\psi(x)|\leq\|\rho-\rho_a\|_2\|\psi(x)\|_2\leq d_{\mathcal{B}}\bar{C}_{\psi}$, thus 
$\displaystyle\sum_{s\in\mathcal{S}_{at}}\{(\rho-\rho_a)^\top\psi(x_{s})\}^2\leq n_{at}\min\{9\overline{\iota}^2,d_{\mathcal{B}}^2\bar{C}_{\psi}^2\}$. Let 
\begin{equation}\label{eq:Q:at}
Q_{at}:=\left\lceil\log_2\left(1+n_{at}\min\{9\overline{\iota}^2,d_{\mathcal{B}}^2\bar{C}_{\psi}^2\}\right)\right\rceil.
\end{equation}
Consider the events $\{\mathcal{E}_q\}$ for $q=\{0\}\cup[Q_{at}]$, where 
\begin{equation}\label{eq:event:E:q}
    \mathcal{E}_q:=\begin{cases}
    \displaystyle\big\{2^{q-1}<\sum_{s\in\mathcal{S}_{at}}\{(\rho-\rho_a)^\top\psi(x_{s})\}^2\leq 2^q\big\}&\text{if }q\geq1\\
    \displaystyle\big\{0\leq\sum_{s\in\mathcal{S}_{at}}\{(\rho-\rho_a)^\top\psi(x_{s})\}^2\leq 1\big\}&\text{if }q=0
    \end{cases}
\end{equation}
Fix any $x>0$. On the event $\mathcal{E}_q$, $\displaystyle W_t(\rho)\leq9\sum_{s\in\mathcal{S}_{at}}\{(\rho-\rho_a)^\top\psi(x_{s})\}^2\leq9\times2^{q}$.
So applying Freedman's inequality (Lemma~\ref{lemma:freedman}) with $\sigma^2=9\times2^{q}$, we have 
$$\mathbb{P}\left(-H_{at}(\rho)\geq3\sqrt{2}\sqrt{2^qx}+2\min\{2\overline{\iota},\bar{C}_{\psi}d_{\mathcal{B}}\}x,\sum_{s\in\mathcal{S}_{at}}\{(\rho-\rho_a)^\top\psi(x_{s})\}^2\leq 2^q\right)\leq e^{-x}.$$
Additionally, on event $\mathcal{E}_q$, $\displaystyle 6\sqrt{x\bigg(1+\sum_{s\in\mathcal{S}_{at}}\{(\rho-\rho_a)^\top\psi(x_{s})\}^2\bigg)}\geq6\sqrt{2^{q-1}x}=3\sqrt{2}\sqrt{2^qx}$.
Thus on event $\mathcal{E}_q$,
$$\begin{array}{rl}
&\quad\displaystyle-H_{at}(\rho)\geq6\sqrt{x\bigg(1+\sum_{s\in\mathcal{S}_{at}}\{(\rho-\rho_a)^\top\psi(x_{s})\}^2\bigg)}+2\min\{2\overline{\iota},\bar{C}_{\psi}d_{\mathcal{B}}\}x\\
&\displaystyle\Rightarrow -H_{at}(\rho)\geq3\sqrt{2}\sqrt{2^qx}+2\min\{2\overline{\iota},\bar{C}_{\psi}d_{\mathcal{B}}\}x.
\end{array}$$
Summing the probability bound above over $q\in\{0\}\cup[Q_{at}]$, we have 
$$\mathbb{P}\left(-H_{at}(\rho)\geq6\sqrt{x+x\sum_{s\in\mathcal{S}_{at}}\{(\rho-\rho_a)^\top\psi(x_{s})\}^2}+2\min\{2\overline{\iota},\bar{C}_{\psi}d_{\mathcal{B}}\}x\right)\leq 2(Q_{at}+1)e^{-x}.$$
Setting $x=\log\left(2(Q_{at}+1)/\delta\right)$ above, then for any fixed $\rho\in\mathcal{B}$, with probability at least $1-\delta$,
\begin{equation}\label{eq:prob:H:at:bound}
-H_{at}(\rho)<6\sqrt{\log\left(\frac{2(Q_{at}+1)}{\delta}\right)\sum_{s\in\mathcal{S}_{at}}\{(\rho-\rho_a)^\top\psi(x_{s})\}^2}+(2\min\{2\overline{\iota},\bar{C}_{\psi}d_{\mathcal{B}}\}+6)\log\left(\frac{2(Q_{at}+1)}{\delta}\right).
\end{equation}
Set $\epsilon=\frac{1}{T\bar{C}_{\psi}}$ and take an Euclidean $\epsilon$-net $\mathcal{N}_\epsilon$ of $\mathcal{B}$. Since $\mathcal{B}$ has diameter $d_{\mathcal{B}}$ by Assumption~\ref{ass:reg_res_index}, $|\mathcal{N}_\epsilon|\leq\left(1+\frac{2d_{\mathcal{B}}}{\epsilon}\right)^m=(1+2d_{\mathcal{B}}\bar{C}_{\psi}T)^m$.
Note that \eqref{eq:prob:H:at:bound} further implies that with probability at least $1-\delta/3$, uniformly over all $\rho\in\mathcal{N}_\epsilon$ and $t\in[T]$, 
\begin{equation}\label{eq:H:at:rho:bound}
\begin{array}{rl}
-H_{at}(\rho)&\displaystyle<6\sqrt{\log\left(\frac{6T|\mathcal{N}_\epsilon|(Q_{at}+1)}{\delta}\right)\sum_{s\in\mathcal{S}_{at}}\{(\rho-\rho_a)^\top\psi(x_{s})\}^2}\\
&\quad\displaystyle+(2\min\{2\overline{\iota},\bar{C}_{\psi}d_{\mathcal{B}}\}+6)\log\left(\frac{6T|\mathcal{N}_\epsilon|(Q_{at}+1)}{\delta}\right)\\
&\displaystyle\leq6\sqrt{\Gamma_{at}(\delta)\sum_{s\in\mathcal{S}_{at}}\{(\rho-\rho_a)^\top\psi(x_{s})\}^2}+(2\min\{2\overline{\iota},\bar{C}_{\psi}d_{\mathcal{B}}\}+6)\Gamma_{at}(\delta),
\end{array}
\end{equation}
where 
\begin{equation}\label{eq:gamma:at}
    \Gamma_{at}(\delta)\!:=m\log(1+2d_{\mathcal{B}}\bar{C}_{\psi}T)+\log(Q_{at}+1)+\log\left(6T/\delta\right).
\end{equation}
Let $\rho_{\epsilon}\in\mathcal{N}_\epsilon$ satisfy $\|\rho_t-\rho_{\epsilon}\|_2\leq\epsilon$. Note that 
$$\begin{array}{rl}
&\quad\displaystyle\left|[\ell_{as}(\rho_t)-\ell_{as}(\rho_a)]-[\ell_{as}(\rho_\epsilon)-\ell_{as}(\rho_a)]\right|\\
&\displaystyle\ =|\ell_{as}(\rho_t)-\ell_{as}(\rho_\epsilon)|=_{(i)}\left|\int_{\rho_\epsilon^\top\psi(x_{s})}^{\rho_t^\top\psi(x_{s})}[c_{as}-(G(\theta_*^\top\phi(x_{s},\omega_{as}))-\Lambda(u))^{+}]du\right|\\
&\displaystyle\leq_{(ii)}2|(\rho_t-\rho_\epsilon)^\top\psi(x_{s})|\leq_{(iii)}2\bar{C}_{\psi}\epsilon\leq_{(iv)}\frac{2}{n_{at}}.
\end{array}$$
where (i) follows by definition of $\ell_{as}(\rho)$ in \eqref{eq:known-theta--loss}, (ii) follows since $|c_{as}-(G(\theta_*^\top\phi(x_{s},\omega_{as}))-\Lambda(u))^{+}|\leq 2$ by definition,
(iii) follows by Cauchy-Schwarz inequality, the fact that $\|\rho_t-\rho_{\epsilon}\|_2\leq\epsilon$ and $\|\psi(x_t)\|_2\leq\bar{C}_{\psi}$ by Assumption~\ref{ass:reg_res_index}, (iv) follows since $\epsilon=1/(T\bar{C}_\psi)$ as defined before. The above inequality implies that
$$|H_{at}(\rho_t)-H_{at}(\rho_{\epsilon})|=\left|\sum_{s\in\mathcal{S}_{at}}[\ell_{as}(\rho_t)-\ell_{as}(\rho_\epsilon)]-\mathbb{E}[\ell_{as}(\rho_t)-\ell_{as}(\rho_\epsilon)\mid\mathcal{G}_{s-1}]\right|\leq\sum_{s\in\mathcal{S}_{at}}\frac{4}{n_{at}}\leq4.$$
Therefore, 
\begin{equation}\label{eq:H:at:diff:net}
|H_{at}(\rho_t)-H_{at}(\rho_{\epsilon})|\leq4.
\end{equation}
Note that 
$\displaystyle\sqrt{\sum_{s\in\mathcal{S}_{at}}\{(\rho_t-\rho_a)^\top\psi(x_{s})\}^2}=\|((\rho_t-\rho_a)^\top\psi(x_{s}))_{j\in[n_{at}]}\|_2$.
Thus by triangle inequality,
\begin{equation}\label{eq:diff:sqrt:sum:rho:net}
\begin{array}{rl}
\sqrt{\sum_{s\in\mathcal{S}_{at}}\{(\rho_t-\rho_a)^\top\psi(x_{s})\}^2}-\sqrt{\sum_{s\in\mathcal{S}_{at}}\{(\rho_\epsilon-\rho_a)^\top\psi(x_{s})\}^2}&\displaystyle\leq\left\|((\rho_t-\rho_\epsilon)^\top\psi(x_{s}))_{j\in[n_{at}]}\right\|_2\\
&\displaystyle\leq\sqrt{n_{at}(\bar{C}_{\psi}\epsilon)^2}=\frac{1}{\sqrt{n_{at}}}.
\end{array}
\end{equation}
On the event that \eqref{eq:H:at:rho:bound} holds uniformly over all $\rho\in\mathcal{N}_\epsilon$ and $t\in[T]$, we have 
$$\begin{array}{rl}
-H_{at}(\rho_t)&\displaystyle\leq_{(i)}-H_{at}(\rho_\epsilon)+4\\
&\displaystyle\leq_{(ii)}6\sqrt{\Gamma_{at}(\delta)\sum_{s\in\mathcal{S}_{at}}\{(\rho_\epsilon-\rho_a)^\top\psi(x_{s})\}^2}+(2\min\{2\overline{\iota},\bar{C}_{\psi}d_{\mathcal{B}}\}+6)\Gamma_{at}(\delta)+4\\
&\displaystyle\leq_{(iii)}6\sqrt{\Gamma_{at}(\delta)}\left(\sqrt{\sum_{s\in\mathcal{S}_{at}}\{(\rho_t-\rho_a)^\top\psi(x_{s})\}^2}+\frac{1}{\sqrt{n_{at}}}\right)+(2\min\{2\overline{\iota},\bar{C}_{\psi}d_{\mathcal{B}}\}+6)\Gamma_{at}(\delta)+4\\
&\displaystyle\leq_{(iv)}6\sqrt{2\Gamma_{at}(\delta)}\sqrt{\sum_{s\in\mathcal{S}_{at}}\{(\rho_t-\rho_a)^\top\psi(x_{s})\}^2+1}+(2\min\{2\overline{\iota},\bar{C}_{\psi}d_{\mathcal{B}}\}+6)\Gamma_{at}(\delta)+4\\
&\displaystyle\leq_{(v)}\frac{c_0}{8}\sum_{s\in\mathcal{S}_{at}}\{(\rho_t-\rho_a)^\top\psi(x_{s})\}^2+\left(\frac{144}{c_0}+2\min\{2\overline{\iota},\bar{C}_{\psi}d_{\mathcal{B}}\}+6\right)\Gamma_{at}(\delta)+4+\frac{c_0}{8}
\end{array}$$
where (i) follows from \eqref{eq:H:at:diff:net}, (ii) follows from \eqref{eq:H:at:rho:bound}, (iii) holds from \eqref{eq:diff:sqrt:sum:rho:net}, (iv) holds from the fact that 
$\displaystyle\sqrt{s}+\frac{1}{\sqrt{n_{at}}}\leq\sqrt{2(s+1/n_{at})}$,
where $\displaystyle s=\sum_{s\in\mathcal{S}_{at}}\{(\rho_t-\rho_a)^\top\psi(x_{s})\}^2$ and $1/n_{at}\leq 1$, (v) follows by applying $2ab\leq\epsilon a^2+\epsilon^{-1}b^2$ to $\displaystyle a=\sqrt{\sum_{s\in\mathcal{S}_{at}}\{(\rho_t-\rho_a)^\top\psi(x_{s})\}^2+1}$, $b=6\sqrt{2\Gamma_{at}(\delta)}$, $\displaystyle\epsilon=\frac{c_0}{4}$, so that 
$$6\sqrt{2\Gamma_{at}(\delta)}\sqrt{\sum_{s\in\mathcal{S}_{at}}\{(\rho_t-\rho_a)^\top\psi(x_{s})\}^2+1}\leq\frac{c_0}{8}\sum_{s\in\mathcal{S}_{at}}\{(\rho_t-\rho_a)^\top\psi(x_{s})\}^2+\frac{144}{c_0}\Gamma_{at}(\delta)+\frac{c_0}{8}.$$
Hence with probability at least $1-\delta/3$, uniformly over all $t\in[T]$ we have
$$-H_{at}(\rho_t)\leq\frac{c_0}{8}\sum_{s\in\mathcal{S}_{at}}\{(\rho_t-\rho_a)^\top\psi(x_{s})\}^2+\left(\frac{144}{c_0}+2\min\{2\overline{\iota},\bar{C}_{\psi}d_{\mathcal{B}}\}+6\right)\Gamma_{at}(\delta)+4+\frac{c_0}{8},$$
where $\Gamma_{at}(\delta)$ is defined as in \eqref{eq:gamma:at}, thus the result follows.
\proofend

\subsection{Index Estimation under Known Reward Function}
Before introducing the formal proof, we first provide the proof's intuition. The argument follows a standard localized empirical-process approach for 
M-estimation, adapted to our martingale setting. The estimator's empirical 
optimality is combined with a population curvature lower bound, while the 
stochastic deviation is controlled uniformly over the parameter space using 
Freedman's martingale inequality, a peeling argument, and an $\epsilon$-net. 
Readers familiar with localized empirical-process and
martingale concentration arguments may skip the following intuition and proceed 
directly to the formal proof. 

\emph{(i) The optimality inequality.} Because $\hat{\rho}_{at}$ minimizes the empirical primitive loss, the empirical excess loss at $\hat{\rho}_{at}$ cannot be positive. Plug $\rho=\hat{\rho}_{at}$ into 
$$\sum_{s \in \mathcal{S}_{at}}\!\big\{\ell_{as}(\rho)-\ell_{as}(\rho_a)\big\}
\;=\;\sum_{s \in \mathcal{S}_{at}}\!\!\mathbb{E}\big[\ell_{as}(\rho)-\ell_{as}(\rho_a)\,\big|\,\mathcal{F}_{s-1},x_s\big]
\;+\;H_{at}(\rho)$$
and use the population curvature lower bound from Lemma~\ref{lem:known-theta-loss-curvature} (with a small Lipschitz adjustment so that the curvature holds uniformly over $\rho$, which costs only an additive $1$). After rearranging, one obtains $\frac{\kappa\mu_1}{2}\sum_{s\in\mathcal{S}_{at}}\!\big\{\psi(x_s)^\top(\hat{\rho}_{at}-\rho_a)\big\}^2
\;\leq\;-H_{at}(\hat{\rho}_{at})\;+1$.
The left side is the local quadratic ``signal'' we want to bound, and the right side is essentially the ``noise''.

\emph{(ii) Bounding the noise at one fixed $\rho$.} For each fixed $\rho$, every term inside $H_{at}(\rho)$ is uniformly bounded and has a conditional variance controlled by the same quadratic quantity $\sum_{s\in\mathcal{S}_{at}}\{\psi(x_s)^\top(\rho-\rho_a)\}^2$. Freedman's \citep[][]{freedman1975tail} martingale inequality (Lemma~\ref{lemma:freedman}) therefore gives, for each fixed $\rho$, $|H_{at}(\rho)|\;\lesssim\;\sqrt{\Big(\textstyle\sum_{s}\{\psi(x_s)^\top(\rho-\rho_a)\}^2\Big)\cdot\log(T/\delta)}\;+\;\log(T/\delta)$. The key point is that the stochastic error enters through the same quadratic form that governs the population curvature, allowing the error term to be absorbed into the curvature in the final bound.

\emph{(iii) Making the bound uniform in $\rho$.} Since $\hat{\rho}_{at}$ is itself random, we need step (ii) to hold for all $\rho$ simultaneously. We do this in two passes. First, a peeling argument splits the parameter space into dyadic shells based on the size of $\sum_{s}\{\psi(x_s)^\top(\rho-\rho_a)\}^2$; on each shell the variance proxy is replaced by a deterministic ceiling and Freedman's bound applies. Summing over the shells costs only a $\log\log$ factor. Second, an $\epsilon$-net argument extends the bound from a finite grid of $\rho$'s to all of $\mathcal{B}$ via the Lipschitz continuity of the loss primitive. The end product is a uniform bound on $|H_{at}(\rho)|$ that still scales with the same quadratic quantity.

\emph{(iv) Closing the loop.} Plug $\rho=\hat{\rho}_{at}$ into the uniform bound from step (iii) and combine with step (i). The resulting inequality has the schematic form $\text{quadratic}\lesssim\sqrt{\text{quadratic}}\cdot\sqrt{\log T}+\log T$. Applying $2ab\leq\epsilon a^2+\epsilon^{-1}b^2$ with a suitable $\epsilon$ absorbs the square-root term into the quadratic, leaving $1/n_{at}\sum_{s\in\mathcal S_{at}}\big\{\psi(x_s)^\top(\hat{\rho}_{at}-\rho_a)\big\}^2\;\lesssim\;\log T/n_{at}$.
Folding in the regularization $\eta_1I_m$ to ensure invertibility and applying Cauchy--Schwarz then yields the stated confidence radius.

\begin{proposition}[Restatement of Proposition~\ref{prop:rho:a:confidence:theta:known}]\label{prop:rho:a:confidence:theta:known:restatement}
Suppose Assumptions~\ref{ass:reg_res_index}, \ref{ass:losscurvature} hold.
When $\mu^*$ is known, given any $\delta>0$, with probability at least $1-\delta$, uniformly over all $t\in[T]$ and $a\in[A]$, 
\begin{equation}\label{eq:index:error:theta:known:restate}
    |(\hat{\rho}_{at}-\rho_a)^\top\psi(x_t)|\leq B_{at}^*\|\psi(x_t)\|_{V_{at}(\eta_1)^{-1}},
\end{equation}
where 
\begin{equation}\label{eq:B:at:star}   B_{at}^*:=\sqrt{\frac{8}{3\kappa\mu_1}\left(\frac{144}{\kappa\mu_1}+2\min\{2\overline{\iota},\bar{C}_{\psi}d_{\mathcal{B}}\}+6\right)\Gamma_{at}^*(\delta)+40/(3\kappa\mu_1)+1/3+\eta_1d_{\mathcal{B}}^2},
\end{equation}
$$\Gamma_{at}^*(\delta)\!:=m\log(1+2d_{\mathcal{B}}\bar{C}_{\psi}T)+\log\left(\left\lceil\log_2\big(1+n_{at}\min\{9\overline{\iota}^2,d_{\mathcal{B}}^2\bar{C}_{\psi}^2\}\right)\right\rceil+1\big)+\log(6AT/\delta),$$ 
and $\displaystyle V_{at}(\eta_1)=\eta_1I_m+\sum_{s\in\mathcal{S}_{at}}\psi(x_{s})\psi(x_{s})^\top$. 
\end{proposition}
\proof{Proof of Proposition~\ref{prop:rho:a:confidence:theta:known:restatement}}
For any $s\in[T]$, define $\mathcal{G}_{s-1}:=\sigma(\mathcal{F}_{s-1},x_{s})$. Note that 
\begin{equation}\label{eq:ineq:key:rho:known:theta}
\begin{array}{rl}
0&\geq_{(1)}\displaystyle\sum_{s\in\mathcal{S}_{at}}\ell_{as}(\hat{\rho}_{at})-\sum_{s\in\mathcal{S}_{at}}\ell_{as}(\rho_a)=_{(2)}\sum_{s\in\mathcal{S}_{at}}\mathbb{E}\left[\ell_{as}(\hat{\rho}_{at})-\ell_{as}(\rho_a)\mid\mathcal{G}_{s-1}\right]+H_{at}(\hat{\rho}_{at})\\
&\displaystyle\geq_{(3)}\frac{\kappa\mu_1}{2}\sum_{s\in\mathcal{S}_{at}}[\psi(x_{s})^\top(\hat{\rho}_{at}-\rho_a)]^2-1+H_{at}(\hat{\rho}_{at}),
\end{array}
\end{equation}
where (1) holds because $\displaystyle\hat{\rho}_{at}=\argmin_{\rho\in\mathcal{B}}\sum_{s\in\mathcal{S}_{at}}\ell_{as}(\rho)$, (2) holds with  
$$H_{at}(\rho)=\sum_{s\in\mathcal{S}_{at}}\left\{\ell_{as}(\rho)-\ell_{as}(\rho_a)-\mathbb{E}[\ell_{as}(\rho)-\ell_{as}(\rho_a)\mid\mathcal{G}_{s-1}]\right\}$$
defined as \eqref{eq:H:at:rho:known:theta},  and (3) follows from Lemma~\ref{lemma:curvature:hat:rho:at}. So \eqref{eq:ineq:key:rho:known:theta} implies that   
\begin{equation}\label{eq:rho:ineq:upper:known:theta}
\frac{\kappa\mu_1}{2}\sum_{s\in\mathcal{S}_{at}}[\psi(x_{s})^\top(\hat{\rho}_{at}-\rho_a)]^2\leq-H_{at}(\hat{\rho}_{at})+1.
\end{equation}
 Lemma~\ref{lemma:H:empirical:process} implies that with probability at least $1-\delta/3$, uniformly over all $t\in[T]$ and $a\in[A]$,
\begin{equation}\label{eq:H:at:rho:at:hat:bound:known:theta}
-H_{at}(\hat{\rho}_{at})\leq\frac{\kappa\mu_1}{8}\sum_{s\in\mathcal{S}_{at}}\{(\hat{\rho}_{at}-\rho_a)^\top\psi(x_{s})\}^2+\left(\frac{144}{\kappa\mu_1}+2\min\{2\overline{\iota},\bar{C}_{\psi}d_{\mathcal{B}}\}+6\right)\Gamma_{at}^*(\delta)+4+\frac{\kappa\mu_1}{8},
\end{equation}
where 
$\Gamma_{at}^*(\delta)\:=m\log(1+2d_{\mathcal{B}}\bar{C}_{\psi}T)+\log(Q_{at}+1)+\log\left(6AT/\delta\right)$. So \eqref{eq:rho:ineq:upper:known:theta} further implies that 
\begin{equation}\label{eq:rho:diff:psi:x:known:theta}
    \frac{3\kappa\mu_1}{8}\sum_{s\in\mathcal{S}_{at}}\{(\hat{\rho}_{at}-\rho_a)^\top\psi(x_{s})\}^2\leq\left(\frac{144}{\kappa\mu_1}+2\min\{2\overline{\iota},\bar{C}_{\psi}d_{\mathcal{B}}\}+6\right)\Gamma_{at}^*(\delta)+5+\frac{\kappa\mu_1}{8}.
\end{equation}
Combining with the fact that $\eta_1\|\hat{\rho}_{at}-\rho_a\|^2\leq\eta_1d_{\mathcal{B}}^2$, and recall that $V_{at}(\eta_1)=\eta_1I_m+\sum_{s\in\mathcal{S}_{at}}\psi(x_{s})\psi(x_{s})^\top$,
\eqref{eq:rho:diff:psi:x:known:theta} then implies that with probability at least $1-\delta/3$, uniformly over all $t\in[T]$ and $a\in[A]$,
\begin{equation}\label{eq:rho:diff:V:square:known:theta}
\begin{array}{rl}
\displaystyle\|\hat{\rho}_{at}-\rho_a\|_{V_{at}(\eta_1)}^2&\displaystyle\leq\frac{8}{3\kappa\mu_1}\left(\frac{144}{\kappa\mu_1}+2\min\{2\overline{\iota},\bar{C}_{\psi}d_{\mathcal{B}}\}+6\right)\Gamma_{at}^*(\delta)+40/(3\kappa\mu_1)+1/3+\eta_1d_{\mathcal{B}}^2.
\end{array}
\end{equation}
By Cauchy-Schwarz inequality, 
$|(\hat{\rho}_{at}-\rho_a)^\top\psi(x_t)|\leq\|\psi(x_t)\|_{V_{at}(\eta_1)^{-1}}\|\hat{\rho}_{at}-\rho_a\|_{V_{at}(\eta_1)}$, so the result follows.
\proofend

\subsection{Regret Analysis under Known Reward Function}
\proof{Proof of Proposition~\ref{prop:cumulative:regret:theta:known}}
Let $\mathcal{E}$ denote the event that uniformly over all $a\in[A]$ and $t\in[T]$,
\begin{equation}\label{eq:rho:at:star:diff}
    |(\hat{\rho}_{at}-\rho_a)^\top\psi(x_t)|\leq B_{at}^*\|\psi(x_t)\|_{V_{at}(\eta_1)^{-1}},
\end{equation}
where $\displaystyle V_{at}(\eta_1)$ is defined as in Proposition~\ref{prop:rho:a:confidence:theta:known} and $B_{at}^*$ is defined as \eqref{eq:B:at:star}. On event $\mathcal{E}$, $\widetilde{\sigma}_{at}\geq\sigma_{at}^*$ for all $a\in[A], t\in[T]$, where $\widetilde{\sigma}_{at}$ is defined as \eqref{eq:tilde:sigma:theta:known}. Then Theorem~\ref{thm:optimistic-regret-decomposition} implies that on $\mathcal{E}$,
$$\begin{array}{rl}
\displaystyle\mathbb{E}\left[\sum_{t=1}^T\Delta_t(\tilde{\pi})\big|\mathcal{E}\right]&\displaystyle\leq\mathbb{E}\left[\sum_{t=1}^T\sum_{a=1}^A(\widetilde{\sigma}_{at}-\sigma_{at}^*)\mathbb{I}\{a\in\mathcal{A}_t\}\big|\mathcal{E}\right]\\
&\displaystyle\leq\mathbb{E}\left[\sum_{a=1}^A\sum_{t=1}^T\left(\Lambda\left(\hat{\rho}_{at}^\top\psi(x_t)+B_{at}^*\|\psi(x_t)\|_{V_{at}(\eta_1)^{-1}}\right)-\Lambda(\rho_a^\top\psi(x_t))\right)\mathbb{I}\{a\in\mathcal{A}_t\}\big|\mathcal{E}\right]\\
&\displaystyle\leq_{(i)}\mathbb{E}\left[\sum_{a=1}^A\sum_{t=1}^TL\left|\hat{\rho}_{at}^\top\psi(x_t)+B_{at}^*\|\psi(x_t)\|_{V_{at}(\eta_1)^{-1}}-\rho_a^\top\psi(x_t)\right|\mathbb{I}\{a\in\mathcal{A}_t\}\big|\mathcal{E}\right]\\
&\displaystyle\leq_{(ii)}2LB_{T}^*\mathbb{E}\left[\sum_{a=1}^A\sum_{t=1}^T\|\psi(x_t)\|_{V_{at}(\eta_1)^{-1}}\mathbb{I}\{a\in\mathcal{A}_t\}\big|\mathcal{E}\right],
\end{array}$$
where (i) follows from the Lipschitz property of $\Lambda$ according to Assumption~\ref{ass:reg_res_index}, (ii) follows from \eqref{eq:rho:at:star:diff}, and $B_T^*=\sup_{a\in[A],t\leq T}B_{at}^*=\sqrt{\frac{8}{3\kappa\mu_1}\left(\frac{144}{\kappa\mu_1}+2\min\{2\overline{\iota},\bar{C}_{\psi}d_{\mathcal{B}}\}+6\right)\Gamma_{T}(\delta)+\frac{40}{3\kappa\mu_1}+\frac{1}{3}+\eta_1d_{\mathcal{B}}^2}$,
$$\Gamma_{T}(\delta)\!:=m\log(1+2d_{\mathcal{B}}\bar{C}_{\psi}T)+\log\left(\left\lceil\log_2\big(1+T\min\{9\overline{\iota}^2,d_{\mathcal{B}}^2\bar{C}_{\psi}^2\}\right)\right\rceil+1\big)+\log(6AT/\delta).$$ 
Then following similar proof steps as in the proof for Theorem~\ref{thm:expected:total:regret}, we have that on $\mathcal{E}$, 
$$\begin{array}{rl}
\displaystyle\sum_{a=1}^A\sum_{t=1}^T\|\psi(x_t)\|_{V_{at}(\eta_1)^{-1}}\mathbb{I}\{a\in\mathcal{A}_t\}&\displaystyle=\sum_{a=1}^A\sum_{s\in\mathcal{S}_{a, T+1}}\|\psi(x_{s})\|_{V_{as}(\eta_1)^{-1}}\\
&\displaystyle\leq A\sqrt{T(1+\bar{C}_{\psi}^2/\eta_1)m\log\left(1+T\bar{C}_{\psi}^2/(\eta_1m)\right)}.
\end{array}$$
Hence on $\mathcal{E}$, 
\begin{equation}\label{eq:high:prob:regret:E:known:theta}
    \mathbb{E}\left[\sum_{t=1}^T\Delta_t(\tilde{\pi})\,\bigg|\,\mathcal{E}\right]\!\leq2LB_{T}^*A\sqrt{T(1+\bar{C}_{\psi}^2/\eta_1)m\log\left(1+T\bar{C}_{\psi}^2/(\eta_1m)\right)}.
\end{equation}
Note that it always holds that $\widetilde{\sigma}_{at}-\sigma_{at}^*\leq2$, and recall from Proposition~\ref{prop:rho:a:confidence:theta:known} that $\mathbb{P}(\mathcal{E})\geq1-\delta$. On $\mathcal{E}^c$, the per-period regret is at most $1+2A$. Taking $\delta=1/T$ gives a failure event contribution of order $\mathrm{O}(A)$.
Thus ignoring logarithmic factors we have $\displaystyle\mathbb{E}\left[\sum_{t=1}^T\Delta_t(\tilde{\pi})\right]\leq\widetilde{O}\left(Am\sqrt{T}\right)$. 
\proofend

\section{Reward Estimation}\label{appendix:reward}
Lemma~\ref{lemma:technical:GLM:bandit:1} below is Theorem 1 from \citet{abbasi2011improved}:
\begin{lemma}\label{lemma:technical:GLM:bandit:1}
Let $\{v_s:s\geq0\}$ be an $\mathbb{R}^d$-valued stochastic process adapted to filtration $\{\mathcal{H}_s: s\geq0\}$, $\{\epsilon_s:s\geq1\}$ be a real-valued stochastic process adapted to $\{\mathcal{H}_s\}$. Assume that $\epsilon_s$ is conditionally sub-Gaussian such that there exists some $\gamma>0$ such that for any $u\in\mathbb{R}$, $s\geq1$, 
$\mathbb{E}\left[\exp\left(u\epsilon_s\right)|\mathcal{H}_{s-1}\right]\leq\exp\left(\frac{u^2\gamma^2}{2}\right)\ \textrm{a.s.}$
Assume that $V$ is a $d\times d$ positive definite matrix. For any $t\geq0$, define $\bar{V}_t=V+\sum_{s=1}^tv_sv_s^\top$, then for any $\delta>0$, with probability at least $1-\delta$, for all $t\geq0$, $\displaystyle\bigg\|\sum_{s=1}^t\epsilon_sv_s\bigg\|_{\bar{V}_t^{-1}}^2\leq2\gamma^2\log\left(\frac{\mathrm{det}(\bar{V}_t)^{1/2}\mathrm{det}(V)^{-1/2}}{\delta}\right)$.
\end{lemma}
Lemma~\ref{lemma:G:diff:hat:theta} below provides a stronger result which implies Lemma~\ref{lemma:G:diff:hat:theta:main} directly:
\begin{lemma}\label{lemma:G:diff:hat:theta}
Suppose Assumptions \ref{ass:G:phi:regularity}, \ref{ass:reward:perturbed} hold. Fix any $\delta>0$. Then for any $x\in\mathcal{X}$, $\omega\in\Omega$, and $t\geq2$, with probability at least $1-\delta$, the following holds:
\begin{equation}\label{eq:theta:bound:x:w}
\begin{array}{rl}
&\quad\displaystyle\left|\theta_*^\top \phi(x,\omega)-\hat{\theta}_{t-1}^\top \phi(x,\omega)\right|\\
&\displaystyle\leq\frac{2}{\min\{1,\underline{\mu}\}}\|\phi(x,\omega)\|_{\Phi_{t-1}^{-1}}\bigg(\gamma_0\sqrt{d\log(1+t\bar{C}_{\phi}^2/\eta_0)+2\log(1/\delta)}+\sqrt{\eta_0}\bar{\alpha}\bigg).
\end{array}
\end{equation}
Particularly, with probability at least $1-\delta/3$, uniformly over all $a\in[A]$ and $t\geq 2$, 
\begin{equation}\label{eq:theta:bound:uniform}
\big|(\hat\theta_{t-1}-\theta_*)^\top\phi(x_t,\omega_{at})\big|\;\leq\;\beta_t\;\big\|\phi(x_t,\omega_{at})\big\|_{\Phi_{t-1}^{-1}},
\end{equation}
where 
\begin{equation}\label{eq:beta:t-1}
    \beta_{t}:=\frac{2}{\min\{1,\underline{\mu}\}}\bigg(\gamma_0\sqrt{d\log(1+t\bar{C}_{\phi}^2/\eta_0)+2\log(3/\delta)}+\sqrt{\eta_0}\bar{\alpha}\bigg).
\end{equation}
\end{lemma}
The proof of Lemma~\ref{lemma:G:diff:hat:theta} follows closely from Proposition 1 of \citet{filippi2010parametric}. 
\proof{Proof of Lemma~\ref{lemma:G:diff:hat:theta}}
Let $\displaystyle g_t(\theta)=\sum_{k=1}^{t-1}G(\theta^\top \phi(x_k,\omega_{a_kk}))\phi(x_k,\omega_{a_kk})+\eta_0\theta$ be the invertible function such that $\displaystyle g_t(\tilde{\theta}_{t-1})=\sum_{k=1}^{t-1}r_k\phi(x_k,\omega_{a_kk})$, where $\tilde{\theta}_{t-1}$ is the unique solution to \eqref{eq:theta:esitmate:t-1}:
\begin{equation}\label{eq:theta:esitmate:t-1}
-\eta_0\theta+\!\sum_{k=1}^{t-1}\left\{r_k-G\left(\theta^\top \phi(x_k,\omega_{a_kk})\right)\right\}\phi(x_k,\omega_{a_kk})=0,
\end{equation}
By Assumption~\ref{ass:G:phi:regularity}, $\nabla g_t$ is continuous, so by the Fundamental Theorem of Calculus, 
\begin{equation}\label{eq:g:calculus}
g_t(\theta_*)-g_t(\tilde{\theta}_{t-1})=\Gamma_t(\theta_*-\tilde{\theta}_{t-1}),
\end{equation}
where $\displaystyle \Gamma_t=\int_0^1\nabla g_t(s\theta_*+(1-s)\tilde{\theta}_{t-1})ds$,
and 
$$\begin{array}{rl}
\nabla g_t(\theta)&\displaystyle=\eta_0I_d+\sum_{k=1}^{t-1}\phi(x_k,\omega_{a_kk})\phi(x_k,\omega_{a_kk})^\top G'\left(\theta^\top \phi(x_k,\omega_{a_kk})\right)\\
&\displaystyle\succeq_{(i)}\eta_0I_d+\underline{\mu}\sum_{k=1}^{t-1}\phi(x_k,\omega_{a_kk})\phi(x_k,\omega_{a_kk})^\top,
\end{array}$$
where (i) follows because $G'\left(\theta^\top \phi(x_k,\omega_{a_kk})\right)\geq\underline{\mu}$ according to Assumption~\ref{ass:G:phi:regularity}. Hence for any $t\in[T]$,
\begin{equation}\label{eq:G:t:ineq}
\Gamma_t\succeq\min\{1,\underline{\mu}\}\Phi_{t-1}\succeq\min\{1,\underline{\mu}\}\eta_0I_d\succ0,
\end{equation}
So $\Gamma_t$ is positive definite and is non-singular for any $t\in[T]$. Therefore, 
\begin{equation}\label{eq:diff:G:abs}
\begin{array}{rl}
\left|\theta_*^\top \phi(x,\omega)-\tilde{\theta}_{t-1}^\top \phi(x,\omega)\right|&\displaystyle=_{(1)}|\phi(x,\omega)^\top \Gamma_t^{-1}\{g_t(\theta_*)-g_t(\tilde{\theta}_{t-1})\}|\\
&\displaystyle\leq_{(2)} \|\phi(x,\omega)\|_{\Gamma_t^{-1}}\|g_t(\theta_*)-g_t(\tilde{\theta}_{t-1})\|_{\Gamma_t^{-1}},
\end{array}
\end{equation}
where (1) of \eqref{eq:diff:G:abs} follows from \eqref{eq:g:calculus}, and (2) follows from Cauchy-Schwarz inequality and the fact that $\Gamma_t^{-1}$ is positive definite. \eqref{eq:G:t:ineq} implies that $\Gamma_t\succeq\min\{1,\underline{\mu}\}\Phi_{t-1}$, which further implies $\max\{1,1/\underline{\mu}\}\Phi_{t-1}^{-1}\succeq \Gamma_t^{-1}$, so $\displaystyle\|v\|_{\Gamma_t^{-1}}\leq\max\left\{1/\sqrt{\underline{\mu}},1\right\}\|v\|_{\Phi_{t-1}^{-1}},\ \forall v\in\mathbb{R}^d$.
Hence, \eqref{eq:diff:G:abs} further implies 
\begin{equation}\label{eq:diff:G:1}
\left|\theta_*^\top \phi(x,\omega)-\tilde{\theta}_{t-1}^\top \phi(x,\omega)\right|\leq\max\bigg\{\frac{1}{\underline{\mu}},1\bigg\}\|\phi(x,\omega)\|_{\Phi_{t-1}^{-1}}\|g_t(\theta_*)-g_t(\tilde{\theta}_{t-1})\|_{\Phi_{t-1}^{-1}}.
\end{equation}
Further, 
\begin{equation}\label{eq:g:diff}
\begin{array}{rl}
\|g_t(\theta_*)-g_t(\hat{\theta}_{t-1})\|_{\Phi_{t-1}^{-1}}&\displaystyle\leq\|g_t(\theta_*)-g_t(\tilde{\theta}_{t-1})\|_{\Phi_{t-1}^{-1}}+\|g_t(\tilde{\theta}_{t-1})-g_t(\hat{\theta}_{t-1})\|_{\Phi_{t-1}^{-1}}\\
&\displaystyle\leq 2\|g_t(\theta_*)-g_t(\tilde{\theta}_{t-1})\|_{\Phi_{t-1}^{-1}},
\end{array}
\end{equation}
where the first inequality of \eqref{eq:g:diff} follows from triangle inequality, and the second inequality of \eqref{eq:g:diff} follows from the fact that $\theta_*\in\Theta$ and the optimality of $\hat{\theta}_{t-1}$ in $\Theta$ by definition. Recall that 
$\displaystyle g_t(\tilde{\theta}_{t-1})-g_t(\theta_*)=\sum_{k=1}^{t-1}\phi(x_k,\omega_{a_kk})\{r_k-G(\theta_*^\top \phi(x_k,\omega_{a_kk}))\}-\eta_0\theta_*$,
so \eqref{eq:diff:G:1} and \eqref{eq:g:diff} imply that 
$$\begin{array}{rl}
&\displaystyle\quad\left|\theta_*^\top \phi(x,\omega)-\hat{\theta}_{t-1}^\top \phi(x,\omega)\right|\\
&\displaystyle\leq\max\bigg\{\frac{2}{\underline{\mu}},2\bigg\}\|\phi(x,\omega)\|_{\Phi_{t-1}^{-1}}\left\|\sum_{k=1}^{t-1}\phi(x_k,\omega_{a_kk})\{r_k-G(\theta_*^\top \phi(x_k,\omega_{a_kk}))\}-\eta_0\theta_*\right\|_{\Phi_{t-1}^{-1}}\\
&\displaystyle\leq_{(i)}\max\bigg\{\frac{2}{\underline{\mu}},2\bigg\}\|\phi(x,\omega)\|_{\Phi_{t-1}^{-1}}\bigg(\left\|\sum_{k=1}^{t-1}\phi(x_k,\omega_{a_kk})\{r_k-G(\theta_*^\top \phi(x_k,\omega_{a_kk}))\}\right\|_{\Phi_{t-1}^{-1}}+\eta_0\left\|\theta_*\right\|_{\Phi_{t-1}^{-1}}\bigg)\\
&\displaystyle\leq_{(ii)}\max\bigg\{\frac{2}{\underline{\mu}},2\bigg\}\|\phi(x,\omega)\|_{\Phi_{t-1}^{-1}}\bigg(\left\|\sum_{k=1}^{t-1}\phi(x_k,\omega_{a_kk})\{r_k-G(\theta_*^\top \phi(x_k,\omega_{a_kk}))\}\right\|_{\Phi_{t-1}^{-1}}+\sqrt{\eta_0}\|\theta_*\|_2\bigg),
\end{array}$$
holds for all $x,\omega$, where inequality (i) above holds from triangle inequality, and inequality (ii) above holds since $\Phi_{t-1}\succeq\eta_0I_d$ so that $\|\theta_*\|_{\Phi_{t-1}^{-1}}\leq\|\theta_*\|_2/\sqrt{\eta_0}$. We now apply Lemma~\ref{lemma:technical:GLM:bandit:1} to bound $\displaystyle\left\|\sum_{k=1}^{t-1}\phi(x_k,\omega_{a_kk})\{r_k-G(\theta_*^\top \phi(x_k,\omega_{a_kk}))\}\right\|_{\Phi_{t-1}^{-1}}$. Set $v_k=\phi(x_k,\omega_{a_kk})$, $\epsilon_k=\zeta_k$, $\mathcal{H}_k=\sigma(v_s,\epsilon_s;s\leq k)$ and $\bar{V}_t=\Phi_{t-1}$. Note that $|\epsilon_k|\leq\gamma_0$, so $\epsilon_k$ is $\gamma_0$-sub-Gaussian, meaning that for any $u\in\mathbb{R}$, $t\geq1$, $\mathbb{E}\left[\exp\left(u\zeta_t\right)|\mathcal{H}_{t-1}\right]\leq\exp\left(\frac{u^2\gamma_0^2}{2}\right)\ \textrm{a.s.}$
Further, by Assumption~\ref{ass:G:phi:regularity}, $\|\phi(x_k,\omega_{a_kk})\|_2\leq\bar{C}_{\phi}$, implying that $\mathrm{det}(\bar{V}_t)\leq\left(\eta_0+(t-1)\bar{C}_{\phi}^2\right)^d$.
So by Lemma~\ref{lemma:technical:GLM:bandit:1}, given any $\delta>0$, with probability at least $1-\delta$, for all $t\geq2$ we have 
\begin{equation}\label{eq:Phi:norm:bound}
\left\|\sum_{k=1}^{t-1}\phi(x_k,\omega_{a_kk})\{r_k-G(\theta_*^\top \phi(x_k,\omega_{a_kk}))\}\right\|_{\Phi_{t-1}^{-1}}\leq\gamma_0\sqrt{d\log(1+t\bar{C}_{\phi}^2/\eta_0)+2\log(1/\delta)}.
\end{equation}
Note that $\gamma_0>1$ according to Assumption~\ref{ass:reward:perturbed}, hence with probability at least $1-\delta$, for all $t\geq2$ and all $x\in\mathcal{X}$ and $\omega\in\Omega$, 
$$\begin{array}{rl}
\displaystyle\left|\theta_*^\top \phi(x,\omega)-\hat{\theta}_{t-1}^\top \phi(x,\omega)\right|\leq\frac{2\|\phi(x,\omega)\|_{\Phi_{t-1}^{-1}}}{\min\{1,\underline{\mu}\}}\bigg(\gamma_0\sqrt{d\log(1+t\bar{C}_{\phi}^2/\eta_0)+2\log(1/\delta)}+\sqrt{\eta_0}\|\theta_*\|_2\bigg).
\end{array}$$
Hence \eqref{eq:theta:bound:x:w} follows from Assumption~\ref{ass:G:phi:regularity} that $\|\theta_*\|_2\leq \bar{\alpha}$. This immediately implies the high-probability bound \eqref{eq:theta:bound:uniform} taken uniformly over all $a\in[A], t\geq2$. 
\proofend

\section{Minimum Eigenvalue of \texorpdfstring{$\Phi_{t-1}$}{Phi}}\label{appendix:minimum:eigenvalue}
In this section, we first verify the examples satisfying Assumption~\ref{ass:projection:anti:concentration} (Lemma~\ref{lemma:example:gaussian:assumption}, Lemma~\ref{lemma:example:student}, Lemma~\ref{lemma:example:uniform}). Then we show that Assumption~\ref{ass:projection:anti:concentration} implies a lower bound on the minimum eigenvalue for $\Phi_{t-1}$ (Proposition~\ref{prop:minimum:eigenvalue}). We first present Lemma~\ref{lemma:truncation:projection:smallball}, which is useful for proving Lemmas~\ref{lemma:example:gaussian:assumption} --~\ref{lemma:example:uniform}.
\begin{lemma}\label{lemma:truncation:projection:smallball}
Let $Z^0_{at}\in\mathbb R^d$ be a possibly unbounded random vector. Suppose 
$\mathbb P\left(|v^\top Z^0_{at}|\le \epsilon
\,\middle|\,\mathcal F_{t-1}\right)\le C_T\epsilon, \forall v\in\mathbb S^{d-1}, \epsilon>0$,
and suppose that $\mathbb P\left(\|Z^0_{at}\|_2\le \bar C_{\phi}\,\middle|\,\mathcal F_{t-1}\right)\ge q_T\ \text{a.s.}$ for some deterministic $q_T\in(0,1]$. Define the bounded vector $Z_{at}$
by the conditional law
$Z_{at}\sim\mathcal L\left(Z^0_{at}\,\middle|\,\|Z^0_{at}\|_2\le \bar C_{\phi},\mathcal{F}_{t-1}\right)$.
If $\frac{C_T}{q_T}\le M_T$, then $Z_{at}$ satisfies Assumption~\ref{ass:projection:anti:concentration}.
Moreover, $\|Z_{at}\|_2\le \bar C_{\phi}$ almost surely. 
\end{lemma}
\proof{Proof of Lemma~\ref{lemma:truncation:projection:smallball}.}
By construction, $\|Z_{at}\|_2\le \bar C_{\phi}\ \ \text{a.s.}$. Fix $v\in\mathbb S^{d-1}$ and $\epsilon>0$. Then
$$\begin{array}{rl}
\mathbb P\left(|v^\top Z_{at}|\le \epsilon\,\middle|\,\mathcal F_{t-1}\right)&\displaystyle=\mathbb P\left(|v^\top Z^0_{at}|\le \epsilon\,\middle|\,\|Z^0_{at}\|\leq\bar{C}_{\phi},\mathcal F_{t-1}\right)=\frac{\mathbb P\left(|v^\top Z^0_{at}|\le \epsilon,\,\|Z^0_{at}\|\leq\bar{C}_{\phi}\,\middle|\,\mathcal F_{t-1}\right)}{\mathbb P\left(\|Z^0_{at}\|\leq\bar{C}_{\phi}\,\middle|\,\mathcal F_{t-1}\right)}\\
&\displaystyle\le\frac{\mathbb P\left(|v^\top Z^0_{at}|\le \epsilon\,\middle|\,\mathcal F_{t-1}\right)}{q_T}\le\frac{C_T}{q_T}\epsilon\le M_T\epsilon.
\end{array}$$
In particular, the above inequality holds for all
\(0<\epsilon\le (2AM_T)^{-1}\), so Assumption~\ref{ass:projection:anti:concentration}
holds.
\proofend
\begin{lemma}[Truncated Gaussian]\label{lemma:example:gaussian:assumption}
Assumption~\ref{ass:projection:anti:concentration} holds under the following conditions:
\begin{itemize}
    \item[(i)] Conditional on $\mathcal{F}_{t-1}$, $Z_{at}\sim\mathcal{N}(\mu_{at},\Sigma_{at})$ and $\phi(x_t,\omega_{at})\sim\mathcal{L}(Z_{at}\mid \|Z_{at}\|_2\le \bar C_{\phi},\mathcal{F}_{t-1})$.
    \item[(ii)]There exist deterministic sequences $\sigma_T>0$, $\bar{\sigma}_T>0$ and $\bar{\mu}_T\in[0,\bar{C}_{\phi})$ such that uniformly over $a\in[A], t\in[T]$, 
$\lambda_{\min}(\Sigma_{at})\ge \frac{1}{\sigma_T^2},\ \lambda_{\max}(\Sigma_{at})\leq\bar{\sigma}_T^2, \ \|\mu_{at}\|_2\leq\bar{\mu}_T,\ \ \mathrm{a.s.}$
\item[(iii)] Let $F_{\chi_d^2}(\cdot)$ denote the cumulative distribution function of a chi-square random variable with $d$ degrees of freedom. Suppose $\sqrt{2/\pi}\sigma_T\leq M_TF_{\chi_d^2}\left(\frac{(\bar{C}_{\phi}-\bar{\mu}_T)^2}{\bar{\sigma}_T^2}\right)$ where $M_T=\sqrt{\mathrm{polylog}(T)}/A$. 
\end{itemize}
\end{lemma}
\proof{Proof of Lemma~\ref{lemma:example:gaussian:assumption}}
Fix $a\in[A],t\in[T]$ and $v\in\mathbb{S}^{d-1}$. Conditional on $\mathcal{F}_{t-1}$, $v^\top Z_{at}\sim\mathcal{N}(v^\top\mu_{at},v^\top\Sigma_{at}v)$.
Since $v\in\mathbb{S}^{d-1}$, $v^\top\Sigma_{at}v\geq1/\sigma_T^2$. The density of $v^\top Z_{at}$ is bounded by $\displaystyle\frac{1}{\sqrt{2\pi}\sqrt{v^\top\Sigma_{at}v}}\leq\frac{\sigma_T}{\sqrt{2\pi}}$. Let $f_{atv}$ be the density of $v^\top Z_{at}$. So for every $\epsilon>0$, 
$$\mathbb{P}(|v^\top Z_{at}|\leq\epsilon\mid\mathcal{F}_{t-1})=\int_{-\epsilon}^{\epsilon}f_{atv}(u)du\leq\frac{2\epsilon\sigma_T}{\sqrt{2\pi}}=\sqrt{2/\pi}\sigma_T\epsilon.$$
Next, note that $Z_{at}=\mu_{at}+\Sigma_{at}^{1/2}Z$ where $Z\sim\mathcal{N}(0,I_d)$. Note that 
$$\|Z_{at}\|_2\leq_{(1)}\|\mu_{at}\|_2+\|\Sigma_{at}^{1/2}Z\|_2\leq_{(2)}\|\mu_{at}\|_2+\sqrt{\lambda_{\max}(\Sigma_{at})}\|Z\|_2\leq_{(3)}\bar{\mu}_T+\bar{\sigma}_T\|Z\|_2,$$
where (1) follows from triangle inequality, (2) follows because $\|\Sigma_{at}^{1/2}Z\|_2\leq\sqrt{\lambda_{\max}(\Sigma_{at})}\|Z\|_2$, (3) follows from condition (ii) in the lemma. Therefore, the event $\{\|Z\|_2\leq\frac{\bar{C}_{\phi}-\bar{\mu}_T}{\bar{\sigma}_T}\}$ implies 
$\|Z_{at}\|_2\leq\bar{\mu}_T+\bar{\sigma}_T\frac{\bar{C}_{\phi}-\bar{\mu}_T}{\bar{\sigma}_T}=\bar{C}_{\phi}$. Hence 
$\mathbb{P}(\|Z_{at}\|_2\le \bar C_{\phi}\mid\mathcal{F}_{t-1})\geq\mathbb{P}\left(\|Z\|_2\leq\frac{\bar{C}_{\phi}-\bar{\mu}_T}{\bar{\sigma}_T}\,\mid\,\mathcal{F}_{t-1}\right)=F_{\chi_d^2}\left(\frac{(\bar{C}_{\phi}-\bar{\mu}_T)^2}{\bar{\sigma}_T^2}\right)=q_T$.
Applying Lemma~\ref{lemma:truncation:projection:smallball} with $C_T=\sqrt{2/\pi}\sigma_T$ and $q_T=F_{\chi_d^2}\left(\frac{(\bar{C}_{\phi}-\bar{\mu}_T)^2}{\bar{\sigma}_T^2}\right)$ gives 
$\mathbb{P}(|v^\top\phi(x_t,\omega_{at})|\leq\epsilon\mid\mathcal{F}_{t-1})\leq\frac{\sqrt{2/\pi}\sigma_T}{q_T}\epsilon\leq M_T\epsilon$. So Assumption~\ref{ass:projection:anti:concentration} is satisfied. 
\proofend

\begin{lemma}[Truncated Student's t distribution]\label{lemma:example:student}
Assumption~\ref{ass:projection:anti:concentration} holds under conditions below:
\begin{itemize}
\item[(i)]$\phi(x_t,\omega_{at})\sim\mathcal{L}(Z_{at}\mid \|Z_{at}\|_2\le \bar C_{\phi},\mathcal{F}_{t-1})$, where $Z_{at}\sim t_{\xi}(\mu_{at},\Sigma_{at})$ conditioning on $\mathcal{F}_{t-1}$, and $t_{\xi}(\mu_{at},\Sigma_{at})$ is a $d$-dimensional Student's $t$ distribution with $\xi>0$ degrees of freedom, location $\mu_{at}$, scale matrix $\Sigma_{at}$.
\item[(ii)] There exists deterministic sequences $\sigma_T>0$, $\bar{\sigma}_T>0$ and $\bar{\mu}_T\in[0,\bar{C}_{\phi})$ such that uniformly over all $a\in[A]$, $t\in[T]$, $\lambda_{\min}(\Sigma_{at})\geq1/\sigma_T^2$, $\lambda_{\max}(\Sigma_{at})\leq\bar{\sigma}_T^2$, $\|\mu_{at}\|_2\leq\bar{\mu}_T$\ $\mathrm{a.s.}$. 
\item[(iii)]$2c_{\xi}\sigma_T/q_T\leq M_T=\sqrt{\mathrm{polylog}(T)}/A$, where $c_{\xi}:=\frac{\Gamma((\xi+1)/2)}{\sqrt{\xi\pi}\Gamma(\xi/2)}$, $q_T=F_{d,\xi}((\bar{C}_{\phi}-\bar{\mu}_T)^2/(d\bar{\sigma}_T^2))$, and $F_{d,\xi}(\cdot)$ denotes the cumulative distribution of the $F$-distribution with $d$ and $\xi$ degrees of freedom. 
\end{itemize}
\end{lemma}
\proof{Proof of Lemma~\ref{lemma:example:student}}
Fix $a\in[A], t\in[T], v\in\mathbb{S}^{d-1}$. The multivariate Student's $t$ random vector $Z_{at}$ can be written as $Z_{at}=\mu_{at}+\frac{\Sigma_{at}^{1/2}Z}{\sqrt{S/\xi}}$,
where $Z\sim\mathcal{N}(0,I_d)$, $S\sim\chi_{\xi}^2$, where $\chi_{\xi}^2$ is the chi-square distribution with $\xi$ degrees of freedom, $Z$ and $S$ are independent. So 
$v^\top Z_{at}=v^\top\mu_{at}+\frac{v^\top \Sigma_{at}^{1/2}Z}{\sqrt{S/\xi}}$.
Since $v^\top \Sigma_{at}^{1/2}Z\sim\mathcal{N}(0,v^\top\Sigma_{at}v)$, it holds that $v^\top Z_{at}\sim t_{\xi}(v^\top\mu_{at},v^\top\Sigma_{at}v)$, where by the eigenvalue lower bound, we have $v^\top\Sigma_{at}v\geq1/\sigma_T^2$. The density $f_{atv}$ of the univariate Student's $t$ random variable $v^\top Z_{at}$ is $f_{atv}(y)=\frac{c_{\xi}}{\sqrt{v^\top\Sigma_{at}v}}\left(1+\frac{(y-v^\top\mu_{at})^2}{\xi [v^\top\Sigma_{at}v]}\right)^{-(\xi+1)/2}$.
Hence $\sup_yf_{atv}(y)\leq c_{\xi}/\sqrt{v^\top\Sigma_{at}v}\leq c_{\xi}\sigma_T$. So for every $\epsilon>0$, $\mathbb{P}\left(|v^\top Z_{at}|\leq\epsilon\mid\mathcal{F}_{t-1}\right)=\int_{-\epsilon}^{\epsilon}f_{atv}(u)du\leq2c_{\xi}\sigma_T\epsilon$.
Next, we lower bound the probability of the truncation event $\{\|Z_{at}\|_2\leq\bar{C}_{\phi}\}$. Note that 
$\|Z_{at}\|_2\leq\|\mu_{at}\|_2+\|\Sigma_{at}^{1/2}Z\|_2/\sqrt{S/\xi}\leq\bar{\mu}_T+\bar{\sigma}_T\|Z\|_2/\sqrt{\frac{S}{\xi}}$.
Therefore, $\|Z\|_2/\sqrt{\frac{S}{\xi}}\leq(\bar{C}_{\phi}-\bar{\mu}_T)/\bar{\sigma}_T$ implies $\|Z_{at}\|_2\leq\bar{\mu}_T+\bar{\sigma}_{T}\frac{\bar{C}_{\phi}-\bar{\mu}_T}{\bar{\sigma}_T}=\bar{C}_{\phi}$. So 
$$\begin{array}{rl}
\mathbb{P}(\|Z_{at}\|_2\leq\bar{C}_{\phi}\mid\mathcal{F}_{t-1})&\displaystyle\geq\mathbb{P}\left(\frac{\|Z\|_2}{\sqrt{S/\xi}}\leq\frac{\bar{C}_{\phi}-\bar{\mu}_T}{\bar{\sigma}_T}\mid\mathcal{F}_{t-1}\right)\\
&\displaystyle=\mathbb{P}\left(\frac{\|Z\|_2^2/d}{S/\xi}\leq\frac{(\bar{C}_{\phi}-\bar{\mu}_T)^2}{d\bar{\sigma}_T^2}\mid\mathcal{F}_{t-1}\right)=F_{d,\xi}\left(\frac{(\bar{C}_{\phi}-\bar{\mu}_T)^2}{d\bar{\sigma}_T^2}\right)=q_T.
\end{array}$$
Hence applying Lemma~\ref{lemma:truncation:projection:smallball}, we have $\mathbb{P}(|v^\top\phi(x_t,\omega_{at})|\leq\epsilon\mid\mathcal{F}_{t-1})\leq M_T\epsilon$, so Assumption~\ref{ass:projection:anti:concentration}
holds.
\proofend

\begin{lemma}[Uniform distribution on a hyperrectangle]\label{lemma:example:uniform}
Assumption~\ref{ass:projection:anti:concentration} holds under (i)--(ii):
\begin{itemize}
    \item[(i)]Conditional on $\mathcal{F}_{t-1}$, $\phi(x_t,\omega_{at})=\bar{\phi}_{at}+U_{at}$, where $\bar{\phi}_{at}$ is a fixed constant vector, define each coordinate of $U_{at}$ as $U_{at,j}\sim\mathrm{Unif}[-R_{at,j},R_{at,j}]$ for $j\in[d]$ and $U_{at,j}$ are independent across $j\in[d]$. 
    \item[(ii)] $R_{at,j}\geq\sqrt{d}/M_T$ where $M_T=\sqrt{\mathrm{polylog}(T)}/A$ and  
$\|\bar{\phi}_{at}\|_2+\left(\sum_{j=1}^dR_{at,j}^2\right)^{1/2}\leq\bar{C}_{\phi},\ \ \mathrm{a.s.}$ 
\end{itemize}
\end{lemma}
\proof{Proof of Lemma~\ref{lemma:example:uniform}}
Fix $a\in[A],t\in[T],v\in\mathbb{S}^{d-1}$. Since $\|v\|_2=1$, there exists a coordinate $j^*\in[d]$, s.t. $|v_{j^*}|\geq1/\sqrt{d}$. Note that $v^\top\phi(x_t,\omega_{at})=v^\top\bar{\phi}_{at}+v^\top U_{at}=v^\top\bar{\phi}_{at}+\sum_{j=1}^dv_jU_{at,j}$.
Conditional on all $\{U_{at,j}: j\neq j^*\}$, $v^\top\phi(x_t,\omega_{at})=C+v_{j^*}U_{at,j^*}$ where $C$ is a fixed constant. 
Conditional on $\mathcal{F}_{t-1}$, $U_{at,j}$ are independent uniform random on $[-R_{at,j},R_{at,j}]$.
Since $U_{at,j^*}\sim\mathrm{Unif}[-R_{at,j^*},R_{at,j^*}]$, $C+v_{j^*}U_{at,j^*}$ is uniform on an interval of length $2R_{at,j^*}|v_{j^*}|$. Therefore, for any interval $I\subset\mathbb{R}$ of length $|I|$, we have $\mathbb{P}(C+v_{j^*}U_{at,j^*}\in I\mid\{U_{at,j}:j\neq j^*\},\mathcal{F}_{t-1})\leq\frac{|I|}{2R_{at,j^*}|v_{j^*}|}$.
Take $I=[-\epsilon,\epsilon]$, then 
$\mathbb{P}\left(|v^\top\phi(x_t,\omega_{at})|\leq\epsilon\mid\mathcal{F}_{t-1},\{U_{at,j}: j\neq j^*\}\right)\leq\frac{2\epsilon}{2R_{at,j^*}|v_{j^*}|}\leq\frac{\sqrt{d}}{R_{at,j^*}}\epsilon$.
Taking expectation over the conditioned coordinates of $U_{at}$ in the above inequality gives $\mathbb{P}(|v^\top\phi(x_t,\omega_{at})|\leq\epsilon\mid\mathcal{F}_{t-1})\leq M_T\epsilon$
with $M_T=\sqrt{\mathrm{polylog}(T)}/A$.
\proofend

\begin{lemma}[Theorem 3.1 of \citet{tropp2011user}]\label{lemma:tropp:2011:eigenvalue}
Let $\mathcal{H}_1\subset\mathcal{H}_2\subset\cdots$ be a filtration and consider a finite sequence $\{X_k\}$ of positive semi-definite matrices with dimension $d$ adapted to this filtration. Suppose that $\lambda_{\max}(X_k)\leq R$ almost surely. Define the series $Y\equiv\sum_kX_k$ and $W\equiv\sum_k\mathbb{E}[X_k|\mathcal{H}_{k-1}]$. Then for all $\mu\geq0$, $\gamma\in[0,1)$,
$\mathbb{P}\left(\lambda_{\min}(Y)\leq(1-\gamma)\mu, \lambda_{\min}(W)\geq\mu\right)\leq d\left(\frac{e^{-\gamma}}{(1-\gamma)^{1-\gamma}}\right)^{\mu/R}$. 
\end{lemma}
\medskip
\begin{proposition}[Minimum eigenvalue]\label{prop:minimum:eigenvalue}
Suppose Assumption~\ref{ass:projection:anti:concentration} holds. Then
\begin{equation}\label{eq:eigenvalue:expected}
\lambda_{\min}\!\left(\mathbb E\!\left[\phi(x_t,\omega_{a_tt})\phi(x_t,\omega_{a_tt})^\top\,\middle|\,\mathcal F_{t-1}\right]\right)\ge\frac{1/8}{\mathrm{polylog}(T)}.
\end{equation}
Further, for any $\delta\in(0,1)$, with probability at least $1-\delta$, \begin{equation}\label{eq:lower:eigenvalue:Phi:t-1}
\lambda_{\min}(\Phi_{t-1})\geq(t-1)/[16\mathrm{polylog}(T)]
\end{equation}
holds uniformly for all $t\geq100\bar{C}_{\phi}^2\mathrm{polylog}(T)\log(Td/\delta)$.
\end{proposition}
\proof{Proof of Proposition~\ref{prop:minimum:eigenvalue}}
Fix $v\in\mathbb S^{d-1}$. Since $a_t\in[A]$, for any $\epsilon>0$,
$\{|v^\top \phi(x_t,\omega_{a_tt})|\le \epsilon\}\subseteq\bigcup_{a=1}^A\{|v^\top\phi(x_t,\omega_{at})|\le \epsilon\}$. Therefore, by applying union bound and
Assumption~\ref{ass:projection:anti:concentration},
$\displaystyle\mathbb P\!\left(
    |v^\top\phi(x_t,\omega_{a_t t})|\le \epsilon
    \,\middle|\,
    \mathcal{F}_{t-1}
\right)\leq\sum_{a=1}^A
\mathbb P\!\left(
    |v^\top\phi(x_t,\omega_{at})|\le \epsilon
    \,\middle|\,
    \mathcal{F}_{t-1}
\right) \le AM_T\epsilon$.
Take $\epsilon=(2AM_T)^{-1}$. Then the above inequality implies 
$\mathbb{P}\!\left(|v^\top\phi(x_t,\omega_{a_t t})|>(2AM_T)^{-1}\,\middle|\,\mathcal{F}_{t-1}\right)\geq\frac{1}{2}$. Hence
$$\begin{array}{rl}
\displaystyle v^\top\mathbb E[\phi(x_t,\omega_{a_t t})\phi(x_t,\omega_{a_t t})^\top\mid\mathcal{F}_{t-1}]v&\displaystyle=\mathbb{E}[(v^\top\phi(x_t,\omega_{a_t t}))^2\mid\mathcal{F}_{t-1}]\\
&\displaystyle\geq(2AM_T)^{-2}\mathbb{P}\!\left(|v^\top\phi(x_t,\omega_{a_t t})|>(2AM_T)^{-1}\,\middle|\,\mathcal{F}_{t-1}\right)\\
&\displaystyle\geq\frac{1/8}{(AM_T)^2}=\frac{1/8}{\mathrm{polylog}(T)}.
\end{array}$$
Since this holds for any $v\in\mathbb{S}^{d-1}$, \eqref{eq:eigenvalue:expected} follows. Setting $\gamma=1/2$, $R=\bar{C}_{\phi}^2$ in Lemma~\ref{lemma:tropp:2011:eigenvalue}, with probability at least $1-\exp\{\log(d)-0.01(t-1)/[\bar{C}_{\phi}^2\mathrm{polylog}(T)]\}$, 
$\lambda_{\min}(\Phi_{t-1})\geq(t-1)/[16\mathrm{polylog}(T)]$. Particularly, for any $\delta\in(0,1)$, with probability at least $1-\delta$, $\lambda_{\min}(\Phi_{t-1})\geq(t-1)/[16\mathrm{polylog}(T)]$ holds uniformly for all $t\geq100\bar{C}_{\phi}^2\mathrm{polylog}(T)\log(Td/\delta)$.
\proofend

\section{Index Estimation under Unknown Reward Function}\label{appendix:index:unknown:theta}
The same argument for Proposition~\ref{prop:rho:a:confidence:theta:known} extends to Proposition~\ref{prop:confidence:rho} for the high probability bound of the index estimation error when the reward parameter $\theta_*$ is unknown and learned online. The main difference is that the primitive loss is evaluated at the plug-in estimate $\hat\theta_{t-1}$ rather than at the true parameter $\theta_*$. This introduces an additional error term in the optimality inequality:
$$\mbox{population curvature}\;\;\leq\;\;\mbox{empirical fluctuation}\;\;+\;\;\mbox{plug-in error from }\hat{\theta}_{t-1}.$$
Consequently, Proposition~\ref{prop:confidence:rho} yields the same type of confidence radius for $\rho_a^\top\psi(x_t)$, with an additional additive term that accounts for uncertainty in the reward parameter.

Recall from Section~\ref{sec:estimating:indices} that 
$\hat{\rho}_{at}\in\argmin_{\rho\in\mathcal{B}}\sum_{s\in\mathcal{S}_{at}}\hat{\ell}_{as,t}(\rho)$, where $\hat{\ell}_{as,t}$ is the loss defined as \eqref{eq:plugin:loss} when plugging in $\hat{\theta}_{t-1}$ for the unknown $\theta_*$. Lemma~\ref{lemma:J:at:bound} below characterizes how the estimation error of $\theta_*$ propagates into the estimation error of $\rho_a^\top\psi(x_t)$. 

\medskip
\proof{Proof of Lemma~\ref{lemma:J:at:bound}}
Recall from \eqref{eq:J:at:main} that 
$$J_{at}(\rho):=\sum_{s\in\mathcal{S}_{at}}\left\{[\hat{\ell}_{as,t}(\rho)-\ell_{as}(\rho)]-[\hat{\ell}_{as,t}(\rho_a)-\ell_{as}(\rho_a)]\right\}.$$
Note that \eqref{eq:known-theta--loss} and \eqref{eq:plugin:loss} imply that 
$$\begin{array}{rl}
&\displaystyle\quad[\hat{\ell}_{as,t}(\rho_a)-\ell_{as}(\rho_a)]-[\hat{\ell}_{as,t}(\hat{\rho}_{at})-\ell_{as}(\hat{\rho}_{at})]\\
&\displaystyle=\int_0^{\rho_a^\top\psi(x_{s})}[(G(\theta_*^\top\phi(x_{s},\omega_{as}))-\Lambda(u))^{+}-(G(\hat{\theta}_{t-1}^\top\phi(x_{s},\omega_{as}))-\Lambda(u))^{+}]du\\
&\displaystyle\quad-\int_{0}^{\hat{\rho}_{at}^\top\psi(x_{s})}[(G(\theta_*^\top\phi(x_{s},\omega_{as}))-\Lambda(u))^{+}-(G(\hat{\theta}_{t-1}^\top\phi(x_{s},\omega_{as}))-\Lambda(u))^{+}]du\\
\\
&\displaystyle=\int_{\rho_a^\top\psi(x_{s})}^{\hat{\rho}_{at}^\top\psi(x_{s})}[(G(\hat{\theta}_{t-1}^\top\phi(x_{s},\omega_{as}))-\Lambda(u))^{+}-(G(\theta_*^\top\phi(x_{s},\omega_{as}))-\Lambda(u))^{+}]du\\
&\displaystyle\leq_{(i)}\left|G(\hat{\theta}_{t-1}^\top\phi(x_{s},\omega_{as}))-G(\theta_*^\top\phi(x_{s},\omega_{as}))\right|\cdot|\psi(x_{s})^\top(\hat{\rho}_{at}-\rho_a)|,
\end{array}$$
where (i) uses the 1-Lipschitz property of $z\mapsto z^{+}$.
Thus 
$$\begin{array}{rl}
-J_{at}(\hat{\rho}_{at})&\displaystyle\leq_{(i)}\sqrt{\sum_{s\in\mathcal{S}_{at}}\{\psi(x_{s})^\top(\hat{\rho}_{at}-\rho_a)\}^2}\sqrt{\sum_{s\in\mathcal{S}_{at}}\left[G(\hat{\theta}_{t-1}^\top\phi(x_{s},\omega_{as}))-G(\theta_*^\top\phi(x_{s},\omega_{as}))\right]^2}\\
&\displaystyle\leq_{(ii)}\frac{c_0}{8}\sum_{s\in\mathcal{S}_{at}}\{\psi(x_{s})^\top(\hat{\rho}_{at}-\rho_a)\}^2+\frac{2}{c_0}\sum_{s\in\mathcal{S}_{at}}\left[G(\hat{\theta}_{t-1}^\top\phi(x_{s},\omega_{as}))-G(\theta_*^\top\phi(x_{s},\omega_{as}))\right]^2\\
&\displaystyle\leq_{(iii)}\frac{c_0}{8}\sum_{s\in\mathcal{S}_{at}}\{\psi(x_{s})^\top(\hat{\rho}_{at}-\rho_a)\}^2+\frac{2}{c_0}\sum_{s\in\mathcal{S}_{at}}L^2\{(\hat{\theta}_{t-1}-\theta_*)^\top\phi(x_{s},\omega_{as})\}^2
\end{array}$$
where (i) above holds from Cauchy-Schwarz inequality, (ii) holds from applying $2ab\leq\epsilon a^2+\epsilon^{-1}b^2$
to $\displaystyle\epsilon=\frac{c_0}{4}$, $\displaystyle a=\sqrt{\sum_{s\in\mathcal{S}_{at}}\{\psi(x_{s})^\top(\hat{\rho}_{at}-\rho_a)\}^2}$, $\displaystyle b=\sqrt{\sum_{s\in\mathcal{S}_{at}}\left[G(\hat{\theta}_{t-1}^\top\phi(x_{s},\omega_{as}))-G(\theta_*^\top\phi(x_{s},\omega_{as}))\right]^2}$, (iii) holds from the Lipschitz property of $G$ according to Assumption~\ref{ass:G:phi:regularity}. Hence the result follows.
\proofend
\begin{proposition}[Restatement of Proposition~\ref{prop:confidence:rho:main:context}]\label{prop:confidence:rho}
Suppose Assumptions~\ref{ass:reg_res_index}, \ref{ass:losscurvature}, \ref{ass:G:phi:regularity}, \ref{ass:reward:perturbed}, \ref{ass:projection:anti:concentration} hold. Let $V_{at}(\eta_1)$ be defined as \eqref{eq:known-theta-V}. Given any $\delta>0$, with probability $1-2\delta/3$, uniformly over all $a\in[A]$ and $t\in[T]$, 
$$\begin{array}{rl}
|(\hat{\rho}_{at}-\rho_a)^\top\psi(x_t)|&\displaystyle\leq\|\psi(x_t)\|_{V_{at}(\eta_1)^{-1}}\Bigg[\sqrt{\frac{4\Gamma_{at}(\delta)}{\kappa\mu_1}\left(\frac{144}{\kappa\mu_1}+2\min\{2\overline{\iota},\bar{C}_{\psi}d_{\mathcal{B}}\}+6\right)+\frac{20}{\kappa\mu_1}+\frac{1}{2}+\eta_1d_{\mathcal{B}}^2}\\
&\quad\quad\quad\quad\quad\quad\quad\quad\displaystyle+\frac{2\sqrt{2}L}{\kappa\mu_1}\sqrt{\bar{C}_{at}(\delta)\sum_{s\in\mathcal{S}_{at}}\|\phi(x_{s},\omega_{as})\|_{\Phi_{t-1}^{-1}}^2}\Bigg],
\end{array}$$
where $\Gamma_{at}(\delta)\!:=m\log(1+2d_{\mathcal{B}}\bar{C}_{\psi}T)+\log\left(\left\lceil\log_2\big(1+n_{at}\min\{9\overline{\iota}^2,d_{\mathcal{B}}^2\bar{C}_{\psi}^2\}\right)\right\rceil+1\big)+\log(6AT/\delta)$,
$\bar{C}_{at}(\delta):=\frac{4}{\min\{1,\underline{\mu}\}^2}\bigg(\gamma_0\sqrt{d\log(1+t\bar{C}_{\phi}^2/\eta_0)+2\log(3/\delta)}+\sqrt{\eta_0}\bar{\alpha}\bigg)^2$.
\end{proposition}
\proof{Proof of Proposition~\ref{prop:confidence:rho}}
For any $s\in[T]$, define $\mathcal{G}_{s-1}:=\sigma(\mathcal{F}_{s-1},x_{s})$. Note that 
\begin{equation}\label{eq:ineq:key:rho}
\begin{array}{rl}
0&\geq_{(1)}\displaystyle\sum_{s\in\mathcal{S}_{at}}\hat{\ell}_{as,t}(\hat{\rho}_{at})-\sum_{s\in\mathcal{S}_{at}}\hat{\ell}_{as,t}(\rho_a)=_{(2)}\sum_{s\in\mathcal{S}_{at}}\mathbb{E}\left[\ell_{as}(\hat{\rho}_{at})-\ell_{as}(\rho_a)\mid\mathcal{G}_{s-1}\right]+H_{at}(\hat{\rho}_{at})+J_{at}(\hat{\rho}_{at})\\
&\displaystyle\geq_{(3)}\frac{\kappa\mu_1}{2}\sum_{s\in\mathcal{S}_{at}}[\psi(x_{s})^\top(\hat{\rho}_{at}-\rho_a)]^2-1+H_{at}(\hat{\rho}_{at})+J_{at}(\hat{\rho}_{at}),
\end{array}
\end{equation}
where (1) holds because $\displaystyle\hat{\rho}_{at}=\argmin_{\rho\in\mathcal{B}}\frac{1}{n_{at}}\sum_{s\in\mathcal{S}_{at}}\hat{\ell}_{as,t}(\rho)$, (2) holds with  
$$H_{at}(\rho):=\sum_{s\in\mathcal{S}_{at}}\left\{\ell_{as}(\rho)-\ell_{as}(\rho_a)-\mathbb{E}[\ell_{as}(\rho)-\ell_{as}(\rho_a)\mid\mathcal{G}_{s-1}]\right\},$$
$$J_{at}(\rho):=\sum_{s\in\mathcal{S}_{at}}\left\{[\hat{\ell}_{as,t}(\rho)-\ell_{as}(\rho)]-[\hat{\ell}_{as,t}(\rho_a)-\ell_{as}(\rho_a)]\right\},$$
and (3) follows from Lemma~\ref{lemma:curvature:hat:rho:at}. So \eqref{eq:ineq:key:rho} implies that 
\begin{equation}\label{eq:rho:ineq:upper}
\frac{\kappa\mu_1}{2}\sum_{s\in\mathcal{S}_{at}}[\psi(x_{s})^\top(\hat{\rho}_{at}-\rho_a)]^2\leq-(H_{at}(\hat{\rho}_{at})+J_{at}(\hat{\rho}_{at}))+1.
\end{equation}
Lemma~\ref{lemma:H:empirical:process} implies that with probability at least $1-\delta/3$, uniformly over all $t\in[T]$ and $a\in[A]$ we have
\begin{equation}\label{eq:H:at:rho:at:hat:bound}
-H_{at}(\hat{\rho}_{at})\leq\frac{\kappa\mu_1}{8}\sum_{s\in\mathcal{S}_{at}}\{(\hat{\rho}_{at}-\rho_a)^\top\psi(x_{s})\}^2+\left(\frac{144}{\kappa\mu_1}+2\min\{2\overline{\iota},\bar{C}_{\psi}d_{\mathcal{B}}\}+6\right)\Gamma_{at}(\delta)+4+\frac{\kappa\mu_1}{8},
\end{equation}
where $\Gamma_{at}(\delta)\!:=m\log(1+2d_{\mathcal{B}}\bar{C}_{\psi}T)+\log(Q_{at}+1)+\log\left(\!
\frac{6AT}{\delta}\!\right).$
Lemma~\ref{lemma:J:at:bound} then implies that 
$\displaystyle-J_{at}(\hat{\rho}_{at})\leq\frac{\kappa\mu_1}{8}\sum_{s\in\mathcal{S}_{at}}\{\psi(x_{s})^\top(\hat{\rho}_{at}-\rho_a)\}^2+\frac{2}{\kappa\mu_1}\sum_{s\in\mathcal{S}_{at}}L^2\{(\hat{\theta}_{t-1}-\theta_*)^\top\phi(x_{s},\omega_{as})\}^2.$
Hence combining this inequality and \eqref{eq:H:at:rho:at:hat:bound}, with probability at least $1-\delta/3$, uniformly over all $t\in[T], a\in[A]$,
\begin{equation}\label{eq:Deta:at:bound:high:prob}
\begin{array}{rl}
&\quad\displaystyle-(H_{at}(\hat{\rho}_{at})+J_{at}(\hat{\rho}_{at}))\\
&\displaystyle\leq\frac{\kappa\mu_1}{4}\sum_{s\in\mathcal{S}_{at}}\{(\hat{\rho}_{at}-\rho_a)^\top\psi(x_{s})\}^2+\left(\frac{144}{\kappa\mu_1}+2\min\{2\overline{\iota},\bar{C}_{\psi}d_{\mathcal{B}}\}+6\right)\Gamma_{at}(\delta)+4+\frac{\kappa\mu_1}{8}\\
&\quad\displaystyle+\frac{2L^2}{\kappa\mu_1}\sum_{s\in\mathcal{S}_{at}}\{(\hat{\theta}_{t-1}-\theta_*)^\top\phi(x_{s},\omega_{as})\}^2
\end{array}
\end{equation}
Thus \eqref{eq:rho:ineq:upper} further implies that with probability at least $1-\delta/3$, uniformly over all $t\in[T]$ we have
\begin{equation}\label{eq:rho:diff:psi:x}
\begin{array}{rl}
\displaystyle\frac{\kappa\mu_1}{4}\sum_{s\in\mathcal{S}_{at}}\{(\hat{\rho}_{at}-\rho_a)^\top\psi(x_{s})\}^2&\displaystyle\leq\left(\frac{144}{\kappa\mu_1}+2\min\{2\overline{\iota},\bar{C}_{\psi}d_{\mathcal{B}}\}+6\right)\Gamma_{at}(\delta)+5+\frac{\kappa\mu_1}{8}\\
&\quad\displaystyle+\frac{2L^2}{\kappa\mu_1}\sum_{s\in\mathcal{S}_{at}}\{(\hat{\theta}_{t-1}-\theta_*)^\top\phi(x_{s},\omega_{as})\}^2,
\end{array}
\end{equation}
combining with the fact that $\eta_1\|\hat{\rho}_{at}-\rho_a\|^2\leq\eta_1d_{\mathcal{B}}^2$, and by recalling that $V_{at}(\eta_1)=\eta_1I_m+\sum_{s\in\mathcal{S}_{at}}\psi(x_{s})\psi(x_{s})^\top$,
the above inequality then implies that with probability at least $1-\delta/3$, uniformly over all $t\in[T]$ and $a\in[A]$, we have
\begin{equation}\label{eq:rho:diff:V:square}
\begin{array}{rl}
\displaystyle\|\hat{\rho}_{at}-\rho_a\|_{V_{at}(\eta_1)}^2&\displaystyle\leq\frac{4}{\kappa\mu_1}\left(\frac{144}{\kappa\mu_1}+2\min\{2\overline{\iota},\bar{C}_{\psi}d_{\mathcal{B}}\}+6\right)\Gamma_{at}(\delta)+\frac{20}{\kappa\mu_1}+\frac{1}{2}+\eta_1d_{\mathcal{B}}^2\\
&\quad\displaystyle+\frac{8L^2}{(\kappa\mu_1)^2}\sum_{s\in\mathcal{S}_{at}}\!\{(\hat{\theta}_{t-1}-\theta_*)^\top\phi(x_{s},\omega_{as})\}^2.
\end{array}
\end{equation}
Lemma~\ref{lemma:G:diff:hat:theta} implies that with probability at least $1-\delta/3$, uniformly over all $t\in[T]$ and $a\in[A]$,
\begin{equation}\label{eq:theta:diff:phi}
\frac{8L^2}{(\kappa\mu_1)^2}\sum_{s\in\mathcal{S}_{at}}\{(\hat{\theta}_{t-1}-\theta_*)^\top\phi(x_{s},\omega_{as})\}^2\leq\frac{8L^2}{(\kappa\mu_1)^2}\bar{C}_{at}(\delta)\sum_{s\in\mathcal{S}_{at}}\|\phi(x_{s},\omega_{as})\|_{\Phi_{t-1}^{-1}}^2,
\end{equation}
where 
$\bar{C}_{at}(\delta):=\frac{4}{\min\{1,\underline{\mu}\}^2}\bigg(\gamma_0\sqrt{d\log(1+t\bar{C}_{\phi}^2/\eta_0)+2\log(3/\delta)}+\sqrt{\eta_0}\bar{\alpha}\bigg)^2.$
Thus with probability $1-2\delta/3$ we have 
\begin{equation}\label{eq:rho:diff:V:at}
\begin{array}{rl}
\displaystyle\|\hat{\rho}_{at}-\rho_a\|_{V_{at}(\eta_1)}&\displaystyle\leq\sqrt{\frac{4}{\kappa\mu_1}\!\left(\frac{144}{\kappa\mu_1}+2\min\{2\overline{\iota},\bar{C}_{\psi}d_{\mathcal{B}}\}+6\right)\Gamma_{at}(\delta)+\frac{20}{\kappa\mu_1}+\frac{1}{2}+\eta_1d_{\mathcal{B}}^2}\\
&\quad\displaystyle+\frac{2\sqrt{2}L}{\kappa\mu_1}\sqrt{\bar{C}_{at}(\delta)\sum_{s\in\mathcal{S}_{at}}\|\phi(x_{s},\omega_{as})\|_{\Phi_{t-1}^{-1}}^2}.
\end{array}
\end{equation}
By Cauchy-Schwarz inequality, 
$|(\hat{\rho}_{at}-\rho_a)^\top\psi(x_t)|\leq\|\psi(x_t)\|_{V_{at}(\eta_1)^{-1}}\|\hat{\rho}_{at}-\rho_a\|_{V_{at}(\eta_1)}$, so combining \eqref{eq:rho:diff:V:at}, the result follows.
\proofend

\section{Regret under Unknown Reward Function}\label{appendix:regret:unknown:theta}
\proof{Proof of Theorem~\ref{thm:expected:total:regret}}
Let $\mathcal{E}$ denote the event that conditions (i)-(ii) hold in the following:
\begin{itemize}
    \item[(i)] $|(\theta_*-\hat{\theta}_{t-1})^\top\phi(x_t,\omega_{at})|\leq\frac{2\|\phi(x_t,\omega_{at})\|_{\Phi_{t-1}^{-1}}}{\min\{\underline{\mu},1\}}\bigg(\gamma_0\sqrt{d\log(1+t\bar{C}_{\phi}^2/\eta_0)+2\log(3/\delta)}+\sqrt{\eta_0}\bar{\alpha}\bigg)$ uniformly over all $a\in[A]$, $t\in[T]$;
    \item[(ii)] Uniformly over all $a\in[A]$, $t\in[T]$,
    $$\begin{array}{rl}
    \displaystyle|(\hat{\rho}_{at}-\rho_a)^\top\psi(x_t)|&\displaystyle\leq \bigg[\sqrt{\frac{4}{\kappa\mu_1}\left(\frac{144}{\kappa\mu_1}+2\min\{2\overline{\iota},\bar{C}_{\psi}d_{\mathcal{B}}\}+6\right)\Gamma_{at}(\delta)+\frac{20}{\kappa\mu_1}+\frac{1}{2}+\eta_1d_{\mathcal{B}}^2}\\
    &\quad\displaystyle+\frac{2\sqrt{2}L}{\kappa\mu_1}\sqrt{C_{at}(\delta)\sum_{s\in\mathcal{S}_{at}}\|\phi(x_{s},\omega_{as})\|_{\Phi_{t-1}^{-1}}^2}\bigg]\|\psi(x_t)\|_{V_{at}(\eta_1)^{-1}},
    \end{array}$$
    where $\displaystyle V_{at}(\eta_1)=\eta_1I_m+\sum_{s\in\mathcal{S}_{at}}\psi(x_{s})\psi(x_{s})^\top$, $Q_{at}:=\left\lceil\log_2\left(1+n_{at}\min\{9\overline{\iota}^2,d_{\mathcal{B}}^2\bar{C}_{\psi}^2\}\right)\right\rceil$,
    $$\Gamma_{at}(\delta)\!:=m\log(1+2d_{\mathcal{B}}\bar{C}_{\psi}T)+\log(2Q_{at}+2)+\log\left(\!
6AT/\delta\!\right),$$
    $$C_{at}(\delta):=\frac{4}{\min\{\underline{\mu},1\}^2}\bigg(\gamma_0\sqrt{d\log(1+t\bar{C}_{\phi}^2/\eta_0)+2\log(6/\delta)}+\sqrt{\eta_0}\bar{\alpha}\bigg)^2.$$
\end{itemize}
Thus on $\mathcal{E}$,
$\widetilde{\mu}_t(x_t,\omega_{at})\geq\mu^*(x_t,\omega_{at})$ holds uniformly over $a\in[A], t\in[T]$ 
and $\widetilde{\sigma}_{at}\geq\sigma^*_{at}$ holds uniformly over all $a\in[A], t\in[T]$. \eqref{eq:mu:tilde}, \eqref{eq:indices:tilde}, Lemma~\ref{lemma:G:diff:hat:theta} and Proposition~\ref{prop:confidence:rho} imply that $\mathcal{E}$ holds with probability at least $1-\delta$. 
Theorem~\ref{thm:optimistic-regret-decomposition} implies
\begin{equation}\label{eq:exp:bound}
\mathbb{E}\left[\sum_{t=1}^T\Delta_t(\tilde{\pi})\ \big|\ \mathcal{E}\right]\leq\mathbb{E}\left[\sum_{t=1}^T\widetilde{\mu}_t(x_t,\omega_{a_tt})-\mu^*(x_t,\omega_{a_tt})\ \big|\ \mathcal{E}\right]+\mathbb{E}\left[\sum_{t=1}^T\sum_{a\in\mathcal{A}_t}(\widetilde{\sigma}_{at}-\sigma^*_{at})\ \big|\ \mathcal{E}\right].
\end{equation}
On $\mathcal{E}$, using the fact that  
$$\widetilde{\mu}_t(x_t,\omega_{a_tt})-\mu^*(x_t,\omega_{a_tt})\!\!\leq\left|\widetilde{\mu}_t(x_t,\omega_{a_tt})-G(\hat{\theta}_{t-1}^\top\phi(x_t,\omega_{a_tt}))\right|+\left|G(\hat{\theta}_{t-1}^\top\phi(x_t,\omega_{a_tt}))-\mu^*(x_t,\omega_{a_tt})\right|,$$ 
and the Lipschitz property of $G$ by Assumption~\ref{ass:G:phi:regularity}, we have that 
$$\begin{array}{rl}
&\quad\displaystyle\sum_{t=1}^T\widetilde{\mu}_t(x_t,\omega_{a_tt})-\mu^*(x_t,\omega_{a_tt})\\
&\displaystyle\leq\frac{4L}{\min\{\underline{\mu},1\}}\bigg(\gamma_0\sqrt{d\log(1+T\bar{C}_{\phi}^2/\eta_0)+2\log(3/\delta)}+\sqrt{\eta_0}\bar{\alpha}\bigg)\sum_{t=1}^T\|\phi(x_t,\omega_{a_tt})\|_{\Phi_{t-1}^{-1}}.
\end{array}$$
Since $\mathrm{det}(\Phi_t)=\mathrm{det}(\Phi_{t-1})(1+\|\phi(x_t,\omega_{a_tt})\|_{\Phi_{t-1}^{-1}}^2)$, $\Phi_0=\eta_0I_d$, it holds that 
$$\begin{array}{rl}
\displaystyle\log\frac{\mathrm{det}(\Phi_T)}{\mathrm{det}(\Phi_0)}=\!\!\sum_{t=1}^T\log(1+\|\phi(x_t,\omega_{a_tt})\|_{\Phi_{t-1}^{-1}}^2)\geq_{(i)}\!\!\sum_{t=1}^T\!\!\frac{\|\phi(x_t,\omega_{a_tt})\|_{\Phi_{t-1}^{-1}}^2}{1+\|\phi(x_t,\omega_{a_tt})\|_{\Phi_{t-1}^{-1}}^2}\!\!\geq_{(ii)}\!\!\sum_{t=1}^T\frac{\|\phi(x_t,\omega_{a_tt})\|_{\Phi_{t-1}^{-1}}^2}{1+\bar{C}_{\phi}^2/\eta_0},
\end{array}$$
where (i) holds because $\log(1+a)\geq\frac{a}{1+a}$ for $a\geq0$, (ii) holds because  $\|\phi(x_t,\omega_{a_tt})\|_{\Phi_{t-1}^{-1}}^2\leq\bar{C}_{\phi}^2/\eta_0$. The above inequality implies that 
$$\sum_{t=1}^T\|\phi(x_t,\omega_{a_tt})\|_{\Phi_{t-1}^{-1}}^2\leq(1+\bar{C}_{\phi}^2/\eta_0)\log\frac{\mathrm{det}(\Phi_T)}{\mathrm{det}(\Phi_0)}\leq_{(1)}(1+\bar{C}_{\phi}^2/\eta_0)d\log\left(1+T\bar{C}_{\phi}^2/(\eta_0d)\right),$$
where (1) follows because 
$\displaystyle\mathrm{det}(\Phi_T)\leq\left(\frac{\mathrm{tr}(\Phi_T)}{d}\right)^d\leq\left(\eta_0+T\bar{C}_{\phi}^2/d\right)^d$ and $\mathrm{det}(\Phi_0)=\eta_0^d$ so that 
$\displaystyle\log\frac{\mathrm{det}(\Phi_T)}{\mathrm{det}(\Phi_0)}\leq d\log\left(1+T\bar{C}_{\phi}^2/(\eta_0d)\right)$. Thus $\displaystyle\sum_{t=1}^T\|\phi(x_t,\omega_{a_tt})\|_{\Phi_{t-1}^{-1}}\leq\sqrt{T\sum_{t=1}^T\|\phi(x_t,\omega_{a_tt})\|_{\Phi_{t-1}^{-1}}^2}\leq\sqrt{T(1+\bar{C}_{\phi}^2/\eta_0)d\log\left(1+T\bar{C}_{\phi}^2/(\eta_0d)\right)}$ by Cauchy-Schwarz inequality,
implying that on $\mathcal{E}$, we have 
\begin{equation}\label{eq:mu:tilde:diff}
    \sum_{t=1}^T\widetilde{\mu}_t(x_t,\omega_{a_tt})-\mu^*(x_t,\omega_{a_tt})\leq (C_{\mu}\sqrt{d}+4L\gamma_0\sqrt{\eta_0}\bar{\alpha})\sqrt{T},\qquad\mbox{where}
\end{equation} 
\begin{equation}\label{eq:C:mu}
C_\mu=\frac{4L\left[\gamma_0\sqrt{d\log(1+T\bar{C}_{\phi}^2/\eta_0)+2\log(3/\delta)}+\sqrt{\eta_0}\bar{\alpha}\right]}{\min\{\underline{\mu},1\}}\sqrt{\left(1+\bar{C}_{\phi}^2/\eta_0\right)\log\left(1+T\bar{C}_{\phi}^2/(\eta_0d)\right)}.
\end{equation}
Additionally, Assumption~\ref{ass:projection:anti:concentration} and Proposition~\ref{prop:minimum:eigenvalue} imply that with probability at least $1-\delta/3$, uniformly for all $t\geq\tilde{\tau}_0:=100\bar{C}_{\phi}^2\mathrm{polylog}(T)\log(6TAd/\delta)$ and $a\in[A]$, we have 
$\lambda_{\min}(\Phi_{t-1})\geq(t-1)/[16\mathrm{polylog}(T)]$, under which $\sum_{s\in\mathcal{S}_{at}}\|\phi(x_{s},\omega_{as})\|_{\Phi_{t-1}^{-1}}^2\leq16\mathrm{polylog}(T)\bar{C}_{\phi}^2,$ 
where the last inequality uses the fact that $n_{at}\leq t$. Henceforth, combining \eqref{eq:rho:diff:V:square} and \eqref{eq:theta:diff:phi}, with probability at least $1-\delta$, uniformly over all $t\geq100\bar{C}_{\phi}^2\mathrm{polylog}(T)\log(6TAd/\delta)$ we have 
$$\begin{array}{rl}
\displaystyle\|\hat{\rho}_{at}-\rho_a\|_{V_{at}(\eta_1)}&\displaystyle\leq\sqrt{\frac{4}{\kappa\mu_1}\!\left(\frac{144}{\kappa\mu_1}+2\min\{2\overline{\iota},\bar{C}_{\psi}d_{\mathcal{B}}\}+6\right)\Gamma_{at}(\delta)+\frac{20}{\kappa\mu_1}+\frac{1}{2}+\eta_1d_{\mathcal{B}}^2}\\
&\quad\displaystyle+\frac{8\sqrt{2}L}{\kappa\mu_1}\bar{C}_{\phi}\sqrt{\bar{C}_{at}(\delta)\mathrm{polylog}(T)}.
\end{array}$$
Furthermore, by \eqref{eq:indices:tilde} and the Lipschitz property of $\Lambda$ by Assumption~\ref{ass:reg_res_index}, conditional on event $\mathcal{E}$, with probability at least $1-\delta$, uniformly over all $a\in[A]$, $t\ge\tilde{\tau}_0=100\bar{C}_{\phi}^2\mathrm{polylog}(T)\log(6TAd/\delta)$, 
\begin{equation}\label{eq:sigma:tilde:tau0:tilde:diff}
    \widetilde{\sigma}_{at}-\sigma^*_{at}\leq C_{\sigma}\|\psi(x_t)\|_{V_{at}(\eta_1)^{-1}},\qquad\mbox{where}
\end{equation}
\begin{equation}\label{eq:C:sigma}
\begin{array}{rl}
C_{\sigma}=2L\bigg[&\displaystyle\sqrt{\frac{4}{\kappa\mu_1}\left(\frac{144}{\kappa\mu_1}+2\min\{2\overline{\iota},\bar{C}_{\psi}d_{\mathcal{B}}\}+6\right)\overline\Gamma_{a}(\delta)+\frac{20}{\kappa\mu_1}+\frac{1}{2}+\eta_1d_{\mathcal{B}}^2}\\
&\quad\displaystyle+\frac{8\sqrt{2}L}{\kappa\mu_1}\bar{C}_{\phi}\sqrt{\bar{C}_{a}(\delta)\mathrm{polylog}(T)}\bigg],
\end{array}
\end{equation}
and $\overline{Q}_{a}$, $\overline\Gamma_{a}(\delta)$, $\bar{C}_a(\delta)$ are defined as $\overline{Q}_{a}:=\left\lceil\log_2\left(1+T\min\{9\overline{\iota}^2,d_{\mathcal{B}}^2\bar{C}_{\psi}^2\}\right)\right\rceil$,
$$\overline\Gamma_{a}(\delta)\!:=m\log(1+2d_{\mathcal{B}}\bar{C}_{\psi}T)+\log(2\overline{Q}_{a}+2)+\log\left(\!
\frac{6AT}{\delta}\!\right),$$
$$\bar{C}_{a}(\delta):=\frac{4}{\min\{\underline{\mu},1\}^2}\bigg(\gamma_0\sqrt{d\log(1+T\bar{C}_{\phi}^2/\eta_0)+2\log(6/\delta)}+\sqrt{\eta_0}\bar{\alpha}\bigg)^2.$$
Define $\displaystyle\tilde{V}_{a,n_{at}}:=\eta_1I_m+\sum_{s\in\mathcal{S}_{at}}\psi(x_{s})\psi(x_{s})^\top=V_{at}(\eta_1)$ and let $\tilde{V}_{a,0}:=\eta_1I_m$. Note that 
$\mathrm{det}(\tilde{V}_{a,n_{at}})=\mathrm{det}\left(\tilde{V}_{a,n_{at}-1}\right)\left(1+\|\psi(x_{t_a(n_{at})})\|_{\tilde{V}_{a,n_{at}-1}^{-1}}^2\right)$, where $t_a(n_{at})$ is the period where $a$ is queried the $n_{at}$-th time. Thus 
$$\begin{array}{rl}
\displaystyle\log\frac{\mathrm{det}(\tilde{V}_{a,n_{a,T+1}})}{\mathrm{det}(\tilde{V}_{a,0})}\!=\!\!\!\!\!\sum_{s\in\mathcal{S}_{a,T+1}}\!\!\!\!\!\log\left(1+\|\psi(x_{s})\|_{\tilde{V}_{a,n_{as}}^{-1}}^2\right)\!\geq_{(i)}\!\!\!\!\!\!\sum_{s\in\mathcal{S}_{a,T+1}}\!\!\!\frac{\|\psi(x_{s})\|_{\tilde{V}_{a,n_{as}}^{-1}}^2}{1+\|\psi(x_{s})\|_{\tilde{V}_{a,n_{as}}^{-1}}^2}\geq_{(ii)}\!\!\!\!\sum_{s\in\mathcal{S}_{a,T+1}}\frac{\|\psi(x_{s})\|_{\tilde{V}_{a,n_{as}}^{-1}}^2}{1+\bar{C}_{\psi}^2/\eta_1},
\end{array}$$
where inequality (i) uses the fact that $\log(1+a)\geq\frac{a}{1+a}$ for $a\geq0$ and (ii) uses the fact that $\|\psi(x_{s})\|_{\tilde{V}_{a,n_{as}}^{-1}}^2\leq\bar{C}_{\psi}^2/\eta_1$. Hence the above inequality implies that 
$$\sum_{s\in\mathcal{S}_{a,T+1}}\|\psi(x_{s})\|_{\tilde{V}_{a,n_{as}}^{-1}}^2\leq(1+\bar{C}_{\psi}^2/\eta_1)\log\frac{\mathrm{det}(\tilde{V}_{a,n_{a,T+1}})}{\mathrm{det}(\tilde{V}_{a,0})}\leq_{(1)}(1+\bar{C}_{\psi}^2/\eta_1)m\log\left(1+n_{a,T+1}\bar{C}_{\psi}^2/(\eta_1m)\right),$$
where (1) follows since $\mathrm{det}(\tilde{V}_{a,0})=\eta_1^m$ and $\displaystyle\mathrm{det}(\tilde{V}_{a,n_{a,T+1}})\leq\left(\frac{\mathrm{tr}(\tilde{V}_{a,n_{a,T+1}})}{m}\right)^m\leq\left(\eta_1+\frac{n_{a,T+1}\bar{C}_{\psi}^2}{m}\right)^m$ 
so that 
$\displaystyle\log\frac{\mathrm{det}(\tilde{V}_{a,n_{a,T+1}})}{\mathrm{det}(\tilde{V}_{a,0})}\leq m\log\left(1+\frac{n_{a,T+1}\bar{C}_{\psi}^2}{\eta_1m}\right)$.
Further, for any $a\in[A]$, 
$$\begin{array}{rl}
\displaystyle\sum_{t=\tilde{\tau}_0}^T\|\psi(x_t)\|_{\tilde{V}_{a,n_{at}}^{-1}}\mathbb{I}\{a\in\mathcal{A}_t\}&\displaystyle\leq\sum_{t=1}^T\|\psi(x_t)\|_{\tilde{V}_{a,n_{at}}^{-1}}\mathbb{I}\{a\in\mathcal{A}_t\}=\sum_{s\in\mathcal{S}_{a,T+1}}\|\psi(x_{s})\|_{\tilde{V}_{a,n_{as}}^{-1}}\\
&\displaystyle\leq_{(i)}\sqrt{n_{a,T+1}\sum_{s\in\mathcal{S}_{a,T+1}}\|\psi(x_{s})\|_{\tilde{V}_{a,n_{as}}^{-1}}^2},
\end{array}$$
where inequality (i) follows from Cauchy-Schwarz inequality. Note that $n_{a,T+1}\leq T$, thus 
for any $a\in[A]$,
$\displaystyle\sum_{t=\tilde{\tau}_0}^T\|\psi(x_t)\|_{\tilde{V}_{a,n_{at}}^{-1}}\mathbb{I}\{a\in\mathcal{A}_t\}\leq\sqrt{T(1+\bar{C}_{\psi}^2/\eta_1)m\log\left(1+T\bar{C}_{\psi}^2/(\eta_1m)\right)}$.
Combining with \eqref{eq:sigma:tilde:tau0:tilde:diff}, conditional on event $\mathcal{E}$, with probability at least $1-\delta$, uniformly over all $a\in[A]$, we have 
$$\begin{array}{rl}
\displaystyle\sum_{t=1}^T(\widetilde{\sigma}_{at}-\sigma^*_{at})\mathbb{I}\{a\in\mathcal{A}_t\}&\displaystyle\leq\sum_{t=1}^{\tilde{\tau}_0-1}(\widetilde{\sigma}_{at}-\sigma^*_{at})\mathbb{I}\{a\in\mathcal{A}_t\}+\sum_{t=\tilde{\tau}_0}^{T}(\widetilde{\sigma}_{at}-\sigma^*_{at})\mathbb{I}\{a\in\mathcal{A}_t\}\\
&\displaystyle\leq2(\tilde{\tau}_0-1)+C_{\sigma}\sqrt{T(1+\bar{C}_{\psi}^2/\eta_1)m\log\left(1+T\bar{C}_{\psi}^2/(\eta_1m)\right)},
\end{array}$$
where the second inequality follows because for $t\in[1,\tilde{\tau}_0]$, $\widetilde{\sigma}_{at}-\sigma^*_{at}\leq2$ by definition of the indices according to Assumption~\ref{ass:reg_res_index}. Henceforth, conditional on event $\mathcal{E}$, with probability at least $1-\delta$,
$$\begin{array}{rl}
\displaystyle\sum_{t=1}^T\sum_{a\in\mathcal{A}_t}(\widetilde{\sigma}_{at}-\sigma^*_{at})&\displaystyle=\sum_{a\in[A]}\sum_{t=1}^T(\widetilde{\sigma}_{at}-\sigma^*_{at})\mathbb{I}\{a\in\mathcal{A}_t\}\\
&\displaystyle\leq2A\tilde{\tau}_0+AC_{\sigma}\sqrt{T(1+\bar{C}_{\psi}^2/\eta_1)m\log\left(1+T\bar{C}_{\psi}^2/(\eta_1m)\right)}.
\end{array}$$
Note that by definition, $\Delta_t(\tilde{\pi})\leq1+2A$ for any $t\in[T]$. So \eqref{eq:exp:bound} and \eqref{eq:mu:tilde:diff} together further imply that 
\begin{equation}\label{eq:regret:E}
\begin{array}{rl}
\displaystyle\mathbb{E}\left[\sum_{t=1}^T\Delta_t(\tilde{\pi})\ \big|\ \mathcal{E}\right]&\displaystyle\leq(C_{\mu}\sqrt{d}+4L\gamma_0\sqrt{\eta_0}\bar{\alpha})\sqrt{T}+2A\tilde{\tau}_0+T(1+2A)\delta\\
&\quad\displaystyle+AC_{\sigma}\sqrt{T(1+\bar{C}_{\psi}^2/\eta_1)m\log\left(1+T\bar{C}_{\psi}^2/(\eta_1m)\right)}.
\end{array}
\end{equation}
Recall that $\mathcal{E}$ holds with probability at least $1-\delta$. Thus
\begin{equation}\label{eq:regret:E:complement}
\mathbb{E}\left[\sum_{t=1}^T\Delta_t(\tilde{\pi})\mathbb{I}\{\mathcal{E}^c\}\right]\leq T(2A+1)\delta.
\end{equation}
Thus setting $\delta=1/\sqrt{T}$, \eqref{eq:regret:E} and \eqref{eq:regret:E:complement} imply that ignoring logarithmic factors, we have 
\begin{equation}\label{eq:regret:bound}
\begin{array}{rl}
\displaystyle\mathbb{E}\bigg[\sum_{t=1}^T\Delta_t(\tilde{\pi})\bigg]&\displaystyle\leq(C_{\mu}\sqrt{d}+4L\gamma_0\sqrt{\eta_0}\bar{\alpha}+4A+2)\sqrt{T}+2A\tilde{\tau}_0\\
&\displaystyle\quad+AC_{\sigma}\sqrt{T(1+\bar{C}_{\psi}^2/\eta_1)m\log\left(1+T\bar{C}_{\psi}^2/(\eta_1m)\right)}.
\end{array}
\end{equation}
Recall that $C_\mu$ and $C_{\sigma}$ are defined as \eqref{eq:C:mu} and \eqref{eq:C:sigma}. So ignoring logarithmic factors we have 
$\mathbb{E}\left[\sum_{t=1}^T\Delta_t(\tilde{\pi})\right]\leq\widetilde{O}\left((d+Am+A\sqrt{dm})\sqrt{T}\right).$
Next, we compute $\eta_0, \eta_1$ that minimize the order of the right-hand side of \eqref{eq:regret:bound} (ignoring logarithmic-factor dependence). In the following, for functions $f,g:\mathbb{N}\to\mathbb{R}_{+}$, we write $f(T)=\widetilde{\Theta}(g(T))$ if both $f(T)=\widetilde{O}(g(T))$ and $g(T)=\widetilde{O}(f(T))$ hold. 

Suppressing logarithmic factors, we have $C_\mu
=
\widetilde\Theta\left(
\sqrt{1+\frac1{\eta_0}}
\right)$. Moreover, $\overline\Gamma_a(\delta)
=
\widetilde\Theta(m)$
and $\sqrt{\bar{C}_a(\delta)}
=
\widetilde\Theta\left(\sqrt d+\sqrt{\eta_0}\right)$.
Therefore, $C_\sigma
=
\widetilde\Theta\left(
\sqrt{m+\eta_1}+\sqrt d+\sqrt{\eta_0}
\right)$. 
Substituting these estimates into the regret bound and dropping additive
terms independent of $\eta_0,\eta_1$, the relevant log-free objective is $R(\eta_0,\eta_1)
\asymp
\sqrt T
\left[
d\sqrt{1+\frac1{\eta_0}}
+\sqrt{\eta_0}
+
A\sqrt m
\sqrt{1+\frac1{\eta_1}}
\left(
\sqrt{m+\eta_1}+\sqrt d+\sqrt{\eta_0}
\right)
\right]$. Since the factor $\sqrt T$ is common to all $\eta_0,\eta_1$-dependent
terms, it does not affect the minimizer. Define $g(\eta_1):=\sqrt{1+\frac1{\eta_1}}, b(\eta_1):=\sqrt{m+\eta_1}+\sqrt d.$
Then the relevant objective can be written as $r(\eta_0,\eta_1)
=
d\sqrt{1+\frac1{\eta_0}}
+
\left[
1+A\sqrt m\,g(\eta_1)
\right]\sqrt{\eta_0}
+
A\sqrt m\,g(\eta_1)b(\eta_1)$.
We first optimize $\eta_0$ for fixed $\eta_1$. Let $K(\eta_1):=1+A\sqrt m\,g(\eta_1)$.
The $\eta_0$-dependent part is $\displaystyle f(\eta_0;\eta_1)
=
d\sqrt{1+1/\eta_0}
+
K(\eta_1)\sqrt{\eta_0}$.
The first-order condition is $-\frac{d}{2\eta_0^2\sqrt{1+1/\eta_0}}
+
\frac{K(\eta_1)}{2\sqrt{\eta_0}}
=
0$, equivalently, $K(\eta_1)
=
\frac{d}{\eta_0\sqrt{1+\eta_0}}$. 
Let $\rho(\eta_1):=\frac{d}{K(\eta_1)}$. 
Then the positive solution satisfies
$\eta_0^\star(\eta_1)
\asymp
\begin{cases}
\rho(\eta_1), & \rho(\eta_1)\lesssim 1,\\
% [0.8ex]
\rho(\eta_1)^{2/3}, & \rho(\eta_1)\gtrsim 1.
\end{cases}$
It remains to optimize $\eta_1$. By the envelope theorem,
$\frac{d}{d\eta_1}
r(\eta_0^\star(\eta_1),\eta_1)
=
A\sqrt m
\left[
g'(\eta_1)
\left(
b(\eta_1)+\sqrt{\eta_0^\star(\eta_1)}
\right)
+
g(\eta_1)b'(\eta_1)
\right].$ Since $\frac{g'(\eta_1)}{g(\eta_1)}
=
-\frac{1}{2\eta_1(\eta_1+1)}, \ b'(\eta_1)
=
\frac{1}{2\sqrt{m+\eta_1}}$,
the first-order condition is equivalent to
$\eta_1(\eta_1+1)
=
\sqrt{m+\eta_1}
\left[
\sqrt{m+\eta_1}
+
\sqrt d
+
\sqrt{\eta_0^\star(\eta_1)}
\right].$ Thus $\eta_1^2-m
=
\left[
\sqrt d+\sqrt{\eta_0^\star(\eta_1)}
\right]
\sqrt{m+\eta_1}$. We next show that the term $\sqrt{\eta_0^\star(\eta_1)}$ never changes
the polynomial order of the $\eta_1$ minimizer. Since
$K(\eta_1)\ge 1$, we have $\rho(\eta_1)\le d$. If
$\rho(\eta_1)\lesssim 1$, then $\sqrt{\eta_0^\star(\eta_1)}
\asymp
\rho(\eta_1)^{1/2}
\lesssim
1
\lesssim
\sqrt d$.
If $\rho(\eta_1)\gtrsim 1$, then $\sqrt{\eta_0^\star(\eta_1)}
\asymp
\rho(\eta_1)^{1/3}
\lesssim
d^{1/3}
\lesssim
\sqrt d$.
Therefore, $\sqrt d+\sqrt{\eta_0^\star(\eta_1)}
\asymp
\sqrt d$, and the $\eta_1$ balancing equation reduces to $\eta_1^2
\asymp
m+\sqrt d\,\sqrt{m+\eta_1}$. We now solve this equation by considering different regimes. If $d\lesssim m$, then at $\eta_1\asymp \sqrt m$, $\sqrt d\,\sqrt{m+\eta_1}\lesssim m$.
Hence $\eta_1^\star\asymp m^{1/2}$. If $m\lesssim d\lesssim m^3$, then the solution satisfies
$\eta_1\lesssim m$, so $m+\eta_1\asymp m$. Hence $\eta_1^2
\asymp
\sqrt{dm}$, and therefore $\eta_1^\star
\asymp
(dm)^{1/4}$. If $d\gtrsim m^3$, then the solution satisfies $\eta_1\gtrsim m$, so
$m+\eta_1\asymp \eta_1$. Hence $\eta_1^2
\asymp
\sqrt{d\eta_1}$,
and therefore $\eta_1^\star
\asymp
d^{1/3}$. Combining the three regimes gives $\eta_1^\star
\asymp
\max\left\{m^{1/2},(dm)^{1/4},d^{1/3}\right\}$. At this optimizer, $g(\eta_1^\star)=\Theta(1)$. Therefore, $K(\eta_1^\star)
=
1+A\sqrt m\,G(\eta_1^\star)
\asymp
A\sqrt m$. Substituting this into the expression for $\eta_0^\star$ gives $\rho(\eta_1^\star)
=
\frac{d}{K(\eta_1^\star)}
\asymp
\frac{d}{A\sqrt m}$. Thus $\eta_0^\star
\asymp
\begin{cases}
\dfrac{d}{A\sqrt m}, & d\lesssim A\sqrt m,\\
\left(\dfrac{d}{A\sqrt m}\right)^{2/3}, & d\gtrsim A\sqrt m.
\end{cases}$
Equivalently, $\eta_0^\star
\asymp
\min\left\{
\frac{d}{A\sqrt m},
\left(\frac{d}{A\sqrt m}\right)^{2/3}
\right\}$. Consequently, suppressing logarithmic factors, the optimal polynomial
choices are $\eta_1^\star
\asymp
\max\left\{m^{1/2},(dm)^{1/4},d^{1/3}\right\},\ \eta_0^\star
\asymp
\min\left\{
\frac{d}{A\sqrt m},
\left(\frac{d}{A\sqrt m}\right)^{2/3}
\right\}$.
\proofend

\clearpage
\end{APPENDICES}

%%%%%%%%%%%%%%%%%
\end{document}